\documentclass[oneside,11pt]{article}
\usepackage[letterpaper,margin=1in]{geometry}
\usepackage{amsmath,amsthm}
\usepackage[hypertexnames=false]{hyperref}
\usepackage[capitalize,nameinlink]{cleveref}
\usepackage[T1]{fontenc}
\usepackage[utf8]{inputenc}
\usepackage{amssymb,mathtools,microtype}
\usepackage{graphicx,array,multirow,multicol}
\usepackage{tikz,tikz-cd}
\usetikzlibrary{positioning,arrows.meta,decorations.markings,hobby,calc}
\usepackage{float}
\usepackage[caption=false]{subfig}
\usepackage{stmaryrd,listings}
\let\ci\lstinline

\DeclareMathOperator*{\argmin}{\arg\min}

\newcommand{\E}{\mathbb{E}}

\newcommand{\Cov}{\operatorname{Cov}}

\newcommand{\given}{\,\vert\,}

\newcommand{\Normal}{\mathcal{N}}

\newcommand{\Unif}{\operatorname{Unif}}

\newcommand{\indicator}{\mathbf{1}}

\newcommand{\R}{\mathbb{R}}
\newcommand{\N}{\mathbb{N}}

\newcommand{\tr}{\operatorname{tr}}
\newcommand{\rank}{\operatorname{rank}}
\newcommand{\diag}{\operatorname{diag}}
\newcommand{\sign}{\operatorname{sign}}

\newcommand{\range}{\operatorname{range}}
\newcommand{\supp}{\operatorname{supp}}

\newcommand{\dist}{\operatorname{dist}}

\newcommand{\mc}{\mathcal}
\newcommand{\mr}{\mathrm}
\newcommand{\mf}{\mathfrak}

\DeclarePairedDelimiter{\norm}{\lVert}{\rVert}
\DeclarePairedDelimiter{\abs}{\lvert}{\rvert}
\DeclarePairedDelimiter{\inner}{\langle}{\rangle}
\DeclarePairedDelimiter{\set}{\{}{\}}

\newcommand{\GW}{\mr{GW}}
\newcommand{\IGW}{\mr{IGW}}
\newcommand{\HS}{\operatorname{HS}}

\newcommand{\redboxed}[1]{#1}
\newcommand{\image}[3]{%
  \begin{figure}[htbp]\centering
  \includegraphics[width=.75\linewidth]{images/#2}
  \caption{#3}\label{fig:#2}\end{figure}}
\newcommand{\imagesbs}[6]{%
  \begin{figure}[htbp]\centering
  \subfloat[#3]{\includegraphics[width=#1\linewidth]{images/#2}\label{fig:#2}}
  \hfill
  \subfloat[#5]{\includegraphics[width=#1\linewidth]{images/#4}\label{fig:#4}}
  \caption{#6}\label{fig:#2-#4}\end{figure}}
\let\imagesbsred\imagesbs

\let\arxivfigure\figure
\let\endarxivfigure\endfigure
\renewenvironment{figure}[1][]{\arxivfigure[!tb]}{\endarxivfigure}
\let\arxivtable\table
\let\endarxivtable\endtable
\renewenvironment{table}[1][]{\arxivtable[!tb]}{\endarxivtable}
\theoremstyle{plain}
\newtheorem{theorem}{Theorem}[section]
\newtheorem{lemma}[theorem]{Lemma}
\newtheorem{proposition}[theorem]{Proposition}
\newtheorem{corollary}[theorem]{Corollary}

\theoremstyle{definition}
\newtheorem{definition}[theorem]{Definition}
\newtheorem{example}[theorem]{Example}
\newtheorem{remark}[theorem]{Remark}
\hypersetup{colorlinks,linkcolor={blue!50!black},citecolor={blue!50!black},urlcolor={blue!50!black}}
\usepackage{titling}
\begin{document}
\title{\Large \textbf{Optimal Transportation and Alignment\\[-.2em] Between Gaussian Measures}}
\newcounter{authorsep}
\let\origand\and
\renewcommand{\and}{%
    \stepcounter{authorsep}%
    \ifnum\value{authorsep}=3 \par\vspace{.4em}\fi
    \origand
}
\author{Sanjit Dandapanthula\protect\footnotemark[2] \protect\footnotemark[1]\\[.4em]
    \texttt{\href{mailto:sanjitd@cmu.edu}{sanjitd@cmu.edu}}
    \and
    Aleksandr Podkopaev\protect\footnotemark[3]\\[.4em]
    \texttt{\href{mailto:sashapod@amazon.com}{sashapod@amazon.com}} \and
    Shiva Prasad Kasiviswanathan\protect\footnotemark[3]\\[.4em]
    \texttt{\href{mailto:kasivisw@gmail.com}{kasivisw@gmail.com}} \and
    Aaditya Ramdas\protect\footnotemark[4]\\[.4em]
    \texttt{\href{mailto:aramdas@stanford.edu}{aramdas@stanford.edu}} \and
    Ziv Goldfeld\protect\footnotemark[6]\\[.4em]
    \texttt{\href{mailto:goldfeld@cornell.edu}{goldfeld@cornell.edu}}
    \vspace{1em}
}
\date{October 6, 2026}
\setlength{\droptitle}{-3em}
\maketitle
\vspace{-1em}

\begingroup
\renewcommand{\thefootnote}{\fnsymbol{footnote}}
\footnotetext[2]{Department of Statistics and Machine Learning Department, Carnegie Mellon University}
\footnotetext[3]{Amazon Web Services (AWS) LogAnalytics}
\footnotetext[4]{Department of Statistics, Stanford University}
\footnotetext[6]{Cornell University, Department of Electrical and Computer Engineering}
\footnotetext[1]{A large part of this work was done while SD was an intern at Amazon Web Services during Summer 2025.}
\endgroup

\begin{abstract}
Optimal transport (OT) and Gromov–Wasserstein (GW) alignment provide interpretable geometric frameworks for comparing, transforming, and aggregating heterogeneous datasets---tasks ubiquitous in data science and machine learning. Because these frameworks are computationally expensive, large-scale applications often rely on closed-form solutions for Gaussian distributions under quadratic cost. This work provides a comprehensive treatment of Gaussian, quadratic cost OT and inner product GW (IGW) alignment, closing several gaps in the literature to broaden applicability. First, we treat the open problem of IGW alignment between uncentered Gaussians on separable Hilbert spaces by giving an exact variational characterization through a quadratic optimization over isometries and co-isometries (orthogonal matrices in finite dimensions), for which we derive tight analytic upper and lower bounds. If at least one Gaussian measure is centered, the solution reduces to a fully closed-form expression, which we further extend to an analytic solution for the IGW barycenter between centered Gaussians. We also present a reduction of Gaussian multimarginal OT with pairwise quadratic costs to a tractable optimization problem and prove that every second-order stationary point of this problem is globally optimal. To demonstrate utility, we compare embedding distributions of already-trained language-model distillations and cluster synthetic users using covariance spectra of their text embeddings.

\end{abstract}
\clearpage
\section{Introduction} \label{sec:introduction}

Comparing and aggregating heterogeneous datasets are fundamental aspects of modern data science and machine learning. To capture the geometry of the data corpus, it is standard to model datasets as probability distributions over high-dimensional manifolds. With this abstraction, two powerful frameworks for comparing probability distributions on the same or possibly different spaces, respectively, are \emph{optimal transport (OT)} and \emph{Gromov-Wasserstein (GW) alignment}; we depict the OT and GW problems in \Cref{fig:ot-gw}. The theory of OT gives rise to the Wasserstein distances, which in recent years have had a significant impact on the fields of statistics and machine learning, including for generative modeling \cite{arjovsky2017wasserstein, dukler2019wasserstein, kwon2022score, haviv2024wasserstein}, domain adaptation \cite{courty2017joint, lee2019sliced}, and general data science \cite{ho2017multilevel, mi2018variational, naumann2025wasserstein}. In a similar vein, GW distances have found numerous applications in language modeling \cite{alvarez2018gromov, kawakita2024gromov}, single-cell genomics \cite{demetci2022scot, klein2024genot}, computational geometry \cite{koehl2023computing, chambers2025stable}, generative modeling \cite{bunne2019learning, klein2023entropic}, and graph matching \cite{xu2019gromov, brogat2022learning}.

\begin{figure}[htbp]
\centering
\begin{minipage}[t]{.47\linewidth}
\centering
\resizebox{!}{2.6cm}{\begin{tikzpicture}[scale=0.7]
            \def\rotangle{-60}

            \def\connectionA{B'}
            \def\connectionB{C'}
            \def\connectionC{D'}
            \def\connectionD{E'}
            \def\connectionE{A'}

            \definecolor{originalcolor}{RGB}{50, 120, 200}
            \definecolor{transformedcolor}{RGB}{200, 50, 120}
            \definecolor{connectioncolor}{RGB}{100, 100, 100}

            \coordinate (A) at (1, 0.5);
            \coordinate (B) at (0.5, 1.2);
            \coordinate (C) at (-0.8, 0.8);
            \coordinate (D) at (-0.5, -1);
            \coordinate (E) at (1.2, -0.7);

            \coordinate (A') at ({cos(\rotangle)*1 - sin(\rotangle)*0.5}, {sin(\rotangle)*1 + cos(\rotangle)*0.5});
            \coordinate (B') at ({cos(\rotangle)*0.5 - sin(\rotangle)*1.2}, {sin(\rotangle)*0.5 + cos(\rotangle)*1.2});
            \coordinate (C') at ({cos(\rotangle)*(-0.8) - sin(\rotangle)*0.8}, {sin(\rotangle)*(-0.8) + cos(\rotangle)*0.8});
            \coordinate (D') at ({cos(\rotangle)*(-0.5) - sin(\rotangle)*(-1)}, {sin(\rotangle)*(-0.5) + cos(\rotangle)*(-1)});
            \coordinate (E') at ({cos(\rotangle)*1.2 - sin(\rotangle)*(-0.7)}, {sin(\rotangle)*1.2 + cos(\rotangle)*(-0.7)});

            \draw[black, thick, fill=black!10]
            plot[smooth cycle, tension=0.7] coordinates {
                    (-0.9, 1.2) (0.3, 1.4) (1.0, 1.4) (1.6, 0.8) (1.7, -0.2)
                    (1.4, -0.9) (0.7, -1.7) (-0.3, -1.7) (-1.4, -0.6) (-1.3, 0.4) (-1.2, 0.9)
                };

            \node[black] at (0, 1.8) {$\mathcal{X}$};
            \node[gray] at (-0.7, 0.3) {$\pi$};

            \draw[connectioncolor, gray, densely dashed] (A) -- (\connectionA);
            \draw[connectioncolor, gray, densely dashed] (B) -- (\connectionB);
            \draw[connectioncolor, gray, densely dashed] (C) -- (\connectionC);
            \draw[connectioncolor, gray, densely dashed] (D) -- (\connectionD);
            \draw[connectioncolor, gray, densely dashed] (E) -- (\connectionE);

            \fill[originalcolor] (A) circle (3pt) node[above right, originalcolor] {};
            \fill[originalcolor] (B) circle (3pt) node[above, originalcolor] {};
            \fill[originalcolor] (C) circle (3pt) node[above left, originalcolor] {};
            \fill[originalcolor] (D) circle (3pt) node[below left, originalcolor] {};
            \fill[originalcolor] (E) circle (3pt) node[right, originalcolor] {};

            \fill[transformedcolor] (A') circle (3pt) node[below, transformedcolor] {};
            \fill[transformedcolor] (B') circle (3pt) node[below right, transformedcolor] {};
            \fill[transformedcolor] (C') circle (3pt) node[above right, transformedcolor] {};
            \fill[transformedcolor] (D') circle (3pt) node[left, transformedcolor] {};
            \fill[transformedcolor] (E') circle (3pt) node[right, transformedcolor] {};

        \end{tikzpicture}}
\par\small (a) Optimal transport:
\par\smallskip
$\displaystyle\inf_{\pi\in\Pi(\mu,\nu)}\int c(x,y)\,d\pi(x,y)$
\par\smallskip ($\Pi(\mu,\nu)$ is the set of couplings)
\end{minipage}\hfill
\begin{minipage}[t]{.49\linewidth}
\centering
\resizebox{!}{2.6cm}{\begin{tikzpicture}[scale=0.7]
            \def\rotangle{-30}

            \definecolor{originalcolor}{RGB}{50, 120, 200}
            \definecolor{transformedcolor}{RGB}{200, 50, 120}

            \begin{scope}[shift={(-3,0)}]
                \coordinate (A) at (1, 0.5);
                \coordinate (B) at (0.5, 1.2);
                \coordinate (C) at (-0.8, 0.8);
                \coordinate (D) at (-0.5, -1);
                \coordinate (E) at (1.2, -0.7);

                \draw[originalcolor, thick, fill=originalcolor!15]
                plot[smooth cycle, tension=0.8] coordinates {(-1.1, 1.4) (0.7, 1.5) (1.4, 0.8) (1.5, -0.4) (1.3, -1.1) (0.2, -1.3) (-0.8, -1.2) (-1.0, -0.3) (-1.2, 0.5)};

                \node[originalcolor] at (-1.0, 1.9) {$\mathcal{X}$};

                \draw[originalcolor, thick] (B) -- (A);
                \draw[originalcolor, thick] (B) -- (C);

                \draw[originalcolor, thick] (B) ++(-54:0.4) arc (-54:-163:0.4);
                \node[originalcolor] at ([shift={(-108.5:0.64)}]B) {\small $\theta$};
            \end{scope}

            \begin{scope}[shift={(1,0)}]
                \coordinate (A') at ({cos(\rotangle)*1 - sin(\rotangle)*0.5}, {sin(\rotangle)*1 + cos(\rotangle)*0.5});
                \coordinate (B') at ({cos(\rotangle)*0.5 - sin(\rotangle)*1.2}, {sin(\rotangle)*0.5 + cos(\rotangle)*1.2});
                \coordinate (C') at ({cos(\rotangle)*(-0.8) - sin(\rotangle)*0.8}, {sin(\rotangle)*(-0.8) + cos(\rotangle)*0.8});
                \coordinate (D') at ({cos(\rotangle)*(-0.5) - sin(\rotangle)*(-1)}, {sin(\rotangle)*(-0.5) + cos(\rotangle)*(-1)});
                \coordinate (E') at ({cos(\rotangle)*1.2 - sin(\rotangle)*(-0.7)}, {sin(\rotangle)*1.2 + cos(\rotangle)*(-0.7)});

                \draw[transformedcolor, thick, fill=transformedcolor!10]
                plot[smooth cycle, tension=0.7] coordinates {(-0.9, 1.4) (0.3, 1.4) (1.0, 1.4) (1.6, 0.6) (1.7, -0.2) (1.4, -0.9) (0.7, -1.5) (-0.3, -1.5) (-1.4, -0.6) (-1.3, 0.4) (-1.2, 0.9)};

                \node[transformedcolor] at (1.0, 1.9) {$\mathcal{Y}$};

                \draw[transformedcolor, thick] (B') -- (A');
                \draw[transformedcolor, thick] (B') -- (C');

                \draw[transformedcolor, thick] (B') ++(-54+\rotangle:0.4) arc (-54+\rotangle:-163+\rotangle:0.4);
                \node[transformedcolor] at ([shift={(-120+\rotangle:0.8)}]B') {\small $\approx \theta$};
            \end{scope}

            \draw[gray, densely dashed] (-2, 0.5) to[bend right=20] ({1 + cos(\rotangle)*1 - sin(\rotangle)*0.5}, {sin(\rotangle)*1 + cos(\rotangle)*0.5});
            \draw[gray, densely dashed] (-2.5, 1.2) to[bend left=35] ({1 + cos(\rotangle)*0.5 - sin(\rotangle)*1.2}, {sin(\rotangle)*0.5 + cos(\rotangle)*1.2});
            \draw[gray, densely dashed] (-3.8, 0.8) to[bend right=50] ({1 + cos(\rotangle)*(-0.8) - sin(\rotangle)*0.8}, {sin(\rotangle)*(-0.8) + cos(\rotangle)*0.8});
            \draw[gray, densely dashed] (-3.5, -1) to[bend left=25] ({1 + cos(\rotangle)*(-0.5) - sin(\rotangle)*(-1)}, {sin(\rotangle)*(-0.5) + cos(\rotangle)*(-1)});
            \draw[gray, densely dashed] (-1.8, -0.7) to[bend right=15] ({1 + cos(\rotangle)*1.2 - sin(\rotangle)*(-0.7)}, {sin(\rotangle)*1.2 + cos(\rotangle)*(-0.7)});

            \node[gray] at (-0.3, 2.0) {$\pi$};

            \begin{scope}[shift={(-3,0)}]
                \fill[originalcolor] (A) circle (3pt) node[above right] {};
                \fill[originalcolor] (B) circle (3pt) node[above right] {};
                \fill[originalcolor] (C) circle (3pt) node[above left] {};
                \fill[originalcolor] (D) circle (3pt) node[below left] {};
                \fill[originalcolor] (E) circle (3pt) node[below right] {};
            \end{scope}

            \begin{scope}[shift={(1,0)}]
                \fill[transformedcolor] (A') circle (3pt) node[above right] {};
                \fill[transformedcolor] (B') circle (3pt) node[above right] {};
                \fill[transformedcolor] (C') circle (3pt) node[above left] {};
                \fill[transformedcolor] (D') circle (3pt) node[below left] {};
                \fill[transformedcolor] (E') circle (3pt) node[below right] {};
            \end{scope}

        \end{tikzpicture}}
\par\small (b) GW alignment:
\par\smallskip
\resizebox{\linewidth}{!}{$\displaystyle\inf_{\pi\in\Pi(\mu,\nu)}\int(k_{\mathcal X}(x,x')-k_{\mathcal Y}(y,y'))^2\,d(\pi\otimes\pi)(x,y,x',y')$}
\par\smallskip ($\pi\otimes\pi$ is the product measure)
\end{minipage}
\caption{(a) The OT problem compares measures over the same space by searching for the coupling $\pi$ which minimizes expected transportation cost $c(x, y)$. Here we depict the quadratic-cost OT problem, with $c(x, y) = \norm{x - y}_2^2$. (b) GW alignment compares measures over (possibly different) spaces by searching for the coupling $\pi$ which matches the kernels $k_{\mathcal{X}}$ and $k_{\mathcal{Y}}$ as well as possible. Here, we depict the inner product Gromov-Wasserstein (IGW) problem with $k_{\mathcal{X}}(\cdot,\, \cdot) = k_{\mathcal{Y}}(\cdot,\, \cdot)= \inner{\cdot,\, \cdot}$, which seeks to find a coupling which is as close to unitary as possible (roughly, preserving pairwise angles).}
    \label{fig:ot-gw}
    \end{figure}

Despite their utility, computation under the OT and GW frameworks is notoriously difficult. The Wasserstein distance between empirical distributions over $n$ data points can be expressed as a linear program; classical min-cost flow algorithms have $O(n^3\log n)$ complexity \cite{peyre2019computational}. Computing GW distances instead requires solving a nonconvex quadratic optimization problem over couplings, which is generally NP-hard \cite{kravtsova2025hardness}. Although there is recent work using entropic regularization \cite{cuturi2013sinkhorn, zhang2024gromov,rioux2024entropic}, slicing \cite{rabin2011wasserstein,kolouri2019generalized,nietert2022statistical}, and neural estimation \cite{makkuva2020optimal,daniels2021score,wang2024neural,mokrov2024energy} for more efficient computation, explicit solutions which can be obtained in special cases of the OT and GW problems are of intense theoretical and practical interest. As an important example, the \emph{Bures--Wasserstein} formulae for the 2-Wasserstein distance and barycenter between Gaussians have been heavily applied in the literature, including to linear metric learning \cite{cooper2023applications}, knowledge distillation in machine learning \cite{lv2024wasserstein}, style transfer for image generation \cite{mroueh2019wasserstein}, and covariance estimation \cite{han2023learning}.

Beyond these two landmark results and recent attempts to extend them to the GW setting \cite{salmona2022gromov, le2022entropic}, structured solutions under Gaussian distributions and quadratic-like costs are scarce. Motivated by the success of the Bures--Wasserstein formulae, this work studies explicit solutions for a breadth of OT and IGW problems between Gaussian measures, spanning the bimarginal, multimarginal, and barycentric settings, filling several important gaps in the literature.

\subsection{Contributions}

{
    We first consider the bimarginal IGW distance between Gaussian measures on a separable Hilbert space, for which Salmona et al.~\cite{salmona2022gromov} provide an expression restricted to centered Gaussians in $\R^d$. We give an exact variational characterization for uncentered Gaussians as a quadratic optimization over isometries and co-isometries (orthogonal matrices in finite dimensions), prove attainment, and characterize an optimal Gaussian coupling. We also show that matching first and second moments gives a Gaussian lower bound on IGW for general measures. Our analytic upper and lower bounds for Gaussian IGW differ by a Cauchy--Schwarz gap and coincide under a covariance-weighted mean alignment condition, including when either measure is centered. We further resolve the univariate case explicitly. As a natural consequence of the centered spectral characterization, we extend the Gaussian IGW barycenter formulae of Salmona et al.~\cite{salmona2022gromov} and Le et al.~\cite{le2022entropic} to separable Hilbert spaces and infinitely many input measures.
}

Next, we treat the multimarginal OT \cite{gangbo1998optimal} and multimarginal IGW problems between Gaussians. Observing that the pairwise optimal bimarginal couplings can be glued together into a multimarginal coupling, we resolve the multimarginal IGW problem between centered Gaussians over separable Hilbert spaces in closed-form. Using similar techniques, we formulate the multimarginal OT problem between Gaussians as the solution of a semidefinite program and show that the solution must have low rank. When considering $p$ Gaussian measures over $\R^d$, we prove that every second-order stationary point of the factorized problem with $k=d+1$ has rank $d$ and is globally optimal. Together with the Burer-Monteiro factorization \cite{burer2003nonlinear} this allows us to reduce the original semidefinite program with $O(d^2 p^2)$ variables to the problem of finding a second-order stationary point of a related problem with only $O(d^2 p)$ variables. Trust-region methods provide approximate second-order stationarity guarantees \cite{boumal2023introduction}, giving a useful computational reduction when $p$ is large.

Finally, in \Cref{app:experiments}, we demonstrate the utility of our theoretical results through experiments on synthetic and real-world data. We begin by proposing the IGW distance as a principled alternative to the centered kernel alignment (CKA), which is commonly used to compare neural network representations \cite{kornblith2019similarity}. We show that our upper and lower bounds on the IGW distance are close enough in practice to compare embedding distributions for common text datasets, using a two-dimensional diagnostic to assess the Gaussian modeling approximation in the leading principal components. We observe a generally negative association between IGW distance and CKA similarity. Next, we model synthetic users each by an empirical measure $\mu_i$ over their text embeddings, which we approximate by a Gaussian measure $\tilde{\mu}_i$. We then cluster the users using the IGW distance between the centered versions of $\tilde{\mu}_i$, showing that covariance spectra themselves contain useful information about user interests. We conclude by verifying in simulations that our Burer-Monteiro approach is faster than standard interior-point methods for large instances while attaining comparable objective values.

\subsection{Related Work} \label{sec:related-work}

\paragraph{IGW Alignment.} Vayer~\cite{vayer2020contribution} introduced the inner-product variant of the GW distance, proved that the IGW distance between measures with finite second moment is finite, and derived a variational form for the IGW distance between measures with finite fourth moment. Only recently, Zhang et al.~\cite{zhang2024gradient} formally proved that the IGW distance is a metric on the zero-distance quotient of the space of probability measures over a Hilbert space with finite second moment and studied the Riemannian geometry induced by the IGW distance. Dumont et al.~\cite{dumont2025existence} proved the existence of deterministic OT maps for the IGW problem between compactly supported and absolutely continuous measures in Euclidean spaces. Our work complements these theoretical results by restricting attention to Gaussian measures, which allows us to obtain closed-form expressions and efficient algorithms for computing the IGW distance and barycenter.

\paragraph{2-Wasserstein Distance Between Gaussians.} The 2-Wasserstein distance between Gaussian measures in Euclidean space was originally studied by Knott and Smith~\cite{knott1984optimal} and Givens and Shortt~\cite{givens1984class} and is known in closed-form, giving rise to the Bures-Wasserstein geometry on the space of positive definite matrices. Their results were later extended by Gelbrich~\cite{gelbrich1990formula} to separable Hilbert spaces. Additionally, the 2-Wasserstein barycenter between Gaussian measures can be written as the solution of a non-linear matrix fixed-point equation, which is known to have a unique positive definite solution \cite{agueh2011barycenters}. The Bures-Wasserstein formulae and the 2-Wasserstein barycenter between Gaussians have been widely applied in the literature, including linear metric learning \cite{cooper2023applications}, knowledge distillation in machine learning \cite{lv2024wasserstein}, style transfer for image generation \cite{mroueh2019wasserstein}, and covariance estimation \cite{han2023learning}. Our work extends these important results by studying analogs of the Bures-Wasserstein formula for the IGW distance.

\paragraph{GW Alignment Between Gaussians.} Salmona et al.~\cite{salmona2022gromov} studied the $(2, 2)$-GW distance between Gaussian measures in Euclidean space and found an explicit solution only after restricting the optimization to Gaussian couplings; they were also able to recover the IGW distance and OT map between \emph{centered} Gaussian measures as a corollary. Le et al.~\cite{le2022entropic} extended this work by finding closed-form expressions for the entropically regularized IGW distance and barycenter between \emph{centered} Gaussian measures in Euclidean space, when the regularization parameter is positive. However, neither of these works treat the (unregularized) IGW alignment of uncentered Gaussian measures or general Hilbert spaces, which we address in this paper.

\paragraph{Multimarginal OT.} The multimarginal OT problem generalizes classical OT to more than two marginal distributions. This problem was introduced for the quadratic cost by Gangbo and {\'S}wi{\k{e}}ch~\cite{gangbo1998optimal}, who focused on extending Brenier's theorem \cite{brenier1991polar} to the multimarginal setting. Agueh and Carlier~\cite{agueh2011barycenters} later related multimarginal OT to the 2-Wasserstein barycenter, by reformulating the latter as a weighted average over the optimal multimarginal coupling. Section 3 of Pass~\cite{pass2015multi} gives an overview of numerous applications of multi-marginal OT, including matching problems in economics, texture mixing in image processing, derivative pricing in finance, and incompressible fluid dynamics in engineering. In a similar vein, Beier et al.~\cite{beier2023multi} studied properties of the multimarginal GW problem between general measures and showed that it recovers the GW barycenter. Our work provides an efficient algorithm to solve the multimarginal OT problem between uncentered Gaussian measures, as well as an analytic solution to multimarginal IGW between centered Gaussians.

\subsection{Notation}

Let $[n] \coloneq \set{1, \dots, n}$. For $1 \leq p \leq \infty$ and probability measures $\mu$ on $\mc{X}$, let $L^p(\mu)$ be the space of real-valued measurable functions $f$ on $\mathcal{X}$ such that $\norm{f}_{L^p(\mu)} \coloneq \left( \int \abs{f(x)}^p\, d\mu(x) \right)^{1/p} <\infty$, with the usual identification of functions that are equal $\mu$-almost everywhere. We treat $\ell^2 \coloneq L^2(\N)$ as a Hilbert space with the Euclidean inner product and write $\norm{\cdot}_p$ for the $p$-norm on $\R^d$. We write $\mc{P}(\mc{X})$ for the class of Borel probability measures over a metric space $(\mc{X}, d_\mc{X})$, and let $\mc{P}_p(\mc{X}) \coloneq \set{\mu \in \mc{P}(\mc{X}) : d_\mc{X}(x_0, \cdot) \in L^p(\mu) \text{ for all } x_0 \in \mc{X}}$ be the subset of $\mc{P}(\mc{X})$ with finite $p$th moment. We use $T_\sharp\, \mu \in \mc{P}(\mc{X})$ to denote the pushforward of $\mu$ under a measurable function $T$, defined by $T_\sharp\, \mu(\cdot) \coloneq \mu(T^{-1}(\cdot))$.

Throughout, all Hilbert spaces are real. For a Hilbert space $\mc{H}$, we use $\inner{\cdot, \cdot}$ and $\norm{x} \coloneq\sqrt{\inner{x, x}}$ to denote the inner product and the induced norm, respectively. The set of trace-class operators is $\mf{S}_1(\mc{H})$ and the set of Hilbert-Schmidt operators is $\mf{S}_2(\mc{H})$. The set of unitary operators is designated by $\mc{U}(\mc{H})$, and we write $I$ for the identity. We let $A^*$ denote the adjoint of $A$, while $A^\dagger$ stands for the Moore-Penrose pseudoinverse \cite{desoer1963note}. {We also write $\overline{\mc{U}}(\mc{H}) \coloneq \set{C : \mc{H} \to \mc{H} \text{ bounded linear} : C^* C = I \text{ or } C C^* = I}$ for the set of linear isometries and co-isometries of $\mc{H}$; these are exactly the extreme points of the closed unit ball of the space of bounded linear operators on $\mc{H}$ \cite{kadison1951isometries}. Note that $\mc{U}(\mc{H}) \subseteq \overline{\mc{U}}(\mc{H})$, with equality when $\mc{H}$ is finite-dimensional, and that $\overline{\mc{U}}(\mc{H})$ is closed under adjoints.}

Throughout, we employ several operator and matrix norms. Let $\norm{A}_{\mr{op}} \coloneq \sup_{\norm{x} = 1}\, \norm{Ax}$ be the operator norm. The Hilbert-Schmidt norm is denoted by $\norm{A}_{\HS} \coloneq \left( \sum_{i=1}^\infty \sigma_i(A)^2 \right)^{1/2}$, where $\sigma_i(A)$ are the singular values of $A$. If $A$ is trace-class, we denote its trace by $\tr(A) \coloneq \sum_{i=1}^\infty \inner{Ae_i, e_i}$, where $\set{e_i}_{i=1}^\infty$ is any orthonormal basis of $\mc{H}$. We use $\succeq$ to denote the Loewner order ($A \succeq 0$ means $A$ is positive semidefinite (psd)) and $\otimes$ to denote the tensor product. For a positive, compact operator $A$, we write $(\lambda_i(A))_{i\geq 1}$ for its positive eigenvalues in nonincreasing order, repeated according to multiplicity and padded with zeros when the rank is finite. These zeros need not enumerate the kernel. We use the convention $(x\otimes y)z=\inner{y,z}x$. If $A \in \R^{d_1 \times d_2}$ is a matrix, we write $A^\intercal$ for its transpose and define its Frobenius norm as $\norm{A}_\mr{F} \coloneq \big(\tr(A^\intercal A)\big)^{1/2}$.

For a probability measure on a Hilbert space $\mu \in \mc{P}_2(\mc{H})$, we write $m_\mu \coloneq \int x\, d\mu(x)$ for its mean. The \emph{covariance operator} of $\mu$ is the linear operator $\Sigma_\mu : \mc{H} \to \mc{H}$ defined by the property $\inner{x, \Sigma_\mu y} = \int \inner{x, z - m_\mu} \inner{y, z - m_\mu}\, d\mu(z)$. Given $\mu_1, \mu_2 \in \mc{P}_2(\mc{H})$ and a coupling $\pi \in \Pi(\mu_1, \mu_2)$, the \emph{cross-covariance operator} is the linear operator $\Sigma_\pi^\times : \mc{H} \to \mc{H}$ defined by the property $\inner{x, \Sigma_\pi^\times y} = \int \inner{x, r - m_{\mu_1}} \inner{y, s - m_{\mu_2}}\, d\pi(r, s)$; this corresponds to the upper-right block of $\Sigma_\pi$ written as a block operator on $\mc{H} \oplus \mc{H}$ (where $\oplus$ denotes the direct sum). We write $\Normal(m, \Sigma)$ for the Gaussian measure on $\mc{H}$ with mean $m$ and covariance operator $\Sigma$. Write $\mc{G}(\mc{H})$ for the set of Gaussian measures on $\mc{H}$, and let $\mc{G}_0(\mc{H}) = \set{\mu \in \mc{G}(\mc{H}) : m_\mu = 0}$ be the set of centered Gaussians. A brief exposition on probability theory over Hilbert spaces is provided in \Cref{app:background-hilbert}.

\section{Background} \label{sec:background}

\paragraph{OT, barycenter, and multimarginal.} The Kantorovich OT problem between probability distributions $\mu_1\in\mc{P}(\mc{X}_1)$ and $\mu_2 \in \mc{P}(\mc{X}_2)$ with cost function $c : \mc{X}_1 \times \mc{X}_2 \to \R$  is given by \cite{kantorovich1942translocation, villani2003topics}:
\begin{align*}
    \mr{OT}_c(\mu_1, \mu_2) \coloneq \inf_{\pi \in \Pi(\mu_1, \mu_2)}\; \int c(x, y)\, d\pi(x, y),
\end{align*}
where $\Pi(\mu_1, \mu_2)\coloneqq \{\pi\in\mathcal{P}(\mathcal{X}_1\times \mathcal{X}_2):\,\pi(\cdot\times\mathcal{X}_2)=\mu_1,\,\pi(\mc{X}_1\times\cdot)=\mu_2\}$ is the set of couplings of $\mu_1, \mu_2$. Intuitively, the Kantorovich problem searches for a (randomized) transportation plan $\pi$ that reshapes $\mu_1$ into $\mu_2$ while minimizing the expected transportation cost. The common choice $\mc{X}_1=\mc{X}_{2}= \R^d$ under the quadratic cost $\norm{x - y}_2^2$ yields the \emph{2-Wasserstein distance}
\begin{equation} \label{eq:w2}
    \mr{W}_2(\mu_1, \mu_2) \coloneq \mr{OT}_{\norm{\cdot}^2}(\mu_1, \mu_2)^{1/2} = \inf_{\pi \in \Pi(\mu_1, \mu_2)}\, \norm{X - Y}_{L^2(\pi)}
\end{equation}
between measures $\mu_1, \mu_2 \in \mc{P}_2(\R^d)$. If $\mu_1$ and $\mu_2$ are also absolutely continuous with respect to the Lebesgue measure, a landmark theorem due to Brenier~\cite{brenier1991polar} shows that a unique optimal plan for \eqref{eq:w2} is given by $\pi = (I,\, \nabla \varphi)_\sharp\, \mu_1$ for a proper lower semi-continuous convex function $\varphi : \R^d \to \R\cup\{+\infty\}$. In other words, the optimal plan is induced by a deterministic transport map $T = \nabla \varphi$, given as the gradient of a convex function (i.e., monotone for $d = 1$ and $\tilde{c}$-cyclically monotone in general \cite[Chapter 2]{villani2003topics}). The Wasserstein distance metrizes weak convergence of probability measures together with convergence of second moments and gives rise to a rich Riemannian geometry on the so-called 2-Wasserstein space $(\mc{P}_2(\R^d),\mr{W}_2)$ \cite[Chapter 8]{villani2003topics}. Further background related to OT between Gaussians is provided in \Cref{app:gaussian-wasserstein} and \Cref{app:background-ot}.

The Wasserstein geometry induces a meaningful notion of an average of collections of probability distributions, under the Wasserstein barycenter framework. For $\rho \in \mc{P}(\mc{P}_2(\R^d))$, the $\rho$-weighted 2-Wasserstein barycenter is defined as the Frechet mean under $\mr{W}_2$:
\begin{align*}
    \bar{\mu} \in \argmin_{\mu \in \mc{P}_2(\R^d)}\; \int \mr{W}_2(\mu, \nu)^2\, d\rho(\nu).
\end{align*}
Under appropriate assumptions, the 2-Wasserstein barycenter exists and is unique \cite{agueh2011barycenters}, and it is usually computed using fixed-point iteration or entropic regularization \cite[Section 9.2]{peyre2019computational}.

Gangbo and {\'S}wi{\k{e}}ch~\cite{gangbo1998optimal} studied the related multimarginal OT problem
\begin{equation} \label{eq:multimarginal-ot}
    \inf_{\pi \in \Pi(\mu_1, \dots, \mu_p)}\, \int \sum_{i=1}^p \sum_{j=i+1}^p \norm{x_i - x_j}_2^2\, d\pi(x_1, \dots, x_p),
\end{equation}
where $\Pi(\mu_1, \dots, \mu_p)$ is the set of $p$-plans, i.e., joint distributions $\pi\in\mathcal{P}(\mathcal{X}_1\times\ldots\times \mathcal{X}_p)$ that have $\mu_i$ as their $i$th marginal, for all $i=1,\ldots,p$. Gangbo and {\'S}wi{\k{e}}ch~\cite{gangbo1998optimal} proved that if $\mu_1, \dots, \mu_p$ are absolutely continuous with respect to the Lebesgue measure, then there exists a unique optimal coupling $\pi$ for the multimarginal problem. The barycenter and multimarginal OT problem are connected by the following result.

\begin{proposition}[Agueh and Carlier~\cite{agueh2011barycenters}, Proposition 4.2] \label{prop:multimarginal-wasserstein-barycenter}
    Suppose that $\rho = \sum_{i=1}^p \rho_i\, \delta_{\mu_i}$, where $\rho_1, \dots, \rho_p > 0$, $\sum_{i=1}^p\rho_i=1$, and the measures $\mu_i \in \mc{P}_2(\R^d)$ are absolutely continuous. Define the weighted-average map
    \begin{align*}
        g(x_1, \ldots, x_p) \coloneq \sum_{i=1}^p \rho_i x_i
    \end{align*}
    for $x_1, \dots, x_p \in \R^d$. If $\pi \in \Pi(\mu_1, \ldots, \mu_p)$ solves the $\rho$-weighted multimarginal OT problem
    {
    \begin{equation} \label{eq:weighted-multimarginal-ot}
        \inf_{\pi \in \Pi(\mu_1, \dots, \mu_p)}\, \int \sum_{i=1}^p \sum_{j=i+1}^p \rho_i\, \rho_j \norm{x_i - x_j}_2^2\, d\pi(x_1, \dots, x_p),
    \end{equation}
    }%
    then $g_\# \pi$ is a $\rho$-weighted 2-Wasserstein barycenter.
\end{proposition}

{For uniform weights $\rho_i = 1/p$, the objective of \eqref{eq:weighted-multimarginal-ot} is proportional to that of \eqref{eq:multimarginal-ot}, so the two problems have the same solutions. Moreover, \eqref{eq:weighted-multimarginal-ot} coincides with the multimarginal formulation of Agueh and Carlier~\cite{agueh2011barycenters} owing to the identity $\sum_{i=1}^p \rho_i \norm{x_i - g(x_1, \ldots, x_p)}_2^2 = \sum_{i=1}^p \sum_{j=i+1}^p \rho_i\, \rho_j \norm{x_i - x_j}_2^2$, which is valid whenever $\sum_{i=1}^p \rho_i = 1$.}

\paragraph{The Gaussian case: Bures-Wasserstein distance.} As an important special case of \eqref{eq:w2}, suppose that $\mu_1 = \Normal(m_1, \Sigma_1)$ and $\mu_2 = \Normal(m_2, \Sigma_2)$ are nondegenerate Gaussian measures, where $\Sigma_1 \succ 0$ and $\Sigma_2 \succ 0$. Brenier's theorem then implies that the optimal coupling between Gaussian measures $\pi = (I,\, T)_\sharp\, \mu_1$ must be supported on the graph of a deterministic function $T : \R^d \to \R^d$. Knott and Smith~\cite{knott1984optimal} and Givens and Shortt~\cite{givens1984class} prove the formula for the OT map
\begin{equation} \label{eq:gaussian-ot-map}
    T(x) = \Sigma_1^{-1/2} (\Sigma_1^{1/2} \Sigma_2 \Sigma_1^{1/2})^{1/2} \Sigma_1^{-1/2} (x - m_1) + m_2,
\end{equation}
resulting in the \emph{Bures-Wasserstein} formula for the 2-Wasserstein distance
\begin{equation} \label{eq:bures-wasserstein}
    \mr{W}_2(\mu_1, \mu_2) = \sqrt{\norm{m_1 - m_2}_2^2 + \tr\left( \Sigma_1 + \Sigma_2 - 2\, (\Sigma_1^{1/2} \Sigma_2 \Sigma_1^{1/2})^{1/2} \right)},
\end{equation}
which yields the Bures-Wasserstein geometry, induced by the metric
\begin{align*}
    \mr{BW}(\Sigma_1, \Sigma_2) \coloneq \mr{W}_2(\Normal(0, \Sigma_1),\, \Normal(0, \Sigma_2))
\end{align*}
on the space of symmetric $d \times d$ positive definite matrices.

Beyond the quadratic bimarginal case, another instance that notably simplifies under the Gaussian setting is the Wasserstein barycenter. Agueh and Carlier~\cite{agueh2011barycenters} proved that if $\supp(\rho) = \set{\mu_1, \dots, \mu_p}$, where $\mu_i = \Normal(0, \Sigma_i)$ and $\Sigma_i \succ 0$ are nondegenerate Gaussians for $i \in [p]$, then the barycenter is Gaussian $\Normal(0, \Sigma)$ whose covariance matrix $\Sigma$ is the unique positive definite fixed-point of the equation
\begin{equation} \label{eq:w2-barycenter}
    \Sigma = \int (\Sigma^{1/2} \Sigma_\nu \Sigma^{1/2})^{1/2}\, d\rho(\nu).
\end{equation}
Equations \eqref{eq:bures-wasserstein} and \eqref{eq:w2-barycenter} are the cornerstones for various recent works studying theory and practical applications of closed-form solutions to Gaussian OT \cite{mroueh2019wasserstein,cooper2023applications,han2023learning,lv2024wasserstein}, motivating our study of various OT problems between Gaussians.

\paragraph{GW alignment.} The Wasserstein distance inherently assumes distributions supported on a shared space, as the cost is given in terms of the base distance. The \emph{Gromov-Wasserstein (GW) distances} provide a framework for quantifying discrepancy between distributions on distinct spaces ($\mu_1 \in \mc{P}(\mc{X})$ and $\mu_2 \in \mc{P}(\mc{Y})$). This is done by minimizing distortion between intrinsic similarity functions on each space ($k_\mc{X}$ and $k_\mc{Y}$ on $\mc{X}$ and $\mc{Y}$ respectively):
\begin{align*}
    \GW_2(\mu_1, \mu_2)
     & \coloneq \left( \inf_{\pi \in \Pi(\mu_1, \mu_2)}\, \int (k_\mc{X}(x, x^\prime) - k_\mc{Y}(y, y^\prime))^2\, d(\pi \otimes \pi)(x, y, x^\prime, y^\prime) \right)^{1/2} \\
     & = \inf_{\pi \in \Pi(\mu_1,\mu_2)}\; \norm{k_\mc{X}(X, X^\prime) - k_\mc{Y}(Y, Y^\prime)}_{L^2(\pi \otimes \pi)}.
\end{align*}
If $\mc{X}$ and $\mc{Y}$ are Hilbert spaces, setting $k_\mc{X}(x, x^\prime) = \inner{x, x^\prime}$ and $k_\mc{Y}(y, y^\prime) = \inner{y, y^\prime}$ yields the \emph{inner product Gromov-Wasserstein (IGW) distance}. If $\mc{H}$ is a Hilbert space, the IGW distance is a metric on the quotient space of $\mc{P}_2(\mc{H})$, modulo zero IGW distance, equivalently an inner-product-preserving identification on the supports \cite[Proposition 3.1]{zhang2024gradient}. The extension of the GW framework to the multimarginal setting was recently considered in Beier et al.~\cite{beier2023multi}. That work also established the relationship between multimarginal GW and the GW barycenter, analogous to \Cref{prop:multimarginal-wasserstein-barycenter} in the Wasserstein case.

Previous work studies the IGW distance and barycenter between Gaussian measures, but existing results are limited to centered Gaussians in finite dimensions. For bimarginal IGW, Salmona et al.~\cite{salmona2022gromov} show that
\begin{align*}
    \IGW\left( \Normal(0, \Sigma_1),\, \Normal(0, \Sigma_2) \right) = \left( \sum_{k=1}^\infty (\lambda_k(\Sigma_1) - \lambda_k(\Sigma_2))^2 \right)^{1/2},
\end{align*}
where $\lambda_k(\Sigma_i)$ denotes the $k$th largest eigenvalue of $\Sigma_i$. Building on this result, Le et al.~\cite{le2022entropic} further provided a formula for the IGW barycenter. In this work, we extend these results to significantly broaden their practical applicability. Further exposition on GW distances and related results are collected in \Cref{app:background-gw}. We also provide some background on Riemannian optimization and the Stiefel manifold in \Cref{app:background-riemannian}.

\section{IGW Alignment Between Gaussian Measures} \label{sec:igw-gaussian}

This section studies the IGW alignment problem between Gaussian measures on general Hilbert spaces. We treat the bimarginal setting, analyzing both the alignment cost and optimal plan, and build on it to resolve the IGW barycenter. Throughout this section, assume that $\mc{H}$ is a separable Hilbert space.

\subsection{IGW Distance and Optimal Coupling}

Our first result provides an exact variational characterization of the IGW distance and optimal coupling between general (possibly uncentered) Gaussian measures, through a quadratic optimization over isometries and co-isometries. Following this statement, we derive tight analytic upper and lower bounds on the optimal value of that optimization problem.

\begin{theorem}[IGW between Gaussians]\label{thm:igw-gaussian-distance}
    
    Let $\mu_i=\Normal(m_i,\Sigma_i)$ be Gaussian measures on $\mc{H}$ for $i\in[2]$, and define
    \begin{equation}\label{eq:gamma}
        \gamma(\mu_1,\mu_2)\coloneq\max_{\norm{C}_{\mr{op}}\leq 1}
        \left\{\norm{\Sigma_2^{1/2}C\Sigma_1^{1/2}}_{\HS}^2
        +2\inner{C\Sigma_1^{1/2}m_1,\Sigma_2^{1/2}m_2}\right\}.
    \end{equation}
    The maximum is attained at some $C_\star\in\overline{\mc{U}}(\mc{H})$, and
    \begin{align*}
        \IGW(\mu_1,\mu_2)^2
         & =\tr(\Sigma_1^2)+\tr(\Sigma_2^2)+2\inner{m_1,\Sigma_1m_1}+2\inner{m_2,\Sigma_2m_2} \\
         & \quad +(\norm{m_1}^2-\norm{m_2}^2)^2-2\gamma(\mu_1,\mu_2).
    \end{align*}
    An optimal IGW coupling is
    \begin{align*}
        \pi=\Normal\left(\begin{bmatrix}m_1\\m_2\end{bmatrix},
        \begin{bmatrix}\Sigma_1                            & \Sigma_1^{1/2}C_\star^*\Sigma_2^{1/2} \\
               \Sigma_2^{1/2}C_\star\Sigma_1^{1/2} & \Sigma_2\end{bmatrix}\right).
    \end{align*}
\end{theorem}

\begin{remark}[Optimal transport map]{}\label{rem:igw-optimal-map}
    
    If $\mc{H}=\R^d$ and $\Sigma_1\succ0$, then $C_\star$ is orthogonal and the optimal coupling in \Cref{thm:igw-gaussian-distance} is induced by the affine map
    \begin{align*}
        T(x)=m_2+\Sigma_2^{1/2}C_\star\Sigma_1^{-1/2}(x-m_1).
    \end{align*}
    In spectral coordinates $\Sigma_i=Q_i\Lambda_iQ_i^*$, with $Q_i$ orthogonal, set $\widetilde C_\star=Q_2^*C_\star Q_1$. Equivalently,
    \begin{align*}
        T(x)=m_2+Q_2\Lambda_2^{1/2}\widetilde C_\star\Lambda_1^{-1/2}Q_1^*(x-m_1).
    \end{align*}
    Indeed, $T_\sharp\mu_1$ has mean $m_2$ and covariance $\Sigma_2^{1/2}C_\star C_\star^*\Sigma_2^{1/2}=\Sigma_2$, while the cross-covariance of $(X,T(X))$ is $\Sigma_1^{1/2}C_\star^*\Sigma_2^{1/2}$, as required. In infinite dimensions, the inverse square root may be unbounded; the joint Gaussian construction in \Cref{thm:igw-gaussian-distance} still provides an optimal coupling.
\end{remark}

The proof of \Cref{thm:igw-gaussian-distance}, deferred to \Cref{sec:proofs}, observes that, upon expanding the IGW cost, the resulting optimization problem only depends on the covariance of $\pi$. Thus, we may restrict to jointly Gaussian couplings $\pi$ without changing the optimal value. Rewriting the constraints using the classical characterization of cross-covariance operators \cite{baker1973joint}, we arrive at an optimization problem over contractive operators. Using Bauer's maximum principle, we then argue that the optimum is attained at an extreme point of the feasible set---that is, at an isometry or co-isometry---resulting in the optimization $\gamma(\mu_1,\mu_2)$ over $\overline{\mc{U}}(\mc{H})$.

\begin{remark}[Finite-dimensional case]
    If $\mu_i=\Normal(m_i,\Sigma_i)$ are Gaussian measures on $\R^{d_i}$, for $i\in[2]$, with $d_1\geq d_2$ w.l.o.g., write $\Sigma_i=Q_i\Lambda_iQ_i^\intercal$ with $Q_i$ orthogonal and $\tilde m_i=Q_i^\intercal m_i$. In these coordinates, the formula from \eqref{eq:gamma} simplifies to
    \begin{align*}
        \gamma(\mu_1, \mu_2) = \sup_{C \in \mr{St}(d_1,d_2)}\; \left\{\tr(\Lambda_1 C \Lambda_2 C^\intercal) + 2 \left\langle\Lambda_1^{1/2} \tilde{m}_1, C \Lambda_2^{1/2} \tilde{m}_2\right\rangle\right\},
    \end{align*}
    where $\Lambda_i = \diag(\lambda_1(\Sigma_i), \dots, \lambda_{d_i}(\Sigma_i))$ for $i  \in [2]$ and $\mr{St}(d_1, d_2) \coloneq \set{C \in\mathbb{R}^{d_1\times d_2} : C^\intercal C = I}$ is the \emph{Stiefel manifold} of orthonormal $d_2$-frames in $\mathbb{R}^{d_1}$. Since $\mr{St}(d_1,d_2)$ is a smooth, compact Riemannian manifold, we may use standard manifold optimization techniques to solve this problem in practice \cite{absil2008optimization}. Due to the non-convexity of the objective, iterative manifold optimization schemes could converge to local optima, which would only provide a lower bound on $\gamma(\mu_1, \mu_2)$. \Cref{thm:igw-gaussian-bound}, stated next, provides analytic upper and lower bounds on $\gamma(\mu_1,\mu_2)$ that require no optimization and hold in general Hilbert spaces.  \end{remark}

\begin{remark}[Suboptimal local maxima]{}
    
    The orthogonal optimization can have strict suboptimal local maxima even in dimension two. Take $\Sigma_1=\Sigma_2=\diag(2,1)$, $m_1=(1/2,0)^\intercal$, and $m_2=(-1/2,0)^\intercal$. For the rotation $C_\theta$ through angle $\theta$, the objective in \eqref{eq:gamma} is
    \begin{align*}
        g(\theta)=5-\sin^2\theta-\cos\theta,\qquad
        g(0)=4,\quad g'(0)=0,\quad g''(0)=-1,\quad g(\pi)=6.
    \end{align*}
    Thus $I$ is a strict suboptimal local maximum, although $I$ and $-I$ belong to the same connected component of the orthogonal group. Since $\mu_2=(-I)_\sharp\mu_1$, the IGW distance is zero, whereas the identity parameter yields squared cost $12-2g(0)=4$. Riemannian gradient ascent may therefore stop at a suboptimal alignment.
\end{remark}

The following bounds on the Gaussian IGW distance follow by inspecting  $\gamma(\mu_1,\mu_2)$ from \eqref{eq:gamma} and identifying closed-form solutions to each of the two added terms in the objective, separately. The resulting expressions are given in terms of the means and covariance operators of $\mu_1,\mu_2$, and are thus computable.

    {
        For the spectral bounds, write $\Sigma_i=Q_i\Lambda_iQ_i^*$, where $\Lambda_i=\diag(\lambda_k(\Sigma_i))$ and $Q_i$ maps each coordinate to its corresponding positive-eigenvalue eigenvector, or to zero for a zero eigenvalue. Set $\tilde m_i=Q_i^*m_i$.
    }

\begin{theorem}[Analytic bounds]\label{thm:igw-gaussian-bound}
    
    Under the setting of \Cref{thm:igw-gaussian-distance}, define
    \begin{align*}
        \xi(\mu_1,\mu_2)
         & \coloneq\sum_{k=1}^\infty(\lambda_k(\Sigma_1)-\lambda_k(\Sigma_2))^2
        +2\inner{m_1,\Sigma_1m_1}+2\inner{m_2,\Sigma_2m_2}                      \\
         & \quad +(\norm{m_1}^2-\norm{m_2}^2)^2.
    \end{align*}
    Then the IGW distance satisfies
    \begin{equation}\label{eq:IGW_ineq}
        \begin{aligned}
            \xi(\mu_1,\mu_2)-4\norm{\Lambda_1^{1/2}\tilde m_1}\norm{\Lambda_2^{1/2}\tilde m_2}
             & \leq\IGW(\mu_1,\mu_2)^2                 \\
             & \leq\xi(\mu_1,\mu_2)-4\sum_{k=1}^\infty
            \left|(\Lambda_1^{1/2}\tilde m_1)_k(\Lambda_2^{1/2}\tilde m_2)_k\right|.
        \end{aligned}
    \end{equation}
\end{theorem}

{
The lower bound follows by bounding the covariance term in \eqref{eq:gamma} using the von Neumann trace inequality and the mean term using Cauchy--Schwarz. For the upper bound, take $C=Q_2 DQ_1^*$, where $D$ is diagonal with
\begin{align*}
    D_{kk}=\sign\big((\Lambda_1^{1/2}\tilde m_1)_k(\Lambda_2^{1/2}\tilde m_2)_k\big),
\end{align*}
choosing either sign when the product is zero. Here, note that the upper and lower bounds differ only by a Cauchy--Schwarz gap. For general probability measures, the following moment-matched Gaussian lower bound applies.
}

\begin{corollary}[Gaussian lower bound] \label{cor:igw-general-bound}
    If $\mu_i \in \mc{P}_2(\mc{H})$ for $i \in [2]$, we have
    \begin{align*}
        \IGW(\mu_1, \mu_2) \geq \IGW\left(\Normal(m_{\mu_1}, \Sigma_{\mu_1}),\, \Normal(m_{\mu_2}, \Sigma_{\mu_2})\right).
    \end{align*}
\end{corollary}

When the reference measure is a centered Gaussian, we can also replace the other measure by a centered Gaussian with the same uncentered second moment without increasing the IGW distance. This absorbs the contribution of its mean into the covariance and reduces the comparison to two centered Gaussians.

\begin{corollary}[Centered Gaussianization] \label{cor:igw-centered-gaussianization}
    Let $\mu \in \mc{P}_2(\mc{H})$, let $\nu = \Normal(0,\Sigma_\nu)$, and define the uncentered second-moment operator $M_\mu \coloneq \Sigma_\mu + m_\mu \otimes m_\mu$. Then
    {
        \begin{align*}
            \IGW(\mu,\nu) \geq \IGW\big(\Normal(0,M_\mu),\nu\big).
        \end{align*}
    }%
\end{corollary}

Lastly, we highlight two cases where $\gamma(\mu_1,\mu_2)$ is solvable in closed-form, leading to an analytic expression for the IGW distance between Gaussians.

\begin{corollary}[Analytic solutions] \label{cor:igw-analytic-gaussian}
    The following hold:
    \begin{enumerate}
        \item[(i)] \emph{(Co-centered Gaussians)} Under the setting of \Cref{thm:igw-gaussian-distance}, suppose that $\inner{\Lambda_1^{1/2}\tilde m_1,\Lambda_2^{1/2}\tilde m_2}=\norm{\Lambda_1^{1/2}\tilde m_1}\norm{\Lambda_2^{1/2}\tilde m_2}$. Then
              \begin{align*}
                  \IGW(\Normal\left(m_1, \Sigma_1),\, \Normal(m_2, \Sigma_2)\right) = \sqrt{\xi(\mu_1, \mu_2) - 4\, \norm{\Lambda_1^{1/2} \tilde{m}_1}\, \norm{\Lambda_2^{1/2} \tilde{m}_2}},
              \end{align*}
              where $\xi(\mu_1, \mu_2)$ is given in the statement of \Cref{thm:igw-gaussian-bound}.
        \item[(ii)] \emph{(Univariate Gaussians)} The IGW distance between Gaussians on $\R$ is given by
              \begin{align*}
                  \IGW\left(\Normal(m_1, \sigma_1^2),\, \Normal(m_2, \sigma_2^2)\right) = \sqrt{(\sigma_1^2 - \sigma_2^2)^2 + (m_1^2 - m_2^2)^2 + 2 (\sigma_1 \abs{m_1} - \sigma_2 \abs{m_2})^2}.
              \end{align*}
    \end{enumerate}
\end{corollary}

The expression for the first case follows from \Cref{thm:igw-gaussian-bound}, since the upper and lower bounds coincide under the stated equality condition in Cauchy--Schwarz, including either covariance-weighted mean being zero. This result generalizes the formula of Salmona et al.~\cite{salmona2022gromov} and Le et al.~\cite{le2022entropic} to (for example) the case where only one of the measures is centered, and further allows us to handle the case of general separable Hilbert spaces. The formula for univariate Gaussians is a direct consequence of \Cref{thm:igw-gaussian-distance} by observing that $\overline{\mc{U}}(\R) = \set{\pm 1}$, which enables resolving $\gamma(\mu_1, \mu_2)$ in closed-form.

\begin{remark}[Gromov-Bures-Wasserstein distance]{}
    
    By (i) of \Cref{cor:igw-analytic-gaussian}, the zero-distance quotient of $(\mc G_0(\mc H),\IGW)$ is isometric to
    \begin{align*}
        \mc S=\left\{(\lambda_k)_{k\geq1}:\lambda_1\geq\lambda_2\geq\cdots\geq0,
        \quad\sum_{k=1}^\infty\lambda_k<\infty\right\},
    \end{align*}
    with the $\ell^2$ distance when $\mc H$ is infinite-dimensional. In dimension $d$, use the ordered nonnegative cone in $\R^d$. Summability is required because Gaussian covariance operators are trace-class. Thus the \emph{Gromov-Bures-Wasserstein (GBW) distance}
    \begin{align*}
        \mr{GBW}(\Sigma_1,\Sigma_2)\coloneq\IGW(\Normal(0,\Sigma_1),\Normal(0,\Sigma_2))
        =\left(\sum_{k=1}^\infty(\lambda_k(\Sigma_1)-\lambda_k(\Sigma_2))^2\right)^{1/2}
    \end{align*}
    is a pseudometric on covariance operators, since distinct operators with the same ordered spectrum have zero distance. On the quotient, the interpolation $\lambda(t)=(1-t)\lambda^{(0)}+t\lambda^{(1)}$ remains ordered, nonnegative, and summable, and satisfies
    \begin{align*}
        \norm{\lambda(t)-\lambda(s)}_{\ell^2}=|t-s|\norm{\lambda^{(1)}-\lambda^{(0)}}_{\ell^2}.
    \end{align*}
    Consequently, the quotient is geodesic and its intrinsic length distance equals the spectral distance. GBW cannot induce a nondegenerate Riemannian distance on the full covariance space. The usual Bures--Wasserstein distance in \eqref{eq:bures-wasserstein}, by comparison, induces a Riemannian geometry on the finite-dimensional positive definite cone \cite{knott1984optimal,givens1984class}.
\end{remark}

We conclude this section with the following example that provides geometric intuition about IGW alignment between Gaussian measures in $\R^2$.

\begin{example}[Gaussians on the plane]
    For Gaussian inputs, any choice of $C \in \overline{\mc{U}}(\mc{H})$ in \Cref{thm:igw-gaussian-distance} (such as one obtained by a manifold optimization scheme) yields a valid upper bound on the Gaussian IGW distance. Initializing such a scheme with the diagonal signs used in \Cref{thm:igw-gaussian-bound}, and following Riemannian gradient ascent for $\gamma$ (equivalently, gradient descent for its negative) allows us to tighten the upper bound in \Cref{thm:igw-gaussian-bound}.

    Now, we can plot the displacement interpolation between two bivariate Gaussians under the IGW distance and $\mr{W}_2$ distance in \Cref{fig:interpolation-comparison}. Recall that for an optimal OT/IGW map $T$, the displacement interpolation between $\mu_1$ and $\mu_2$ is defined by $\big((1 - t) I + t T\big)_\sharp\, \mu_1$, for $t \in [0, 1]$. The IGW optimal map is given in \Cref{rem:igw-optimal-map}; to estimate $C$, we use Riemannian gradient descent (RGD) with a simple adaptive step size for a maximum of 50 iterations or until the norm of the Riemannian gradient is less than $10^{-2}$. The 2-Wasserstein OT map is given by \eqref{eq:gaussian-ot-map}.

    {
    \Cref{fig:interpolation-comparison} compares the two interpolations. The IGW alignment changes the distributional orientation in a way that depends on the origin, whereas the 2-Wasserstein interpolation is invariant to translating both measures. In both displacement interpolations, each particle follows the straight segment $t\mapsto(1-t)x+tT(x)$.
    }

    \begin{figure}[htbp]
        \centering

        {\begin{tabular}{@{}r@{\hspace{1em}}cccc@{}}
                \raisebox{3em}{\footnotesize (a) $\IGW$:}                                      &
                \redboxed{\includegraphics[scale=0.15]{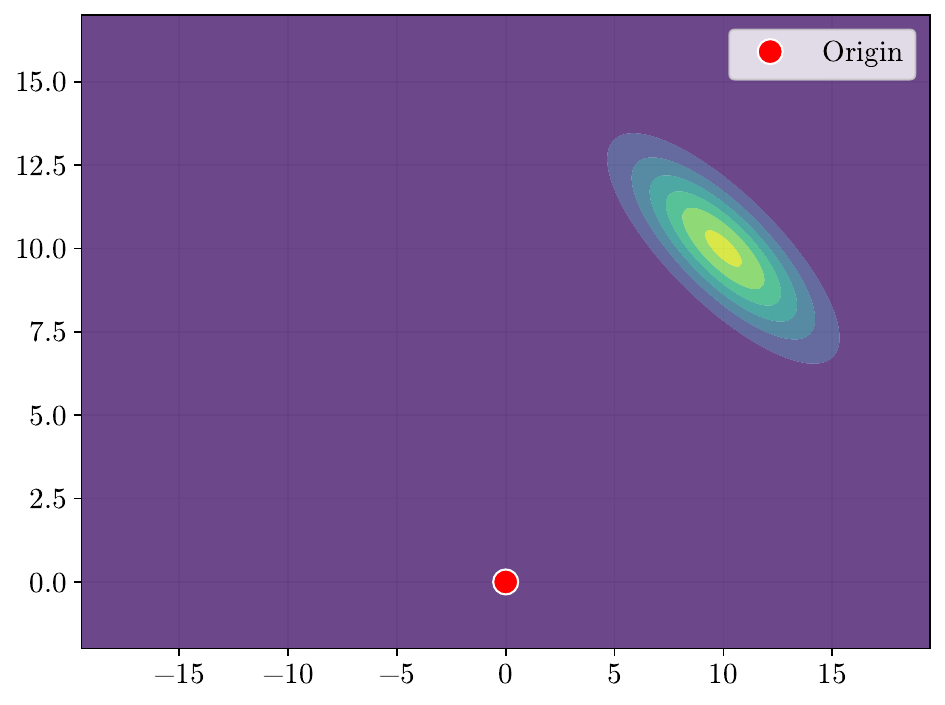}}         &
                \redboxed{\includegraphics[scale=0.15]{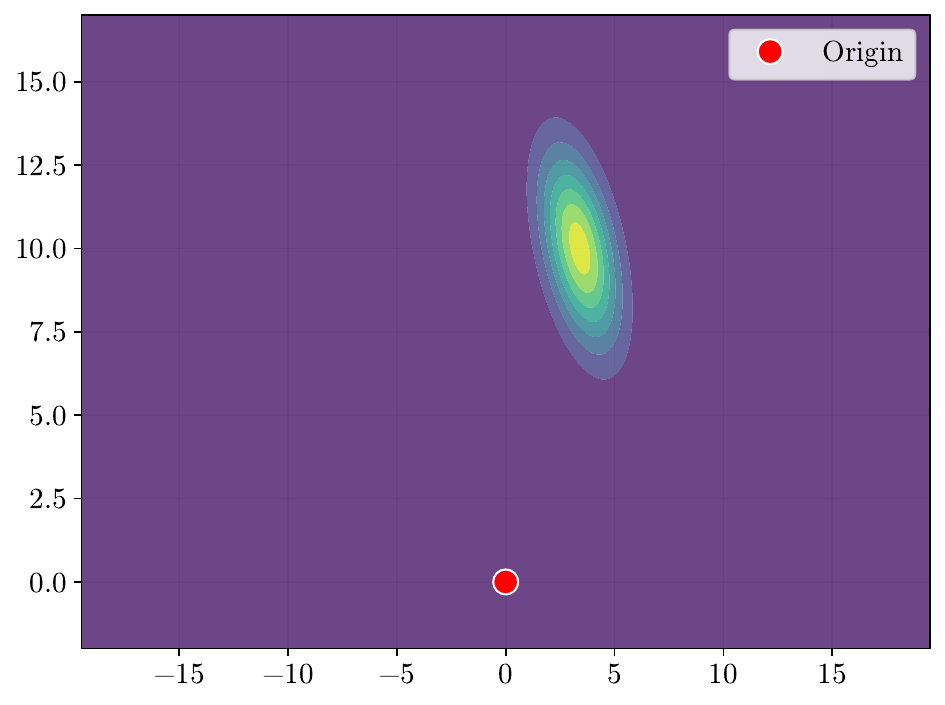}}         &
                \redboxed{\includegraphics[scale=0.15]{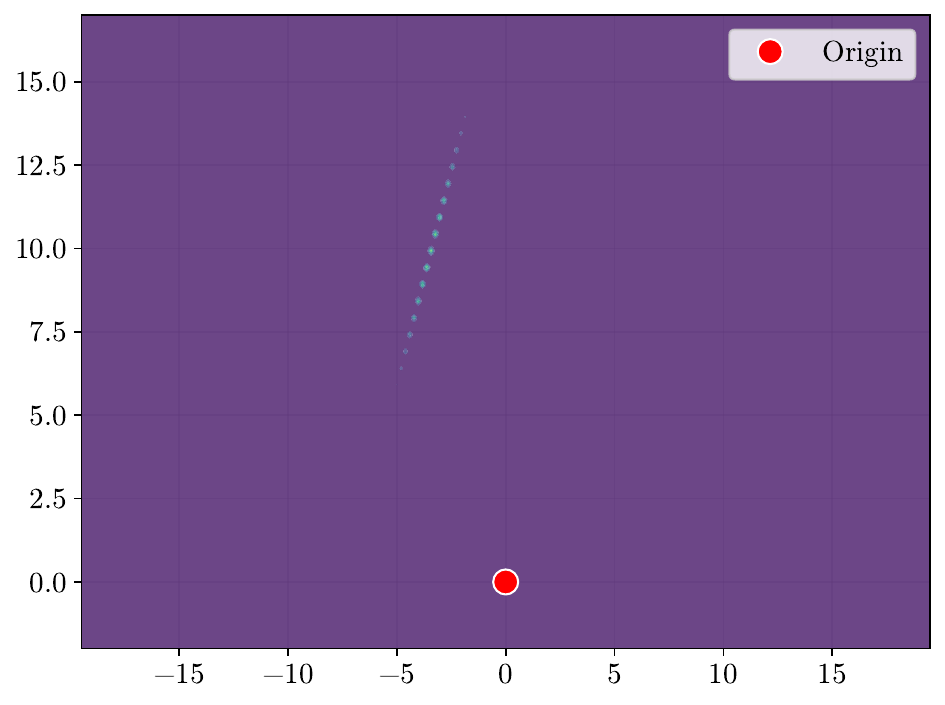}}         &
                \redboxed{\includegraphics[scale=0.15]{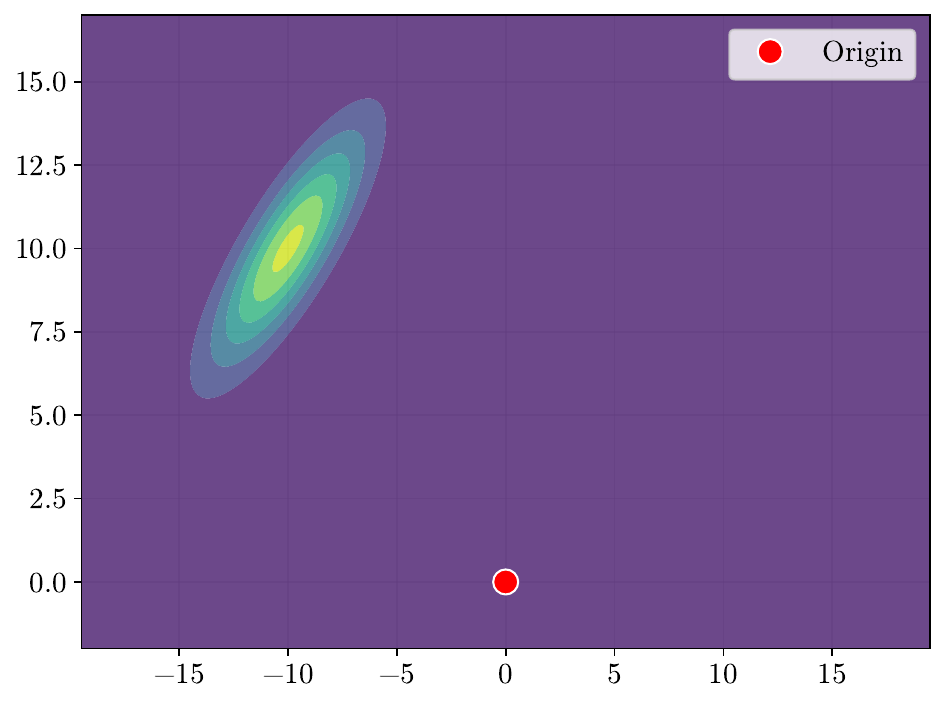}}           \\

                \raisebox{3em}{\footnotesize (a) $\mr{W}_2$:}                                  &
                \redboxed{\includegraphics[scale=0.15]{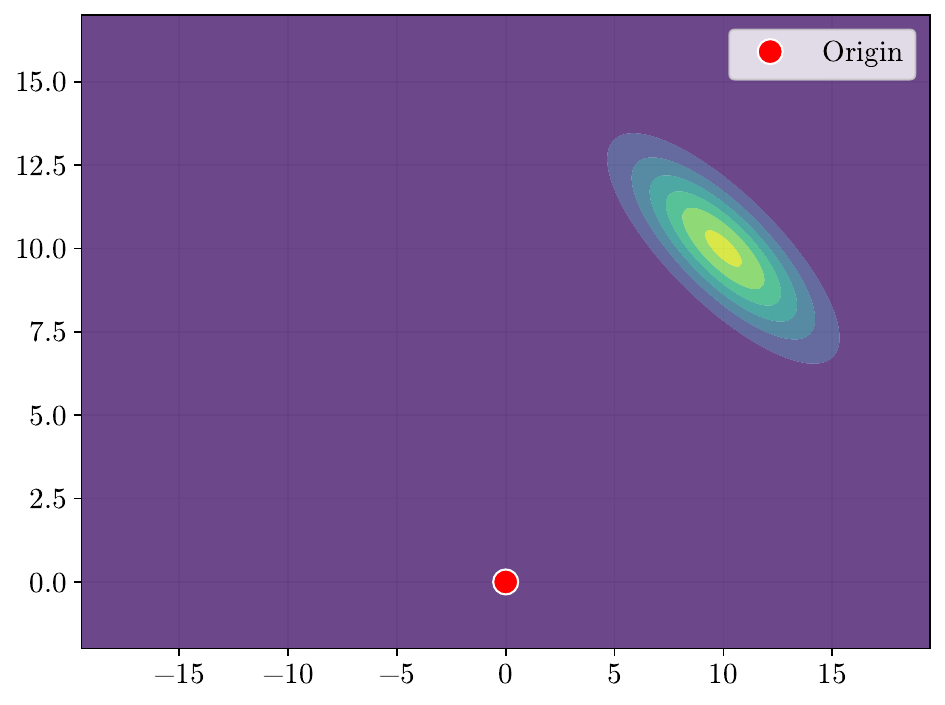}}          &
                \redboxed{\includegraphics[scale=0.15]{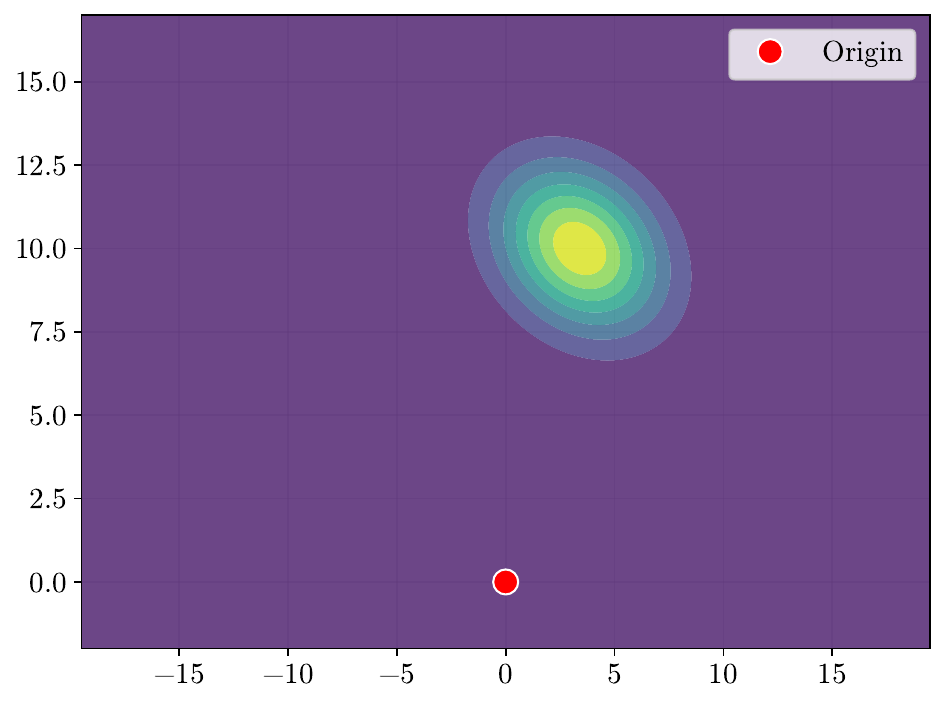}}          &
                \redboxed{\includegraphics[scale=0.15]{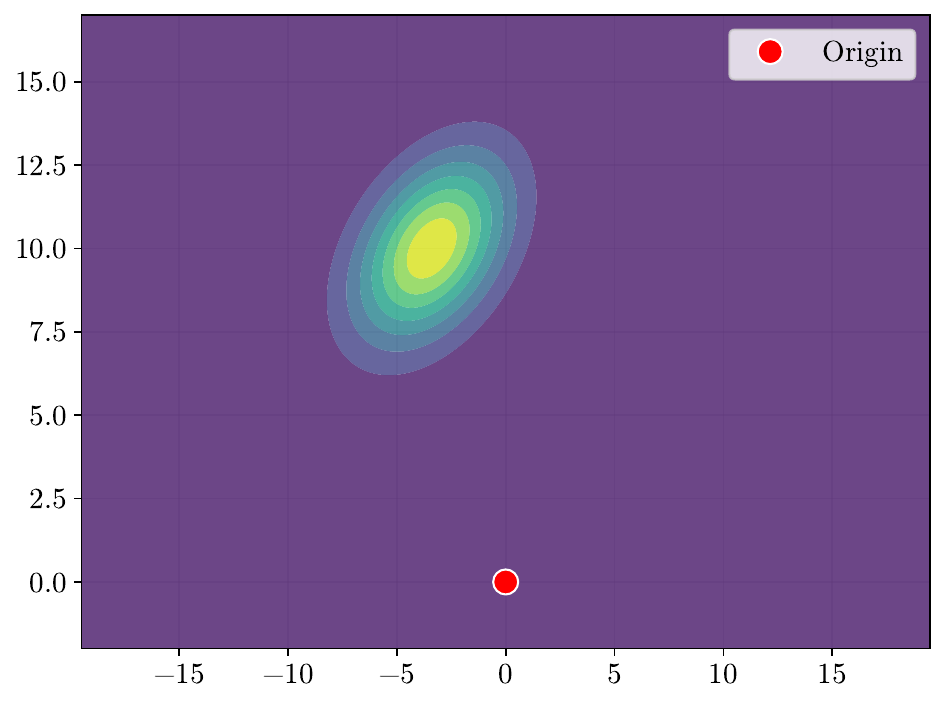}}          &
                \redboxed{\includegraphics[scale=0.15]{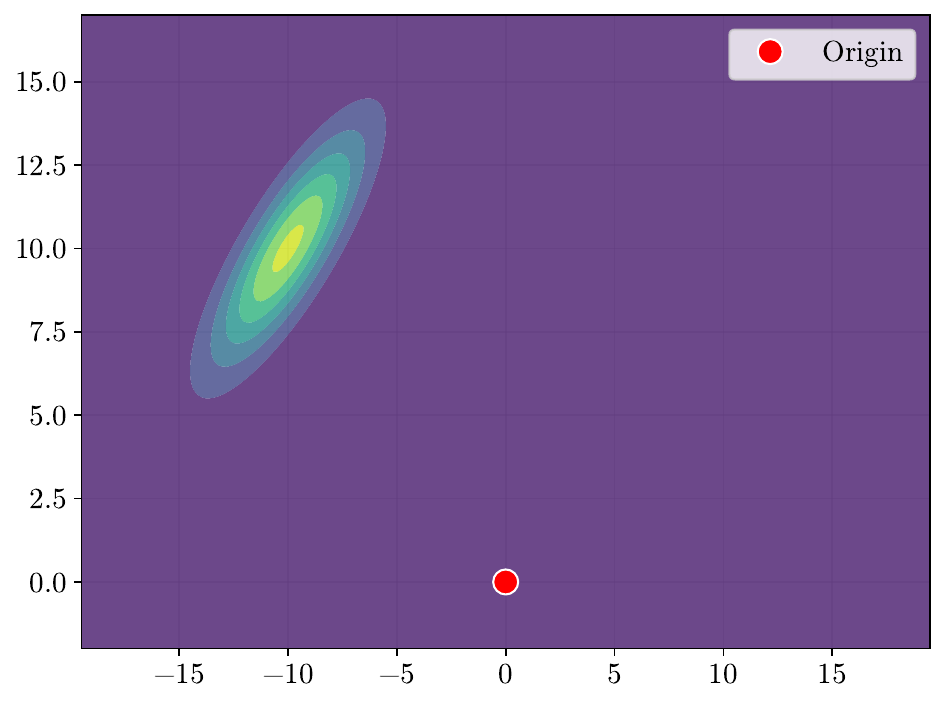}}            \\[1em]

                \raisebox{3em}{\footnotesize (b) $\IGW$:}                                      &
                \redboxed{\includegraphics[scale=0.15]{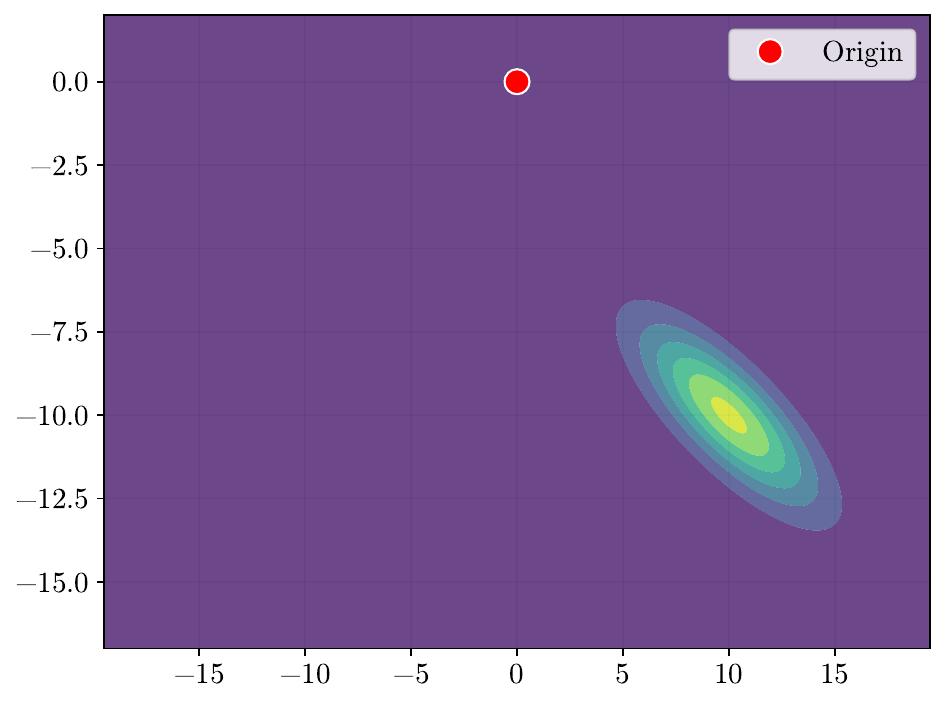}} &
                \redboxed{\includegraphics[scale=0.15]{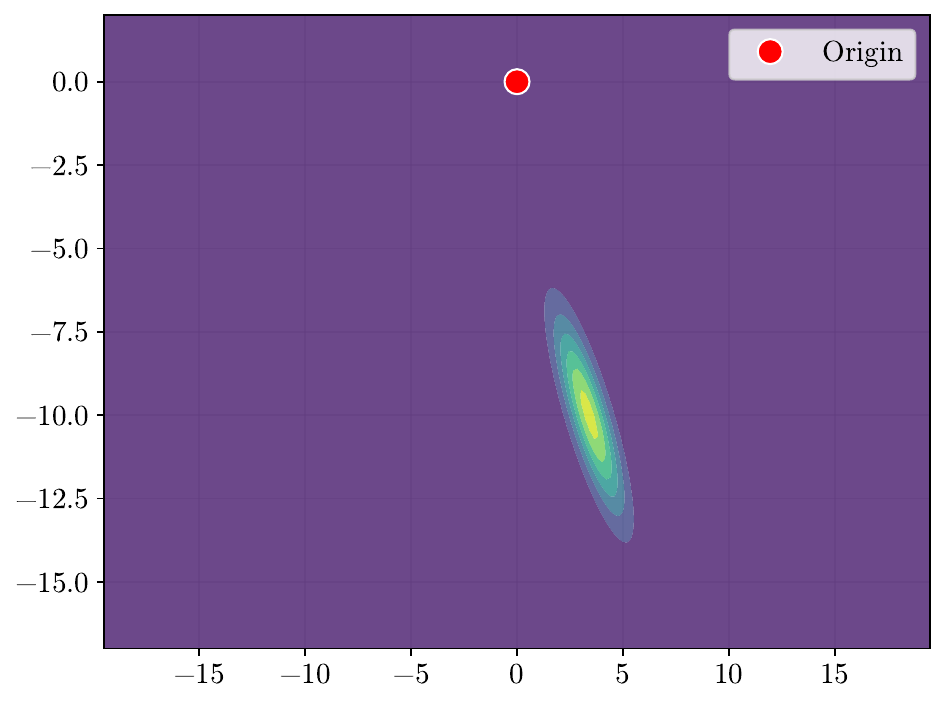}} &
                \redboxed{\includegraphics[scale=0.15]{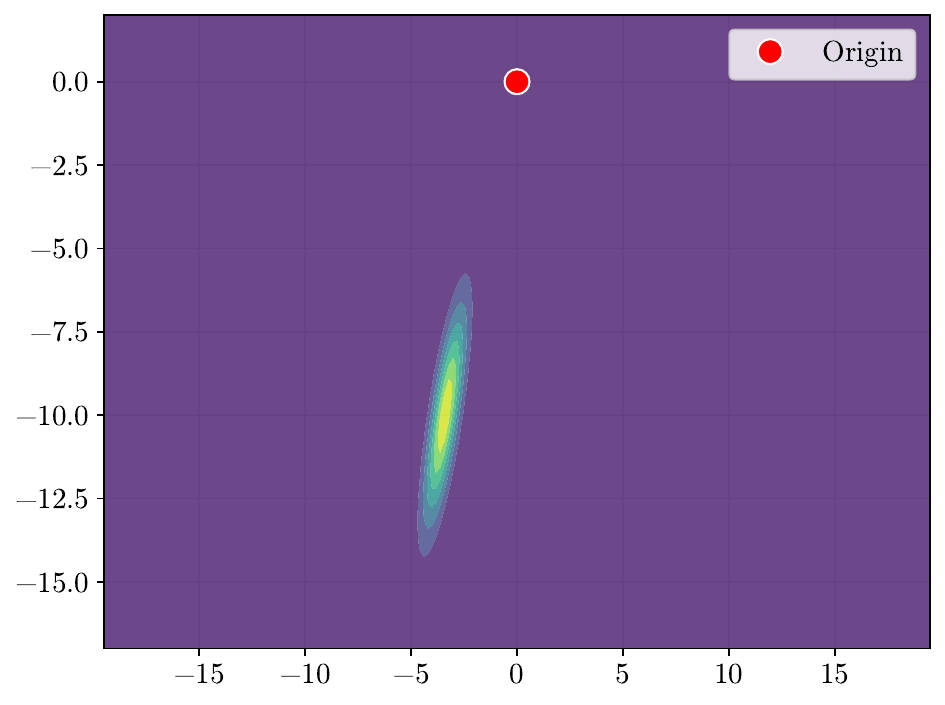}} &
                \redboxed{\includegraphics[scale=0.15]{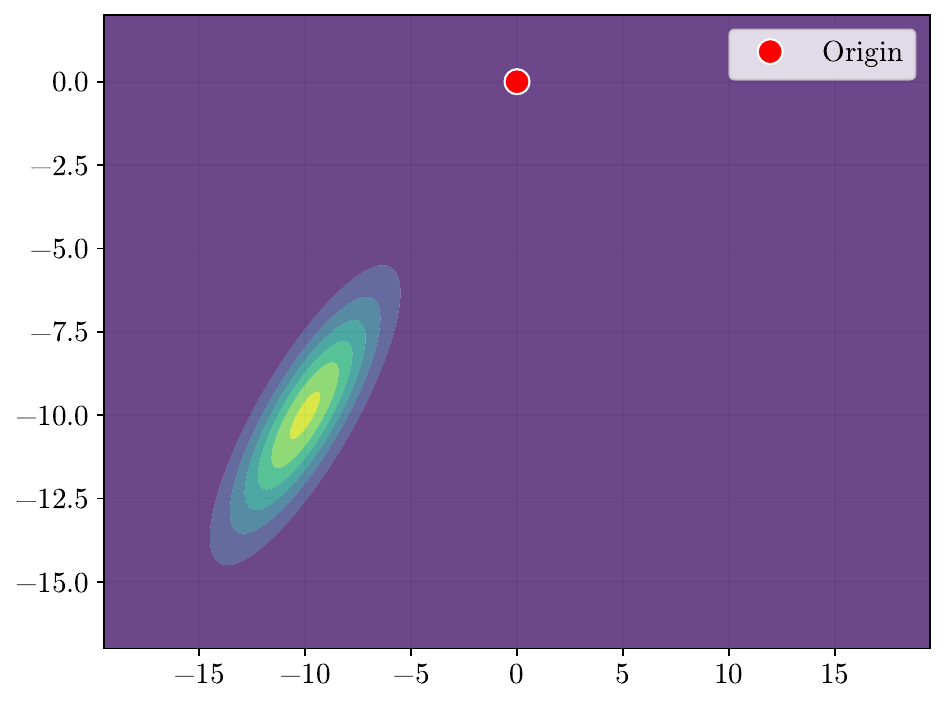}}   \\

                \raisebox{3em}{\footnotesize (b) $\mr{W}_2$:}                                  &
                \redboxed{\includegraphics[scale=0.15]{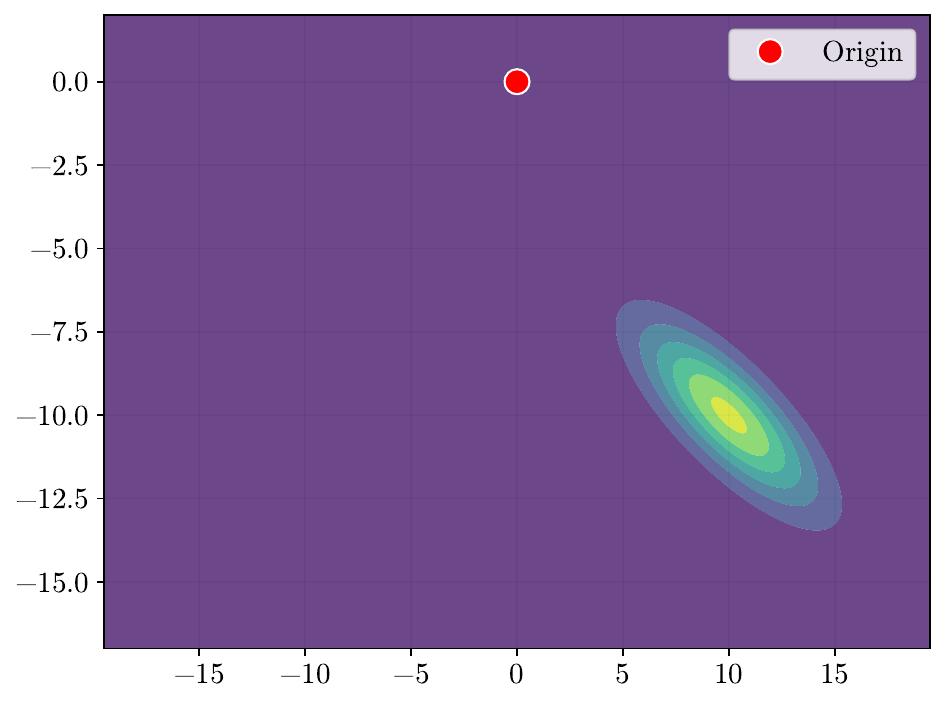}}  &
                \redboxed{\includegraphics[scale=0.15]{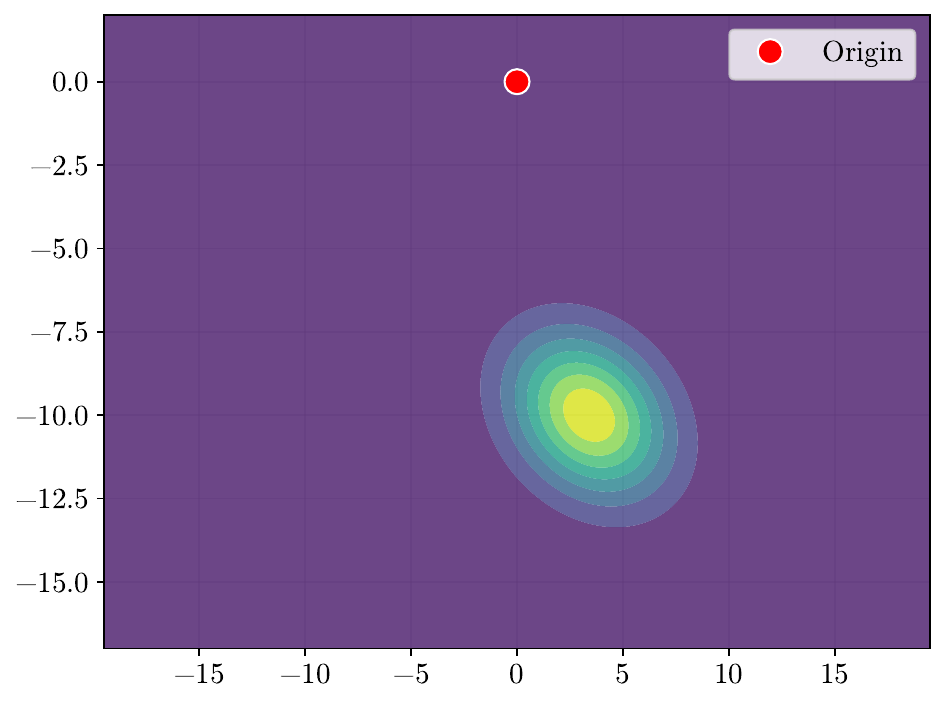}}  &
                \redboxed{\includegraphics[scale=0.15]{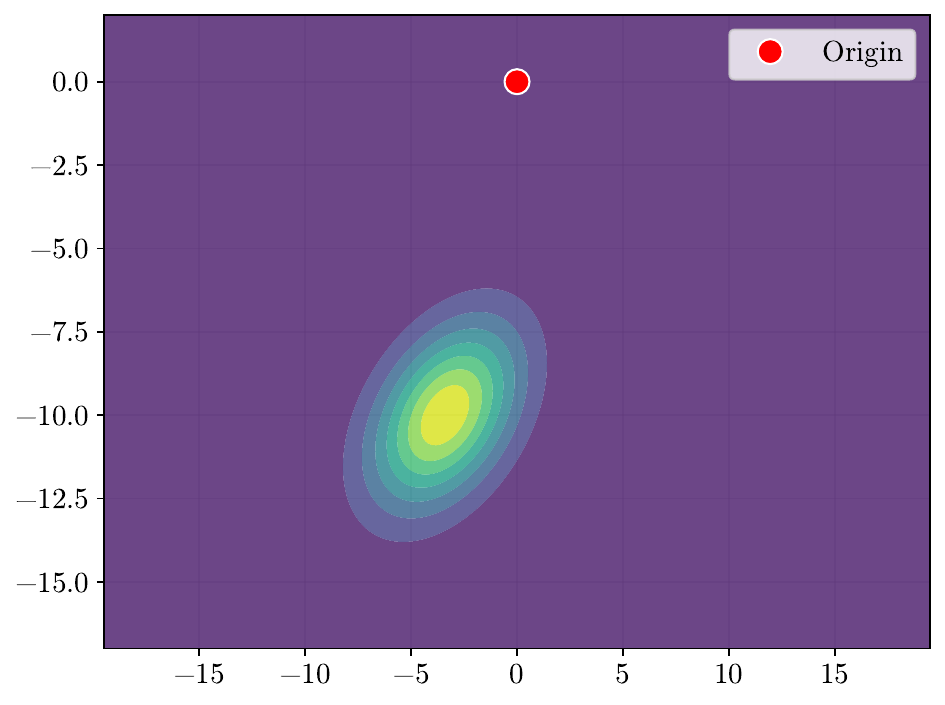}}  &
                \redboxed{\includegraphics[scale=0.15]{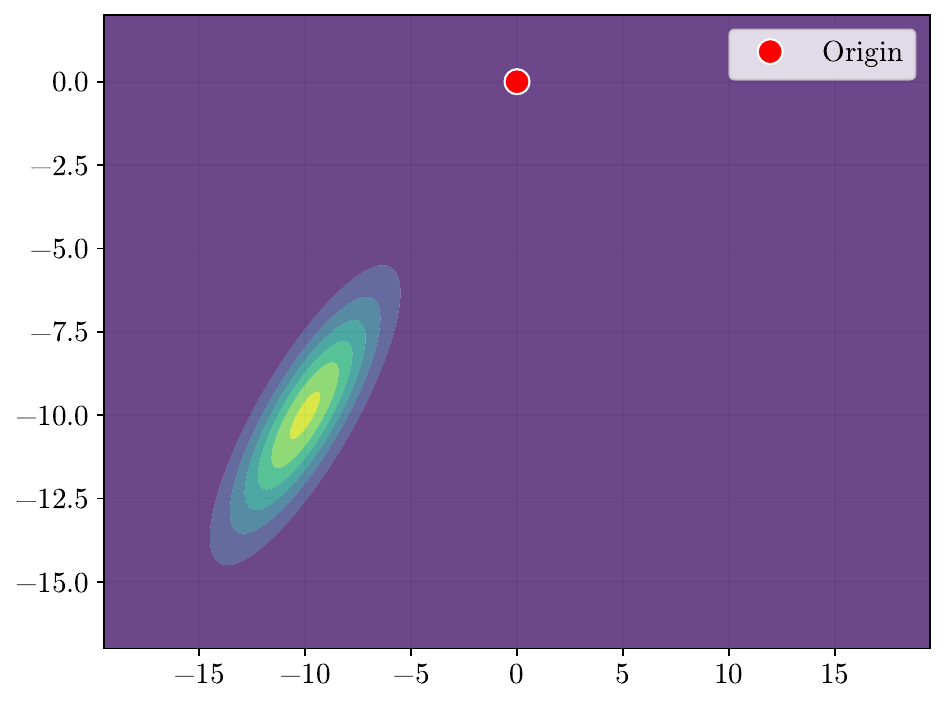}}    \\

                                                                                               &
                {\footnotesize $t = 0$}                                                        &
                {\footnotesize $t = 0.33$}                                                     &
                {\footnotesize $t = 0.67$}                                                     &
                {\footnotesize $t = 1$}
            \end{tabular}}

        \vspace*{-.5em}
        \caption{\textit{Comparison of displacement interpolations between two Gaussian distributions (origin marked in red). (a) has the origin at the bottom and (b) has the origin at the top; the distributional orientation along the IGW interpolation depends on the origin, while the $\mr W_2$ interpolation is invariant to translating both measures. Particles follow straight segments in both interpolations. Top row of (a) and (b): IGW displacement interpolation between two Gaussians with OT map estimated using RGD. Bottom row of (a) and (b): 2-Wasserstein displacement interpolation with OT map from the Bures-Wasserstein formula in \eqref{eq:gaussian-ot-map}.}}
        \label{fig:interpolation-comparison}
        \vspace*{-0.5em}
    \end{figure}
\end{example}

\subsection{IGW Barycenter Between Centered Gaussians}

In this section, we generalize the formula of Le et al.~\cite{le2022entropic} for the IGW barycenter between centered Gaussian measures. First, we define the notion of an IGW barycenter as any Frech\'et mean under the IGW distance (analogously to the 2-Wasserstein barycenter of Agueh and Carlier~\cite{agueh2011barycenters}).

\begin{definition}[IGW barycenter]{} \label{def:igw-barycenter}
    Suppose that $\rho \in \mc{P}(\mc{P}_2(\mc{H}))$, where $\mc P_2(\mc H)$ is equipped with the $\mr W_2$ topology and its Borel $\sigma$-algebra; supports of measures on this space are taken in this topology. A measure $\bar{\mu} \in \mc{P}_2(\mc{H})$ is a \emph{$\rho$-weighted IGW barycenter} if
    \begin{align*}
        \bar{\mu} \in \argmin_{\mu \in \mc{P}_2(\mc{H})}\; \set*{\int \IGW(\mu, \nu)^2\, d\rho(\nu)}.
    \end{align*}
\end{definition}

Using this definition, we can solve the IGW barycenter problem when $\rho$ is supported on the space $\mc{G}_0(\mc{H})$ of centered Gaussian measures over $\mc{H}$. We have the following representation.

\begin{proposition}[Gaussian IGW barycenter] \label{prop:igw-gaussian-barycenter}
    Suppose that $\rho \in \mc{P}(\mc{P}_2(\mc{H}))$ has $\supp(\rho) \subseteq \mc{G}_0(\mc{H})$ and $\int \tr(\Sigma_\nu)^2\, d\rho(\nu) < \infty$. For any fixed total orthonormal set $\set{e_k}_{k \in \N}$ of $\mc{H}$, the measure $\Normal(0, \Sigma)$ with covariance operator
    \begin{align*}
        \Sigma \coloneq\sum_{k=1}^\infty \left( \int \lambda_k(\Sigma_\nu)\, d\rho(\nu) \right) e_k \otimes e_k,
    \end{align*}
    is a $\rho$-weighted IGW barycenter.
\end{proposition}

A special case of this result occurs when $\rho$ is a finitely-supported measure over centered Gaussian measures, each over a Euclidean space. Restricting the barycenter to lie in $\mc{P}_2(\R^d)$ for some $d \in \N$, we recover the formula of Le et al.~\cite{le2022entropic}.

\begin{remark}[Finite-dimensional IGW barycenter]
    In the setting of \Cref{prop:igw-gaussian-barycenter}, suppose that $\supp(\rho) = \set{\mu_1, \ldots, \mu_p}$, where $\mu_i = \Normal(0, \Sigma_i)$ and $\rank(\Sigma_i) = d_i$ for $i \in [p]$. Further, suppose that we restrict the optimization in \Cref{def:igw-barycenter} to be over $\mu \in \mc{P}_2(\R^d)$ instead of over all of $\mc{P}_2(\mc{H})$; this is called a \emph{fixed-support barycenter} in Beier et al.~\cite{beier2023multi}. Define $d_{\mu_i} \coloneq d_i$ for $i \in [p]$. Then, $\Normal(0, \Sigma)$ is a $\rho$-weighted fixed-support IGW barycenter, where
    \begin{align*}
        \Sigma \coloneq \diag\left( \int \lambda_1(\Sigma_\nu)\, \indicator_{1 \leq d_\nu}\, d\rho(\nu),\, \dots,\, \int \lambda_d(\Sigma_\nu)\, \indicator_{d \leq d_\nu}\, d\rho(\nu) \right).
    \end{align*}
    Setting $d \geq \max_{i \in [p]} d_i$ recovers the same value of the objective function as the arbitrary-support barycenter from \Cref{prop:igw-gaussian-barycenter}.
\end{remark}

To elucidate the action of the IGW barycenter found in \Cref{prop:igw-gaussian-barycenter}, we conclude this section with the following example.

\begin{example}[IGW barycenter preserves geometry]
    Suppose we have input measures $\mu_i = \Normal(0, \Sigma_i)$ for $i \in [2]$ on $\R^2$, where
    \begin{align*}
        \Sigma_1 &
        \coloneq \begin{bmatrix}
                     \sqrt{3} / 2 & 1 / 2        \\
                     -1 / 2       & \sqrt{3} / 2
                 \end{bmatrix}
        \begin{bmatrix}
            100 & 0  \\
            0   & 10
        \end{bmatrix}
        \begin{bmatrix}
            \sqrt{3} / 2 & 1 / 2        \\
            -1 / 2       & \sqrt{3} / 2
        \end{bmatrix}^\intercal, \\
        \vspace{3mm}
        \Sigma_2 &
        \coloneq \begin{bmatrix}
                     \sqrt{2} / 2 & -\sqrt{2} / 2 \\
                     \sqrt{2} / 2 & \sqrt{2} / 2
                 \end{bmatrix}
        \begin{bmatrix}
            100 & 0  \\
            0   & 10
        \end{bmatrix}
        \begin{bmatrix}
            \sqrt{2} / 2 & -\sqrt{2} / 2 \\
            \sqrt{2} / 2 & \sqrt{2} / 2
        \end{bmatrix}^\intercal,
    \end{align*}
    and let $\rho = \Unif(\set{\mu_1, \mu_2})$. The $\rho$-weighted 2-Wasserstein barycenter (which can be computed as the solution of the fixed-point equation in Theorem 6.1 of Agueh and Carlier~\cite{agueh2011barycenters}) does not preserve the shape of the input measures, but the $\rho$-weighted IGW barycenter (computed from \Cref{prop:igw-gaussian-barycenter}) does. We display contour plots of both barycenters in \Cref{fig:igw-barycenter} to illustrate this.

    The 2-Wasserstein barycenter minimizes a weighted combination of 2-Wasserstein distances to the $\mu_i$, even if the shape of the barycenter needs to be completely different than those of the input measures. On the other hand, the IGW barycenter is orthogonally invariant and incurs no cost from rotations. In this sense, the IGW barycenter is an average with respect to the task of \emph{alignment}, whereas the 2-Wasserstein barycenter is an average with respect to the task of \emph{transport}.

    \begin{figure}[htbp]
        \centering
        \subfloat[\centering Density of $\mu_1$]{\includegraphics[scale=0.2]{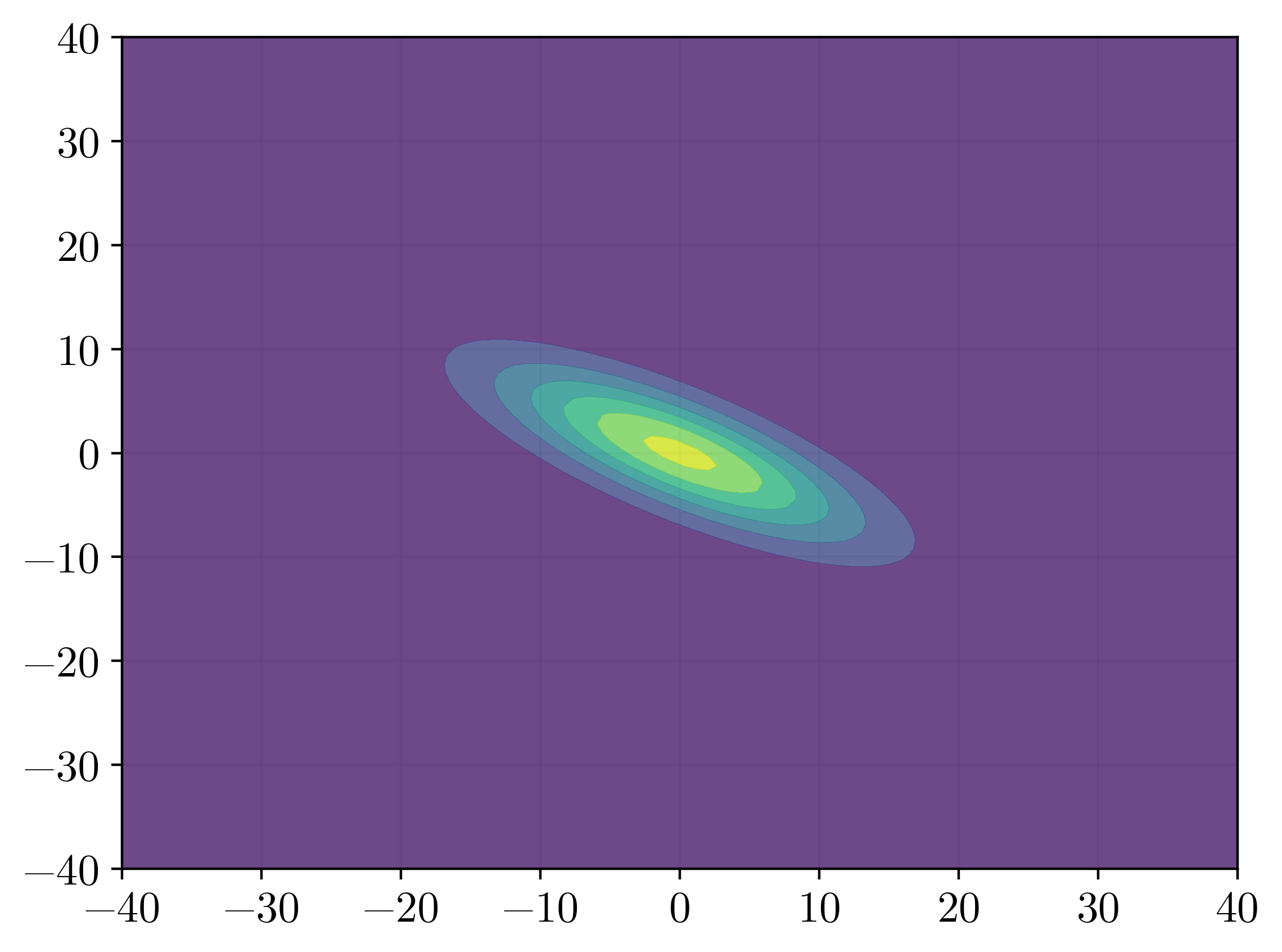}}
        \quad
        \subfloat[\centering Density of $\mu_2$]{\includegraphics[scale=0.2]{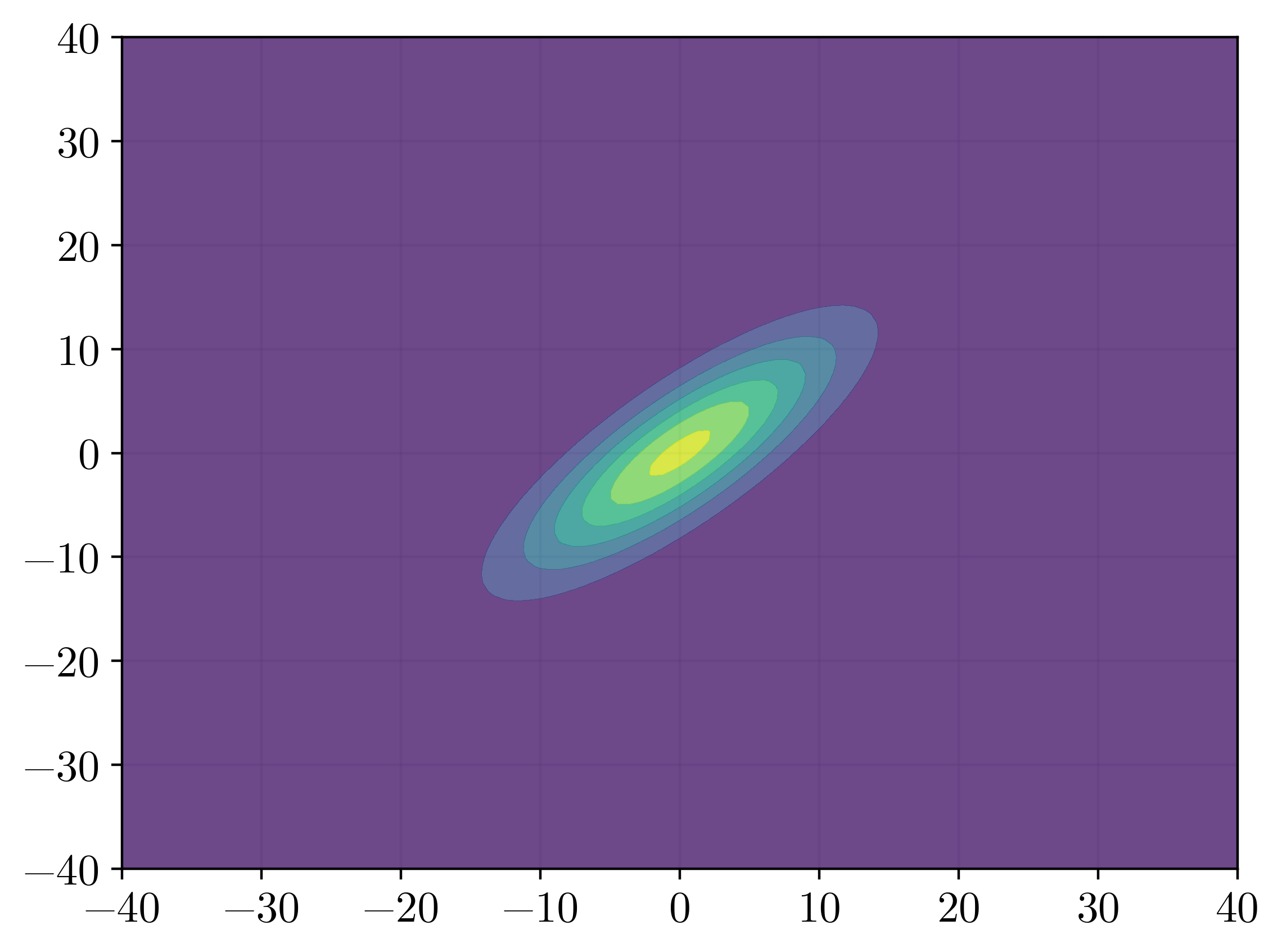}}
        \quad
        \subfloat[\centering $\mr{W}_2$ barycenter]{\includegraphics[scale=0.2]{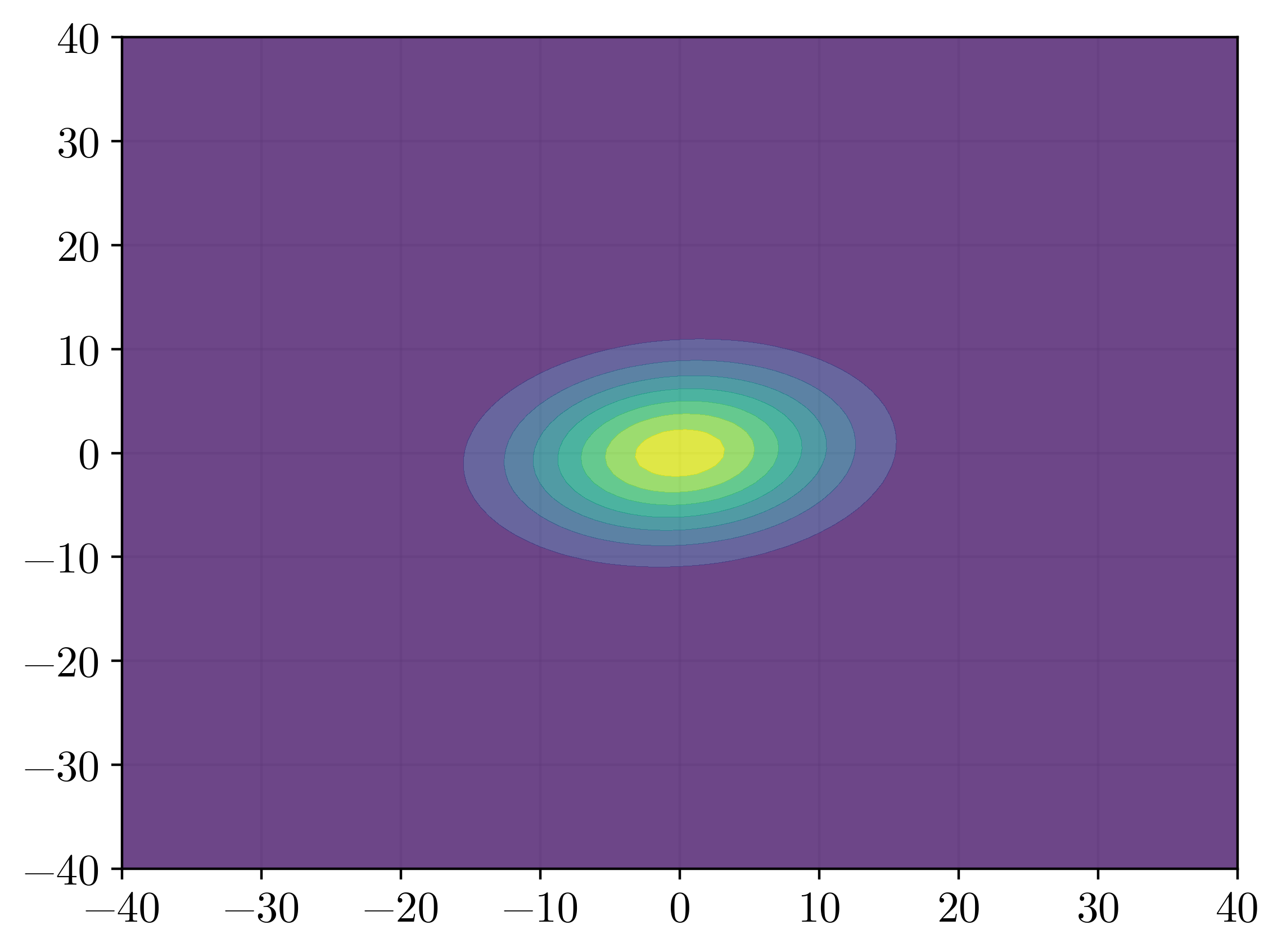}}
        \quad
        \subfloat[\centering IGW barycenter]{\includegraphics[scale=0.2]{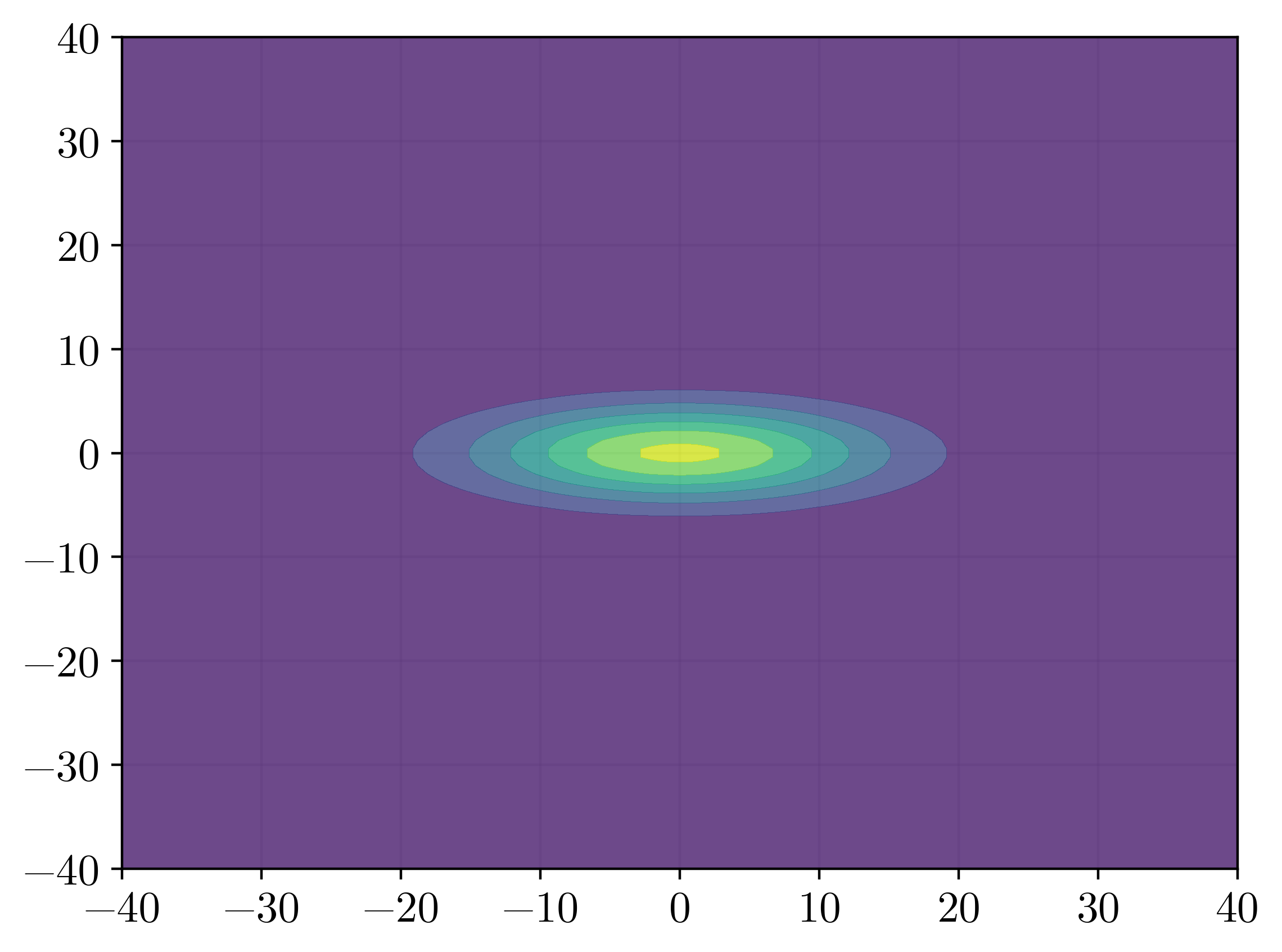}}
        \vspace*{-.5em}
        \caption{\textit{Contour plots of the $\rho$-weighted 2-Wasserstein barycenter and $\rho$-weighted IGW barycenter. The 2-Wasserstein barycenter does not preserve the covariance structure of the input measures, but the IGW barycenter naturally does.}}
        \label{fig:igw-barycenter}
        \vspace*{-.5em}
    \end{figure}
\end{example}

While the results in this section pertained to the bi-marginal IGW problem, we now discuss the multimarginal OT problem and derive new results, under both the IGW cost and under the 2-Wasserstein cost.

\section{Multimarginal Problems Between Gaussian Measures} \label{sec:multimarginal}

In this section, we consider a multimarginal version of the IGW problem, as well as the related multimarginal 2-Wasserstein problem introduced in Gangbo and {\'S}wi{\k{e}}ch~\cite{gangbo1998optimal}, between Gaussian measures. All proofs are deferred to \Cref{sec:proofs}.

\subsection{Multimarginal IGW}

Given $\mu_1, \ldots, \mu_p \in \mc{P}_2(\mc{H})$ for a separable Hilbert space $\mc{H}$, the multimarginal IGW problem is
\begin{align*}
    \inf_{\pi \in \Pi(\mu_1, \ldots, \mu_p)} \quad \int \sum_{i=1}^p \sum_{j=i+1}^p \left(\inner{x_i, x_i^\prime} - \inner{x_j, x_j^\prime}\right)^2\, d(\pi \otimes \pi)(x_1, \dots, x_p, x_1^\prime, \ldots, x_p^\prime).
\end{align*}
Note that the cost is now the sum of pairwise bi-marginal IGW costs, which is analogous to the multimarginal 2-Wasserstein problem of Gangbo and {\'S}wi{\k{e}}ch~\cite{gangbo1998optimal} shown in \eqref{eq:multimarginal-ot}. When the inputs to the problem are centered Gaussian measures, we obtain the following closed-form solution.

\begin{theorem}[Multimarginal IGW between Gaussians]\label{thm:multimarginal-igw}
    
    Suppose $\mu_i=\Normal(0,\Sigma_i)$ are centered Gaussian measures on a separable Hilbert space $\mc H$, for $i\in[p]$. Let $v_{ik}$ be orthonormal eigenvectors corresponding to the positive eigenvalues of $\Sigma_i$ in nonincreasing order, and let $(Z_k)_{k\geq1}$ be shared independent standard Gaussian variables. Then the joint law of
    \begin{align*}
        X_i=\sum_{k:\lambda_k(\Sigma_i)>0}\sqrt{\lambda_k(\Sigma_i)}\,Z_kv_{ik},\qquad i\in[p],
    \end{align*}
    solves the multimarginal IGW problem between $\mu_1,\ldots,\mu_p$. Each sum converges in $L^2(\mc H)$, and terminates for finite-rank covariances.
\end{theorem}

{
The construction aligns the positive covariance spectra and uses the same Gaussian variable for corresponding eigenvectors. Every pairwise marginal is IGW-optimal, so, as a direct consequence of the centered spectral characterization, the joint coupling attains the sum of the pairwise minima. In \Cref{app:multimarginal-igw-barycenter}, we recover the corresponding finite-support Gaussian barycenter formula. Next, we discuss multimarginal OT under the quadratic Wasserstein cost.
}

\subsection{Multimarginal Optimal Transport} \label{sec:multimarginal-gaussian}

Recall the multimarginal OT problem of Gangbo and {\'S}wi{\k{e}}ch~\cite{gangbo1998optimal} with pairwise quadratic cost:
\begin{equation} \label{eq:multimarginal-ot_section}
    \inf_{\pi \in \Pi(\mu_1, \dots, \mu_p)}\, \int \sum_{i=1}^p \sum_{j=i+1}^p \norm{x_i - x_j}_2^2\, d\pi(x_1, \dots, x_p).
\end{equation}
For Gaussians, we can rewrite the multimarginal OT problem with pairwise quadratic costs from \eqref{eq:multimarginal-ot} as a semidefinite program. Exploring its structure reveals that it can be solved efficiently via a rank-deficiency constraint.

\begin{theorem}[Multimarginal OT between Gaussians] \label{thm:multimarginal-gaussian}
    Suppose that $\mu_i = \Normal(m_i, \Sigma_i)$ are Gaussian measures on $\R^d$ with $\Sigma_i \succ 0$, for $i \in [p]$. The optimal multimarginal transport plan between $\mu_1, \ldots, \mu_p$ for the problem \eqref{eq:multimarginal-ot_section} is given by the Gaussian measure
    \begin{align*}
        \Normal\left( \begin{bmatrix}
                          m_1 \\ m_2 \\ \vdots \\ m_p
                      \end{bmatrix},\, \begin{bmatrix}
                                           \Sigma_1         & C_{12}           & \cdots & C_{1p}   \\
                                           C_{12}^\intercal & \Sigma_2         & \cdots & C_{2p}   \\
                                           \vdots           & \vdots           & \ddots & \vdots   \\
                                           C_{1p}^\intercal & C_{2p}^\intercal & \cdots & \Sigma_p
                                       \end{bmatrix} \right),
    \end{align*}
    where $\{C_{ij}\}_{i\neq j}\subset \R^{d \times d}$ solve the semidefinite program
    \begin{align*}
        \max_{\{C_{ij}\}} & \quad \sum_{i=1}^p \sum_{j=i+1}^p \tr(C_{ij})                                 \\
        \text{s.t.}       & \quad \Sigma \coloneq \begin{bmatrix}
                                                      \Sigma_1         & C_{12}           & \cdots & C_{1p}   \\
                                                      C_{12}^\intercal & \Sigma_2         & \cdots & C_{2p}   \\
                                                      \vdots           & \vdots           & \ddots & \vdots   \\
                                                      C_{1p}^\intercal & C_{2p}^\intercal & \cdots & \Sigma_p
                                                  \end{bmatrix} \succeq 0.
    \end{align*}
    The solution of this problem exists, is unique, and satisfies $\rank(\Sigma) = d$.
\end{theorem}

\Cref{thm:multimarginal-gaussian} implies that the optimal multimarginal Gaussian plan---a probability measure over $\R^{dp}$---is obtained by pushing the isotropic Gaussian measure in $\R^d$ through $p$ affine functions $\ell_1,\ldots,\ell_p:\R^d \to \R^d$. Writing $\ell_i(z)=m_i+B_i z$, the optimal plan is $(\ell_1, \ldots, \ell_p)_\sharp\,\mc{N}(0,I)\in\Pi(\mu_1,\ldots,\mu_p)$.

To prove \Cref{thm:multimarginal-gaussian}, we expand the cost function and argue that the optimal coupling can be chosen to be jointly Gaussian since the optimization problem only depends on the centered covariance of $\pi$. The theorem of Gangbo and {\'S}wi{\k{e}}ch~\cite{gangbo1998optimal} gives uniqueness and represents the optimal coupling as the joint law of $(X_1,T_2(X_1),\ldots,T_p(X_1))$ for measurable maps $T_i$. Joint Gaussianity implies that $T_i(X_1)=\E[X_i\mid X_1]$ is affine in $X_1$ almost surely. Since $\Sigma_1\succ0$, the coupling is therefore supported on a $d$-dimensional affine subspace of $\R^{dp}$.

\begin{remark}[Weighted multimarginal OT]{} \label{rem:weighted-multimarginal}
    
    \Cref{thm:multimarginal-gaussian} extends directly to the $\rho$-weighted multimarginal problem \eqref{eq:weighted-multimarginal-ot} with weights $\rho_1, \dots, \rho_p > 0$. Indeed, since the marginal constraints fix $\int \norm{x_i}_2^2\, d\pi(x_1, \dots, x_p)$ for every $i \in [p]$, the identity
    \begin{align*}
        \rho_i\, \rho_j \norm{x_i - x_j}_2^2 - \norm{\rho_i x_i - \rho_j x_j}_2^2
        = \rho_i (\rho_j - \rho_i) \norm{x_i}_2^2 + \rho_j (\rho_i - \rho_j) \norm{x_j}_2^2
    \end{align*}
    shows that, after identifying couplings via the invertible map $(x_1, \dots, x_p) \mapsto (\rho_1 x_1, \dots, \rho_p x_p)$, the objective of \eqref{eq:weighted-multimarginal-ot} differs from that of the unweighted problem \eqref{eq:multimarginal-ot_section} between the rescaled measures $\nu_i \coloneq \Normal(\rho_i m_i,\, \rho_i^2 \Sigma_i)$ only by an additive constant. Applying \Cref{thm:multimarginal-gaussian} to the nondegenerate Gaussian measures $\nu_1, \dots, \nu_p$ and undoing the rescaling, we conclude that the $\rho$-weighted multimarginal problem between $\mu_1, \dots, \mu_p$ is solved by the Gaussian measure whose covariance $\Sigma$ solves the semidefinite program of \Cref{thm:multimarginal-gaussian} with the objective $\sum_{i=1}^p \sum_{j=i+1}^p \rho_i\, \rho_j \tr(C_{ij})$ and the same constraint; this solution again exists, is unique, and satisfies $\rank(\Sigma) = d$.
\end{remark}

\begin{example}[Visualization of multimarginal OT]
    We can further provide insight into the structure of the optimal multimarginal coupling through a representative example. Suppose that $p = 3$ and that $\mu_i = \Normal(0, \Sigma_i)$, $i \in [3]$, are centered Gaussians. We can solve the semidefinite program in \Cref{thm:multimarginal-gaussian} using a standard interior-point method \cite{diamond2016cvxpy} to obtain the optimal coupling $\pi$ between $\mu_1, \mu_2, \mu_3$. Although $\pi$ is a measure on $\R^6$, \Cref{thm:multimarginal-gaussian} implies that it is supported on a 2-dimensional linear subspace of $\R^6$. Therefore, we can visualize the nondegenerate part of $\pi$ by projecting onto this 2-dimensional subspace. We display such an example in \Cref{fig:mm-comparison}.

    \begin{figure}[htbp]
        \centering

        \begin{tabular}{@{}r@{\hspace{1em}}cccc@{}}
            \raisebox{3em}{\footnotesize Aligned:}             &
            \includegraphics[scale=0.15]{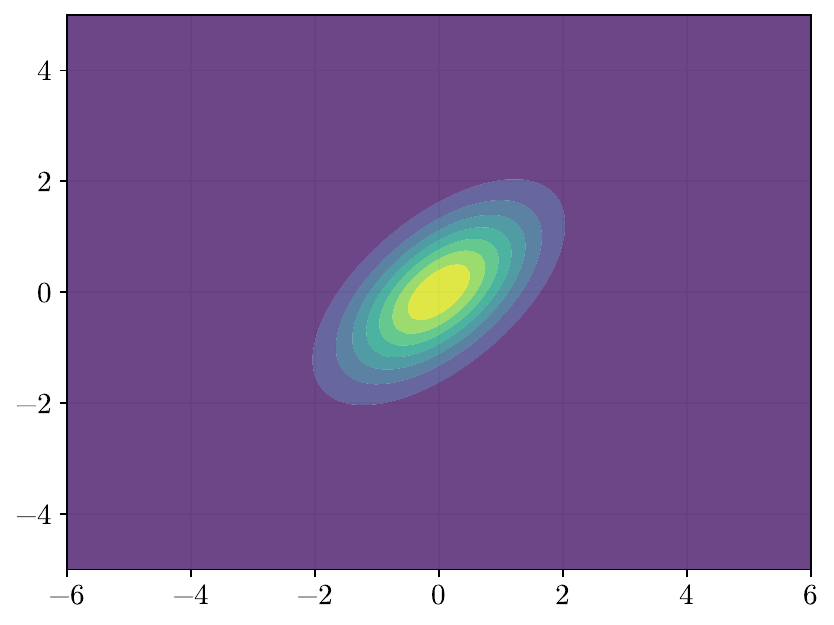}         &
            \includegraphics[scale=0.15]{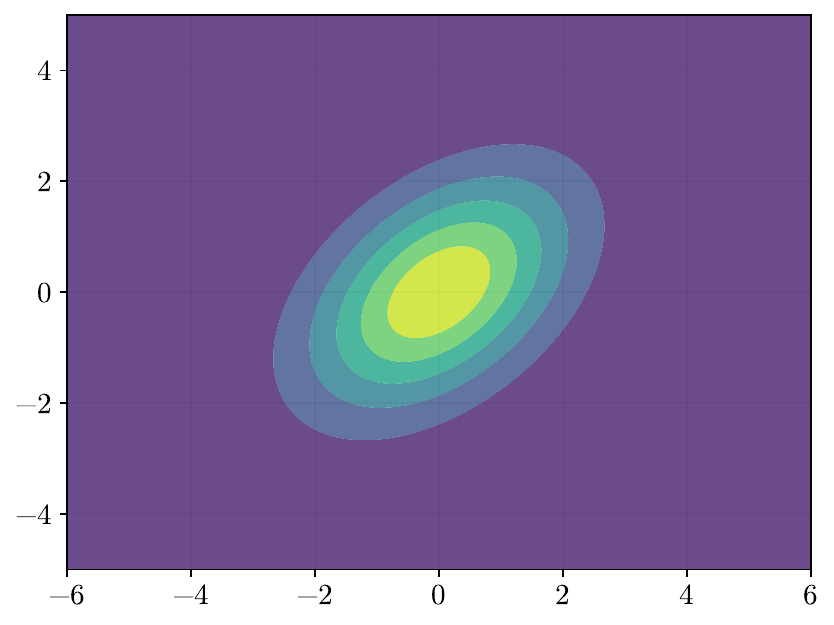}         &
            \includegraphics[scale=0.15]{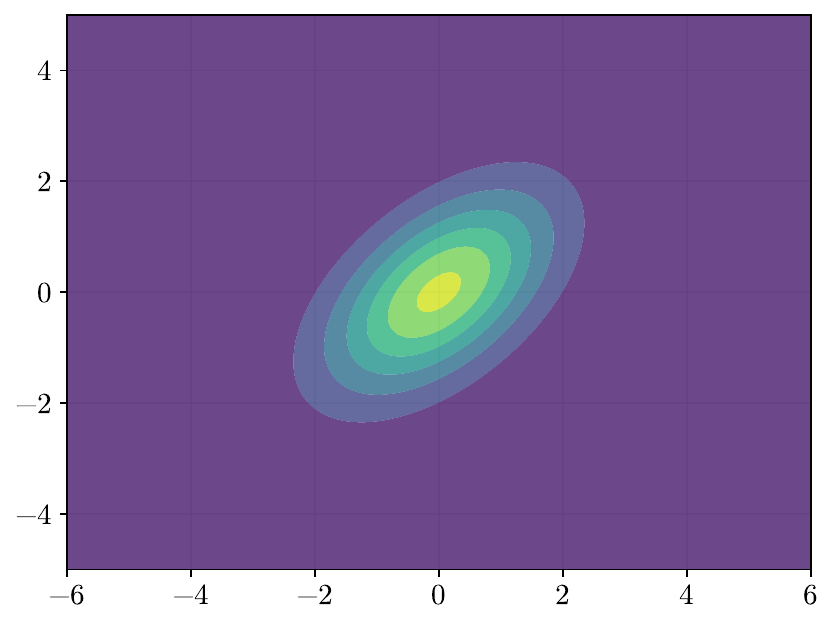}         &
            \includegraphics[scale=0.15]{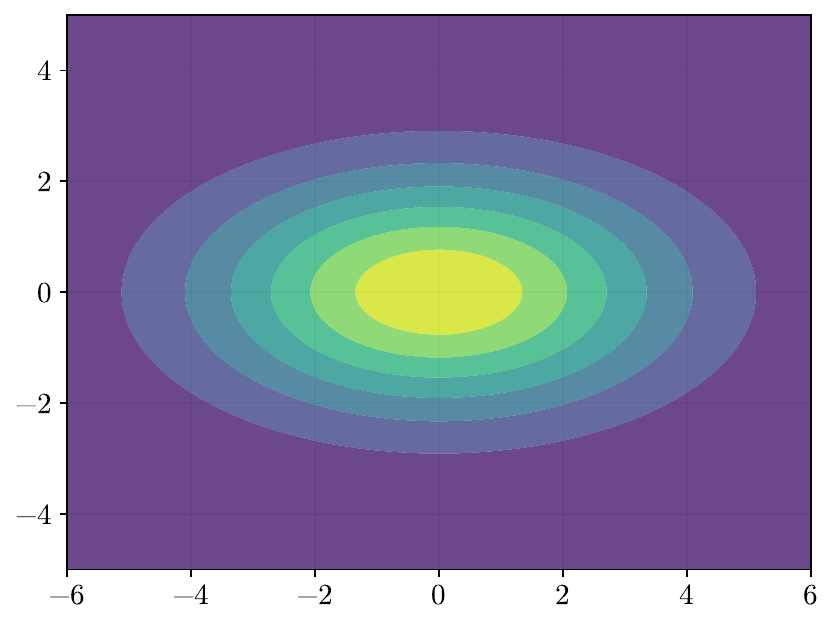}           \\

            \raisebox{3em}{\footnotesize Misaligned:}          &
            \includegraphics[scale=0.15]{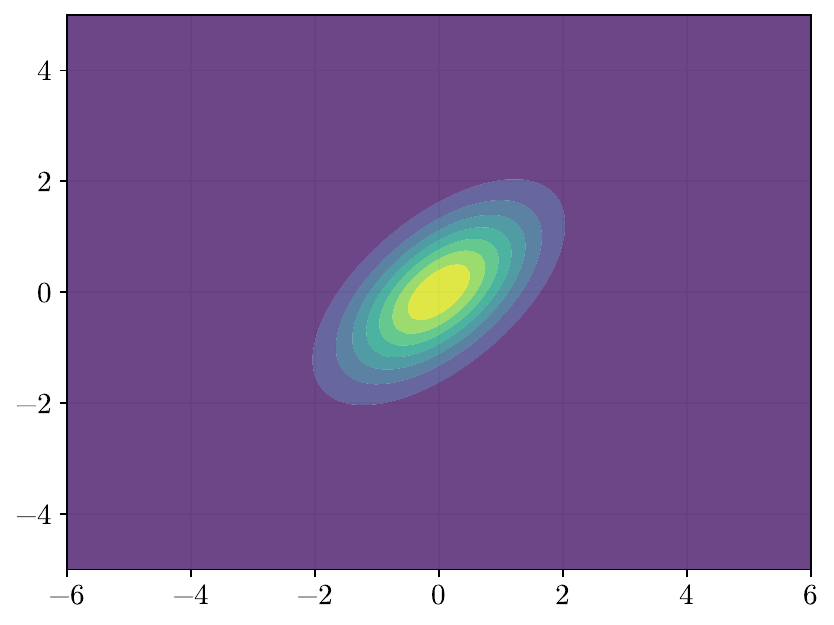}         &
            \includegraphics[scale=0.15]{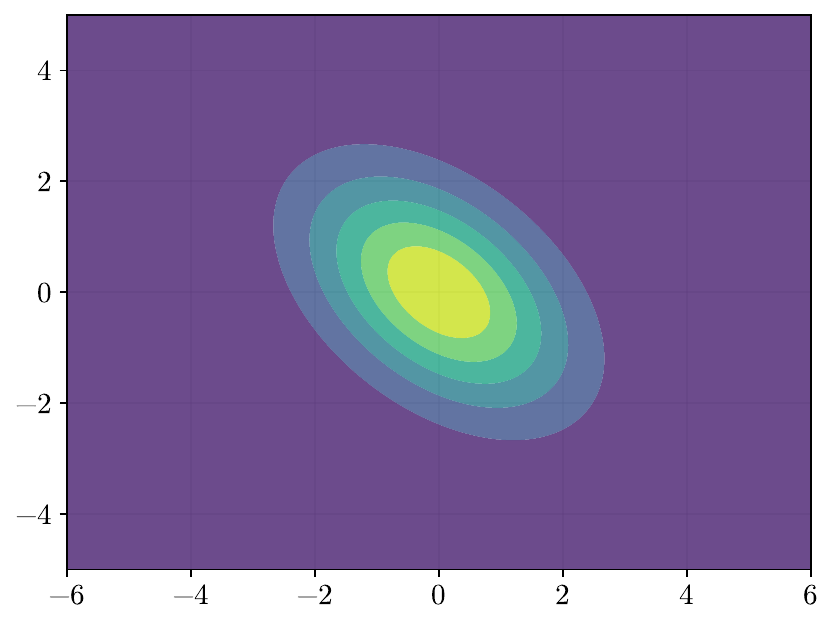}         &
            \includegraphics[scale=0.15]{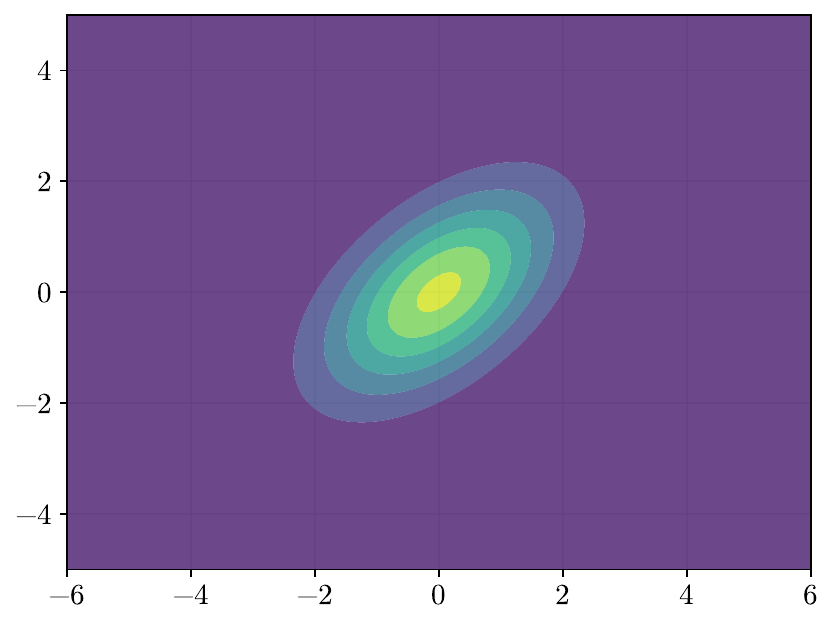}         &
            \includegraphics[scale=0.15]{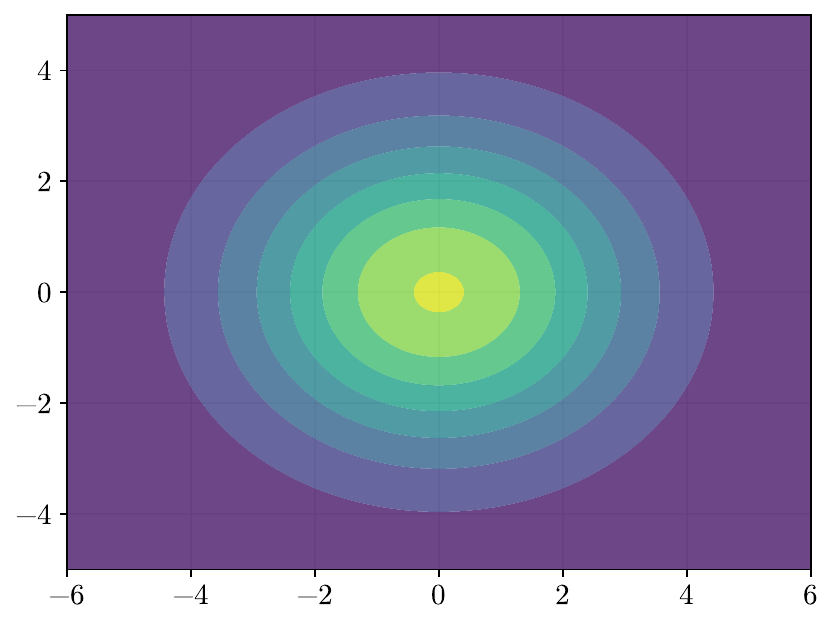}           \\

                                                               &
            {\footnotesize (a) $\mu_1 = \Normal(0, \Sigma_1)$} &
            {\footnotesize (b) $\mu_2 = \Normal(0, \Sigma_2)$} &
            {\footnotesize (c) $\mu_3 = \Normal(0, \Sigma_3)$} &
            {\footnotesize (d) Optimal coupling $\pi$}
        \end{tabular}

        \vspace*{-.5em}
        \caption{\textit{Comparison of multimarginal OT between two sets of Gaussian distributions. Top row: multimarginal OT between three aligned Gaussians. Bottom row: multimarginal OT between three misaligned Gaussians.}}
        \label{fig:mm-comparison}
        \vspace*{-0.5em}
    \end{figure}

    From \Cref{fig:mm-comparison}, we see that the optimal coupling $\pi$ can afford to be less spread out when the input measures are more aligned, so that even a ``flatter'' two-dimensional covariance structure can be orthogonally transformed in $\R^6$ to have the correct marginals. This is in contrast to the bottom row in \Cref{fig:mm-comparison}, where the input measures are not aligned, which forces the optimal multimarginal coupling to be more spread out.
\end{example}

\subsection{An Efficient Burer-Monteiro Algorithm for Multimarginal OT} \label{sec:multimarginal-bm}

A natural question following \Cref{thm:multimarginal-gaussian} is whether the rank deficiency of the solution can help design faster algorithms for solving the semidefinite program. We answer this question in the affirmative, by proposing one such approach based on the Burer-Monteiro factorization \cite{burer2003nonlinear}. Fix $d<k\leq dp$, fix weights $\rho_1, \dots, \rho_p > 0$, and consider the optimization problem
\begin{align*}
    \inf_{U = (U_1^\intercal, \dots, U_p^\intercal)^\intercal \in \R^{dp \times k}} & \quad -\sum_{i=1}^p \sum_{j=i+1}^p \rho_i\, \rho_j\, \tr(U_i U_j^\intercal) \\
    \text{s.t.}                                                                     & \quad U_i U_i^\intercal = \Sigma_i \quad \text{for } i \in [p]
\end{align*}
and let $\mc{M} \coloneq \set{U = (U_1^\intercal, \dots, U_p^\intercal)^\intercal \in \R^{dp \times k} : U_i U_i^\intercal = \Sigma_i \text{ for } i \in [p]}$ denote the search space, which is a smooth and compact Riemannian submanifold of $\R^{dp \times k}$ (note that $\mc{M}$ does not depend on the weights). By Corollary 8 of Boumal et al.~\cite{boumal2016non} (applied to the semidefinite program with the linear objective $\sum_{i < j} \rho_i \rho_j \tr(C_{ij})$ from \Cref{rem:weighted-multimarginal}), if $U = (U_1^\intercal, \dots, U_p^\intercal)^\intercal \in \R^{dp \times k}$ is a rank-deficient second-order stationary point of the above problem (as defined in \Cref{app:background-riemannian}), then $C_{ij} = U_i U_j^\intercal$ is a \emph{global} optimizer of the corresponding semidefinite program in \Cref{thm:multimarginal-gaussian} and \Cref{rem:weighted-multimarginal}. In particular, so long as we can find a rank-deficient second-order stationary point of the above problem, we can solve the $\rho$-weighted multimarginal OT problem between $\mu_1, \ldots, \mu_p$.

    {
        For the positive product weights considered here, a stronger statement holds: every second-order stationary point has rank $d$ whenever $k>d$. Thus choosing $k=d+1$ suffices. This conclusion follows from the second-order condition and the structure of the objective. The multimarginal solution then gives the Gaussian barycenter by \Cref{prop:multimarginal-wasserstein-barycenter}.
    }

\begin{theorem}[Burer-Monteiro for multimarginal OT] \label{thm:burer-monteiro-multimarginal}
    Let $\mu_i = \Normal(m_i, \Sigma_i)$ be Gaussian measures on $\R^d$ with $\Sigma_i \succ 0$ for $i \in [p]$, and let $\rho = \sum_{i=1}^p \rho_i\, \delta_{\mu_i}$ for weights $\rho_1, \dots, \rho_p > 0$ satisfying $\sum_{i=1}^p \rho_i = 1$. Suppose $U \in \R^{dp \times (d+1)}$ is any second-order stationary point of
    \begin{align*}
        \inf_{U = (U_1^\intercal, \dots, U_p^\intercal)^\intercal \in \R^{dp \times (d+1)}} & \quad -\sum_{i=1}^p \sum_{j=i+1}^p \rho_i\, \rho_j\, \tr(U_i U_j^\intercal) \\
        \text{s.t.}                                                                         & \quad U_i U_i^\intercal = \Sigma_i \quad \text{for } i \in [p].
    \end{align*}
    Then $\rank(U)=d$, and the Gaussian measure $\Normal\big((m_1^\intercal, \ldots, m_p^\intercal)^\intercal,\, UU^\intercal\big)$ solves the $\rho$-weighted multimarginal OT problem \eqref{eq:weighted-multimarginal-ot} between $\mu_1, \ldots, \mu_p$. Further, define $U_{\mu_i} \coloneq U_i$ for $i \in [p]$. Then, the $\rho$-weighted 2-Wasserstein barycenter is
    \begin{align*}
        \bar{\mu} \coloneq \Normal\left( \int m_\nu\, d\rho(\nu),\, \left( \int U_\nu\, d\rho(\nu) \right) \left( \int U_\nu\, d\rho(\nu) \right)^\intercal \right).
    \end{align*}
\end{theorem}

Though the original semidefinite program in \Cref{thm:multimarginal-gaussian} has $d^2 p^2$ optimization variables, the optimization problem in \Cref{thm:burer-monteiro-multimarginal} has only $(d + 1) dp$ variables, which is a significant reduction when $p$ is large. Now, finding second-order stationary points of the problem in \Cref{thm:burer-monteiro-multimarginal} can be done using standard optimization techniques. Riemannian trust-region methods provide approximate second-order stationarity guarantees under the conditions of Theorem 6.20 of Boumal~\cite{boumal2023introduction}. The exact second-order stationary points are globally optimal by the rank argument above. To assess numerical optimality, we report covariance-constraint errors, gradient norms, and gaps to SDP dual lower bounds in \Cref{app:experiments}. Therefore, \Cref{thm:burer-monteiro-multimarginal} immediately yields a method to efficiently compute the multimarginal OT between a large number of Gaussian measures. Furthermore, \Cref{thm:burer-monteiro-multimarginal} provides an alternative method for computing the 2-Wasserstein barycenter between Gaussian measures (equivalently, finding barycenters between positive definite matrices in the Bures-Wasserstein geometry).

\section{Experiments}\label{app:experiments}

In this section, we illustrate the practical utility of our results in several applications. First, we use IGW for unpaired distributional comparisons of already trained language-model representations. We also include numerical experiments which illustrate the effectiveness of our Burer-Monteiro approach to solving Gaussian multimarginal OT. Parameter settings, implementation details, and additional results are provided in \Cref{app:experimental-details}. All code for our experiments can be found in the supplementary material.

\subsection{IGW Metric for Representation Similarity of Language Models}\label{sec:representational-similarity}

A common choice of similarity metric is the centered kernel alignment (CKA) introduced by Kornblith et al.~\cite{kornblith2019similarity}. Linear CKA is invariant to orthogonal transformations and isotropic scaling, but its centering and normalization remove location and global scale information, which may be undesirable when the embedding origin and scale are informative. Uncentered IGW provides a complementary distributional comparison that retains this information and accommodates different embedding dimensions. CKA uses the correspondence between inputs, whereas IGW optimizes over all such correspondences. Gaussian IGW compares the means and covariances of embedding distributions, making it useful when the goal is to compare their overall structure, particularly when inputs are unpaired or embedding dimensions differ.

We compare the \ci{bert-base-uncased} teacher model and 23 of its already trained distillations \cite{turc2019well}, using 1000 samples each from \ci{amazon_polarity} \cite{muennighoff2022mteb} and \ci{ag_news} \cite{zhang2015character}. To employ our results for Gaussian IGW, we fit a Gaussian to the embeddings of each model on both datasets. We assess the Gaussian modeling approximation in the leading two principal components; these diagnostics are provided in \Cref{app:representation-details}. We compute the analytic upper and lower bounds from \Cref{thm:igw-gaussian-bound} and tighten the upper bound by Riemannian gradient descent (RGD), initialized with the diagonal signs in that theorem.

\Cref{fig:igw_distance_amazon_polarity-igw_distance_ag_news} reports distances, taking square roots of the squared-cost bounds. The largest absolute bracket widths are $0.285$ for \ci{amazon_polarity} and $0.359$ for \ci{ag_news}, at most $1.50\%$ and $3.65\%$ of the respective feasible upper bounds. These brackets certify numerical approximation quality for the fitted, regularized Gaussians. The computation, including both analytic bounds, takes at most 6.2 seconds per comparison on the local CPU.

\imagesbsred{0.45}{igw_distance_amazon_polarity}{\ci{amazon_polarity} IGW distances}{igw_distance_ag_news}{\ci{ag_news} IGW distances}{Upper and lower bounds on the IGW distance between the \ci{bert-base-uncased} model and its distillations on the (a) \ci{amazon_polarity} and (b) \ci{ag_news} datasets. The analytic upper bound is tightened by numerically optimizing over the Stiefel manifold. We see that some of the smaller distillations preserve the embedding distribution almost as well as the larger ones. The resulting upper and lower bounds are usually extremely tight in practice.}

\Cref{fig:cka_vs_igw_amazon_polarity-cka_vs_igw_ag_news} compares CKA with the feasible RGD upper bounds. We find a negative association on both datasets, stronger on \ci{amazon_polarity} ($r=-0.504$, $p=0.014$) and weaker on \ci{ag_news} ($r=-0.198$, $p=0.365$). This comparison highlights the relationship between paired representation similarity and the geometry of the overall embedding distributions.

\imagesbsred{0.45}{cka_vs_igw_amazon_polarity}{\ci{amazon_polarity}}{cka_vs_igw_ag_news}{\ci{ag_news}}{CKA between the \ci{bert-base-uncased} model and its distillations on the (a) \ci{amazon_polarity} and (b) \ci{ag_news} datasets, compared to the upper bound on the IGW distance obtained by RGD. Pearson correlation is $-0.504$ ($p=0.014$) for \ci{amazon_polarity} and $-0.198$ ($p=0.365$) for \ci{ag_news}.}

\subsection{Heterogeneous User Clustering}

Clustering based on text data, e.g., users of a social media platform based on their text posts, is ubiquitous across applications. A common approach represents each user as the mean embedding of their posts (mean-pooling) \cite{chen2018enhancing}, which discards covariance information. We therefore examine clustering based on covariance spectra, using the centered Gaussian IGW formula.

We construct a population of 120 synthetic users arranged into three interest groups: science, news, and entertainment. Each user is represented by 400 posts from the corresponding category. Discarding the mean information, we model each user as $\Normal(0,\Sigma_i)$, where $\Sigma_i$ is the empirical covariance of their embeddings. By the centered formula,
\begin{align*}
    \IGW(\Normal(0,\Sigma_i),\Normal(0,\Sigma_j))
    =\norm{\lambda(\Sigma_i)-\lambda(\Sigma_j)}_2.
\end{align*}
Thus the implementation is ordinary Euclidean $k$-means++ on ordered covariance eigenvalue vectors, with zero padding across dimensions. Cluster centers are arithmetic means of these vectors, exactly as in \Cref{prop:igw-gaussian-barycenter}.

The IGW clustering algorithm is flexible in allowing for input measures over Euclidean spaces of different dimensions. We embed each user via one of two LLMs based on a fair coin flip: either \ci{all-MiniLM-L6-v2} ($d=384$) \cite{wang2020minilm} or \ci{all-mpnet-base-v2} ($d=768$) \cite{song2020mpnet}. Such a situation could arise in practice if some users were using a previous version of an application with an older embedding model. Different dimensions prevent direct Euclidean comparison of the mean vectors without an additional alignment, whereas the IGW distance remains directly computable. The displayed clusters nearly recover the original groups for these two models (\Cref{fig:heterogeneous-experiment}).

\begin{figure}[htbp]
    \centering
    \subfloat[\centering Embedding models]{\includegraphics[scale=0.3]{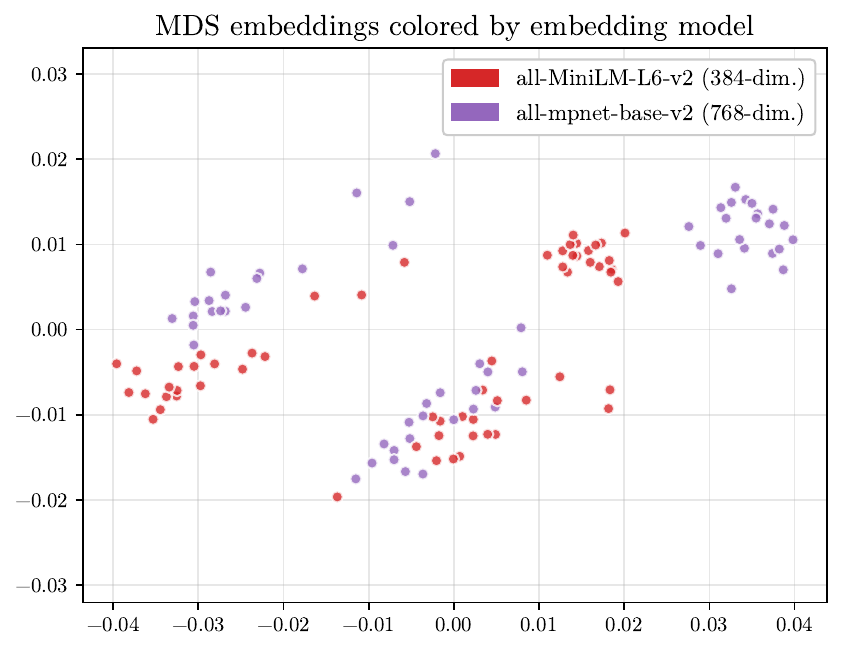} \label{fig:heterogeneous_mds_embeddings}}
    \quad
    \subfloat[\centering Ground truth]{\includegraphics[scale=0.3]{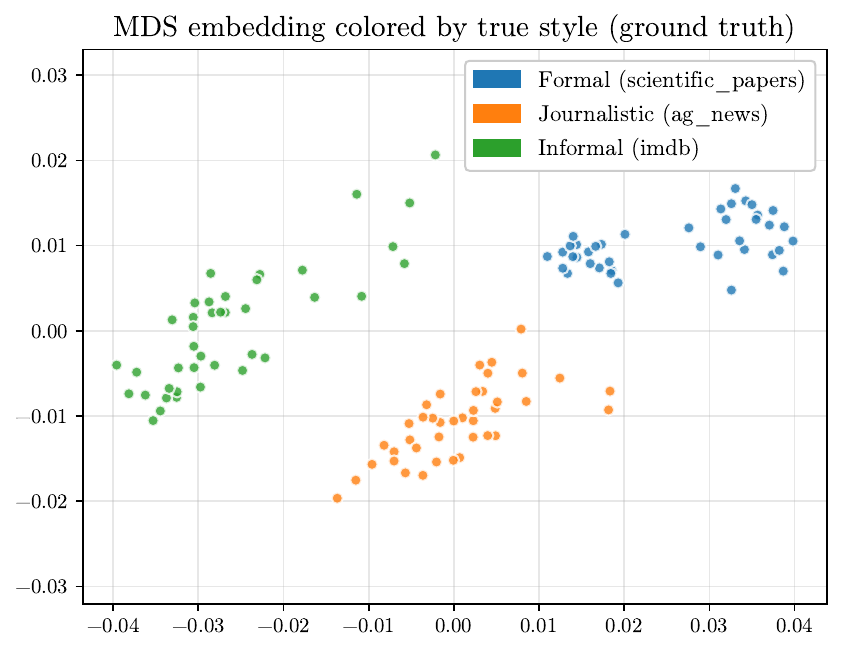} \label{fig:heterogeneous_true_styles}}
    \quad
    \subfloat[\centering $k$-means clusters]{\includegraphics[scale=0.3]{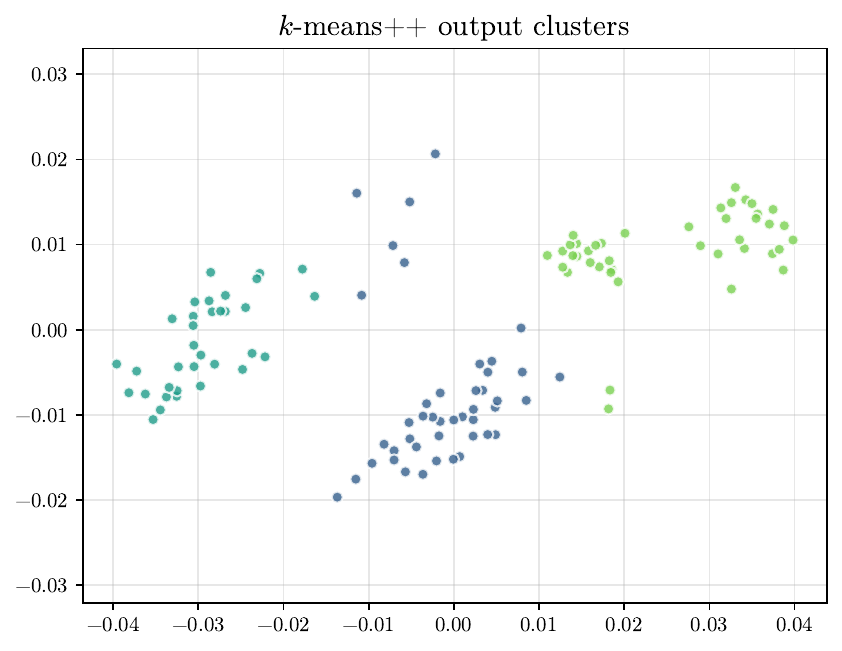} \label{fig:heterogeneous_kmeans_clusters}}
    \vspace*{-.5em}
    \caption{\textit{(a) 2-D MDS embedding of the IGW distance matrix between users, colored by the embedding model used for each user. (b) The same MDS embedding, colored by the ground truth interests of each user. (c) Euclidean $k$-means++ clustering of the zero-padded ordered covariance spectra of the heterogeneous users. The IGW distance is able to almost recover the original clusters, even only using heterogeneous covariance information.}}
    \label{fig:heterogeneous-experiment}
    \vspace*{-.5em}
\end{figure}

Using a common embedding model, we also compare mean-only, covariance-spectrum-only, and combined user representations over 10 repeated user samples. All three representations recover the category labels perfectly in every repeat (\Cref{tab:clustering-comparisons}). These results show that covariance spectra capture useful information about user interests. The sampling procedure and additional clustering results are given in \Cref{app:clustering-details}.

\subsection{Efficient Computation of Multimarginal OT}

We use the Burer-Monteiro approach of \Cref{thm:burer-monteiro-multimarginal} with $k=d+1$ and RGD to solve Gaussian multimarginal OT. We compare it with the semidefinite program modeled in \ci{cvxpy} \cite{diamond2016cvxpy} and solved with an interior-point conic solver. All experiments were run on a MacBook with an M1 Pro processor.

\Cref{fig:multimarginal-experiment} compares the methods for $p$ Gaussian measures on $\R^3$. The Burer-Monteiro approach is significantly faster than interior-point methods for large $p$, while attaining objective values that agree with the interior-point solutions to a relative error $|J_{\mathrm{BM}}-J_{\mathrm{SDP}}|/|J_{\mathrm{SDP}}|$ below $10^{-8}$, where $J=\sum_{i<j}\tr(C_{ij})$ is the maximized covariance-pairing objective. The interior-point solver exhausted the available memory for $p>50$, while the Burer-Monteiro approach solved the $p=100$ instance in about three seconds. This benchmark uses a pairwise-loop implementation; the faster timings in \Cref{app:optimization-details} use a vectorized implementation and a different stopping criterion, as detailed there. The Burer-Monteiro approach has significantly fewer optimization variables ($O(d^2p)$ compared to $O(d^2p^2)$), explaining its speedup and substantially lower memory footprint.

\begin{figure}[htbp]
    \centering
    \subfloat[\centering Computation times]{\redboxed{\includegraphics[scale=0.3]{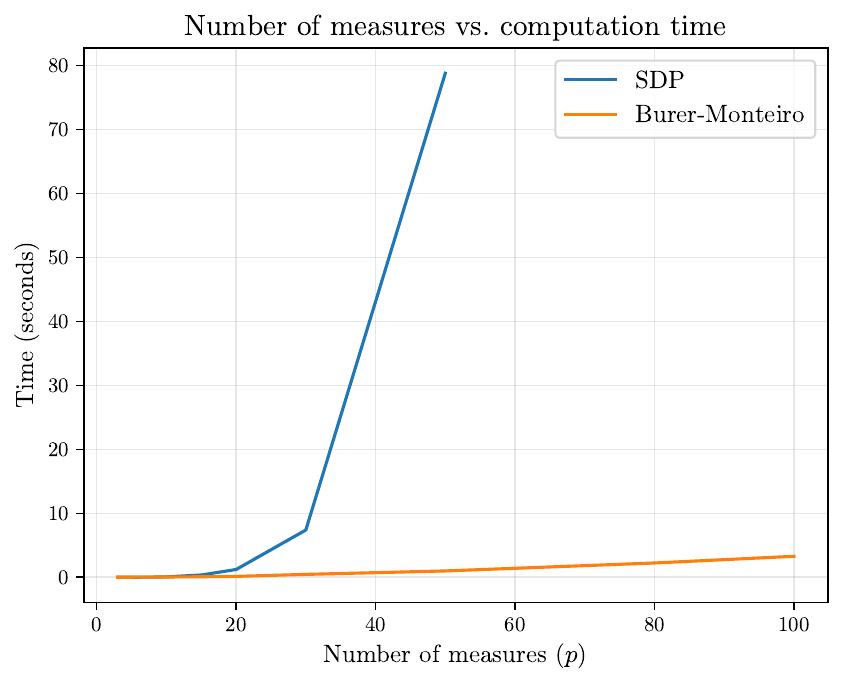}} \label{fig:computation-times}}
    \quad
    \subfloat[\centering Objective comparison]{\redboxed{\includegraphics[scale=0.3]{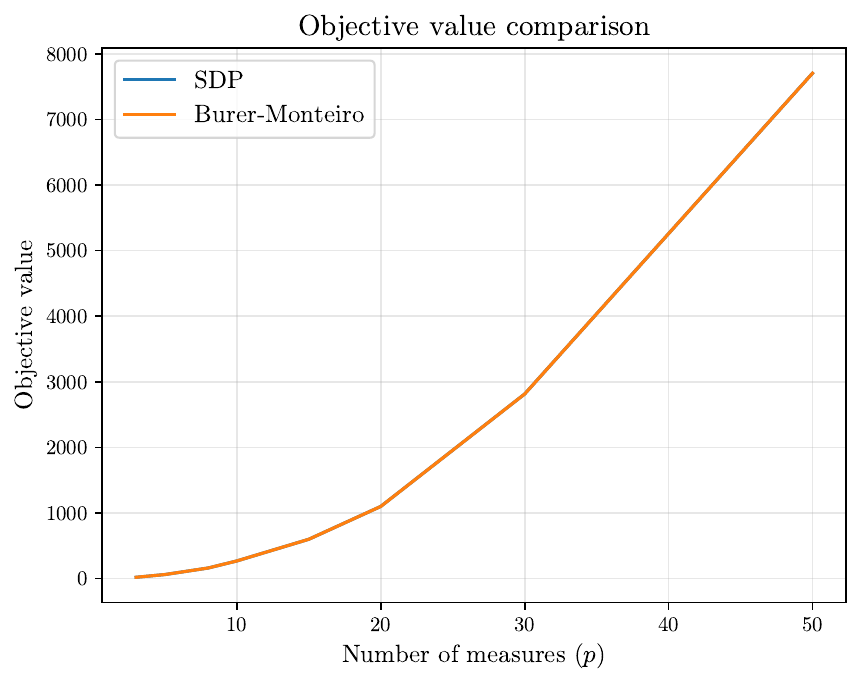}} \label{fig:objective-comparison}}
    \quad
    \subfloat[\centering Optimization variables]{\redboxed{\includegraphics[scale=0.3]{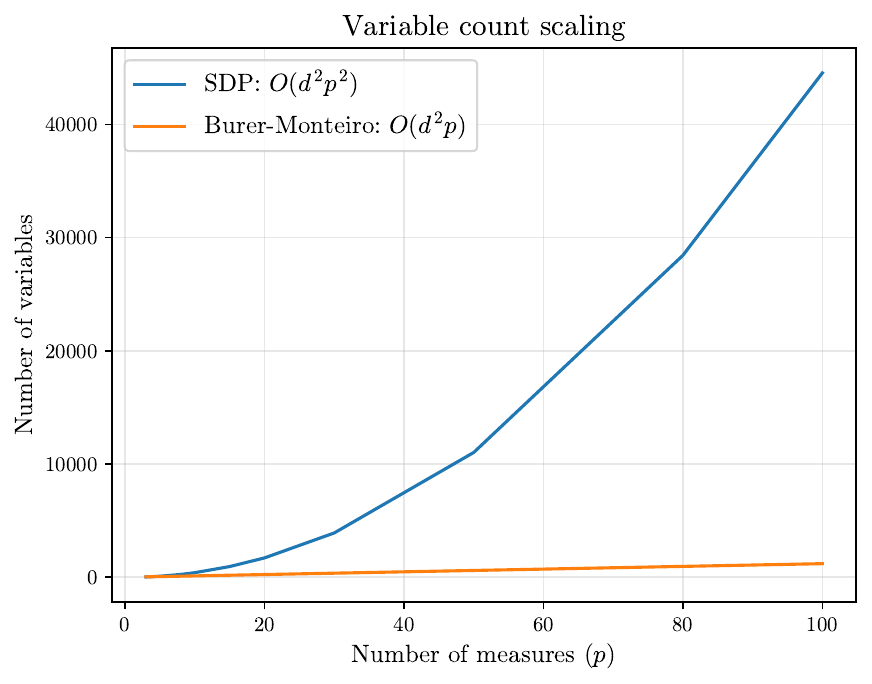}} \label{fig:variable-counts}}
    \vspace*{-.5em}
    \caption{Comparison of the Burer-Monteiro and interior-point approaches for solving multimarginal OT between $p$ Gaussian measures on $\R^3$: (a) computation times, (b) covariance-pairing objectives $J=\sum_{i<j}\tr(C_{ij})$, with relative error $|J_{\mathrm{BM}}-J_{\mathrm{SDP}}|/|J_{\mathrm{SDP}}|$, and (c) number of optimization variables.}
    \label{fig:multimarginal-experiment}
\end{figure}

We also compare with a specialized fixed-point method for Gaussian Wasserstein barycenters in \Cref{app:optimization-details}. Although this method is faster in the tested settings, our Burer-Monteiro approach provides a direct formulation of the full multimarginal OT problem in terms of coupling factors, with global optimality guaranteed at every second-order stationary point by \Cref{thm:burer-monteiro-multimarginal}.

\section{Conclusion}\label{sec:conclusion}

{
    This work gives an exact variational characterization of IGW alignment between uncentered Gaussian measures on real separable Hilbert spaces, together with moment-based bounds for Gaussian and general measures. The centered spectral characterization yields explicit barycenter and multimarginal IGW formulae. For quadratic multimarginal OT with positive product weights and nondegenerate Gaussian marginals, we show that every second-order stationary point of the Burer--Monteiro problem with $k=d+1$ has rank $d$ and is globally optimal. Our experiments demonstrate applications of IGW to distributional comparisons of already trained language-model distillations, covariance-spectrum clustering, and numerical computation of multimarginal Gaussian OT. Further work could study tighter bounds for uncentered IGW alignment and the relationship between unpaired Gaussian comparisons and paired representational comparisons for knowledge distillation.
}

\par
\begingroup
\linespread{1}\selectfont
\bibliographystyle{plain}
\bibliography{bibliography}
\endgroup
\clearpage
\appendix
\section{Proofs of Main Results} \label{sec:proofs}

In this section, we collect proofs of our main results.

\subsection{Proof of \texorpdfstring{\Cref{thm:igw-gaussian-distance}}{Theorem \getrefnumber{thm:igw-gaussian-distance}}}

We begin with a technical lemma that will be useful in the proof of \Cref{thm:igw-gaussian-distance}.

\begin{lemma}[Cyclic property of trace] \label{lem:gaussian-inner-product}
    Suppose that $\mu_i \in \mc{P}_2(\mc{H})$ are centered for $i \in [2]$ and that $\pi \in \Pi(\mu_1, \mu_2)$. Then, if $((X_1, X_2), (X_1^\prime, X_2^\prime)) \sim \pi \otimes \pi$, we have
    \begin{align*}
        \E[\inner{X_1, X_1^\prime} \inner{X_2, X_2^\prime}] = \tr((\Sigma_\pi^\times)^* \Sigma_\pi^\times).
    \end{align*}
\end{lemma}
\begin{proof}
    The proof of this lemma is relatively standard and follows that of the usual cyclic property of trace in a Hilbert space, so we omit it here (e.g., see \cite{simon2005trace}).
\end{proof}

\begin{proof}[Proof of \Cref{thm:igw-gaussian-distance}]
    
    Let $(X,Y)\sim\pi\in\Pi(\mu_1,\mu_2)$ and let $(X',Y')$ be an independent copy. Define $M_X=\E[X\otimes X]$, $M_Y=\E[Y\otimes Y]$, and $S_{XY}=\E[X\otimes Y]$. Expanding in orthonormal bases and using independence gives the second-moment identity
    \begin{align*}
        \E\big[(\inner{X,X'}-\inner{Y,Y'})^2\big]
        =\norm{M_X}_{\HS}^2+\norm{M_Y}_{\HS}^2-2\norm{S_{XY}}_{\HS}^2.
    \end{align*}
    The interchanges of expectation and sums are justified by Cauchy--Schwarz and the finite second moments. If $K=\Sigma_\pi^\times$, then $M_X=\Sigma_1+m_1\otimes m_1$, $M_Y=\Sigma_2+m_2\otimes m_2$, and $S_{XY}=K+m_1\otimes m_2$. Consequently, the cost equals
    \begin{align*}
         & \tr(\Sigma_1^2)+\tr(\Sigma_2^2)+2\inner{m_1,\Sigma_1m_1}+2\inner{m_2,\Sigma_2m_2}  \\
         & \quad +(\norm{m_1}^2-\norm{m_2}^2)^2-2\big(\norm K_{\HS}^2+2\inner{m_1,Km_2}\big).
    \end{align*}
    The covariance of any coupling is a positive semidefinite, self-adjoint block operator with diagonal blocks $\Sigma_1,\Sigma_2$. Conversely, every such block operator is trace-class, with trace $\tr(\Sigma_1)+\tr(\Sigma_2)$, and is the covariance of a Gaussian coupling with the prescribed marginal means and covariances \cite{bogachev1998gaussian}. Thus restricting to Gaussian couplings leaves the infimum unchanged. By the cross-covariance factorization of Baker~\cite{baker1973joint}, equivalently Theorem 1.6.1 of Bikchentaev et al.~\cite{bikchentaev2024trace}, block positivity is equivalent to
    \begin{align*}
        K=\Sigma_1^{1/2}C^*\Sigma_2^{1/2},\qquad \norm C_{\mr{op}}\leq1.
    \end{align*}
    Substitution gives \eqref{eq:gamma} and the stated distance formula, without any assumption on covariance kernels or ranges.

    It remains to prove attainment. On the unit ball of bounded operators, the objective is weak-* continuous. Indeed, for any total orthonormal set $(e_j)$,
    \begin{align*}
        \norm{\Sigma_2^{1/2}C\Sigma_1^{1/2}}_{\HS}^2
        =\sum_{i,j}\inner{\Sigma_2^{1/2}e_i,C\Sigma_1^{1/2}e_j}^2.
    \end{align*}
    Each summand is weak-* continuous, since $\inner{Cu,v}=\tr((u\otimes v)C)$. On the unit ball it is bounded by $\norm{\Sigma_2^{1/2}e_i}^2\norm{\Sigma_1^{1/2}e_j}^2$, whose double sum is $\tr(\Sigma_1)\tr(\Sigma_2)<\infty$. Dominated convergence gives sequential continuity of the first term. The second term is weak-* continuous by the same rank-one trace identity. Since $\mc H$ is separable, its trace-class predual is separable, so the weak-* topology on the unit ball is metrizable and sequential continuity suffices.

    By Banach--Alaoglu the unit ball is weak-* compact. The objective is convex, so Bauer's maximum principle \cite{bauer1958minimalstellen} gives a maximizer at an extreme point. These extreme points are precisely the isometries and co-isometries \cite{kadison1951isometries}. The stated Gaussian coupling therefore attains the infimum.
\end{proof}

\subsection{Proof of \texorpdfstring{\Cref{thm:igw-gaussian-bound}}{Theorem \getrefnumber{thm:igw-gaussian-bound}}}

{
    For every contraction $C$, the von Neumann trace inequality \cite[Remark 1]{grigorieff1991note} and the min-max principle give
    \begin{align*}
        \norm{\Sigma_2^{1/2}C\Sigma_1^{1/2}}_{\HS}^2
         & =\tr(\Sigma_1C^*\Sigma_2C)                           \\
         & \leq\sum_k\lambda_k(\Sigma_1)\lambda_k(C^*\Sigma_2C)
        \leq\sum_k\lambda_k(\Sigma_1)\lambda_k(\Sigma_2).
    \end{align*}
    Also, by Cauchy--Schwarz,
    \begin{align*}
        \inner{C\Sigma_1^{1/2}m_1,\Sigma_2^{1/2}m_2}
        \leq\norm{\Sigma_1^{1/2}m_1}\norm{\Sigma_2^{1/2}m_2}
        =\norm{\Lambda_1^{1/2}\tilde m_1}\norm{\Lambda_2^{1/2}\tilde m_2}.
    \end{align*}
    These two inequalities give the lower bound. For the upper bound, $C=Q_2DQ_1^*$ with the diagonal signs specified after \Cref{thm:igw-gaussian-bound} is a contraction and satisfies
    \begin{align*}
        \norm{\Sigma_2^{1/2}C\Sigma_1^{1/2}}_{\HS}^2
         & =\sum_k\lambda_k(\Sigma_1)\lambda_k(\Sigma_2),                                  \\
        \inner{C\Sigma_1^{1/2}m_1,\Sigma_2^{1/2}m_2}
         & =\sum_k\left|(\Lambda_1^{1/2}\tilde m_1)_k(\Lambda_2^{1/2}\tilde m_2)_k\right|.
    \end{align*}
    Substituting this feasible choice in \eqref{eq:gamma} proves the upper bound. The use of partial isometries ensures that the construction also applies to infinite-rank covariances with nontrivial kernels.
}

\subsection{Proof of \texorpdfstring{\Cref{cor:igw-general-bound}}{Corollary \getrefnumber{cor:igw-general-bound}}}

{
    The second-moment identity in the proof of \Cref{thm:igw-gaussian-distance} holds for every coupling of measures in $\mc P_2(\mc H)$. Replacing each such coupling by the Gaussian measure with the same joint mean and covariance preserves its IGW cost and gives a coupling of the moment-matched Gaussian marginals. Taking infima proves the claimed lower bound.
}

\subsection{Proof of \texorpdfstring{\Cref{cor:igw-centered-gaussianization}}{Corollary \getrefnumber{cor:igw-centered-gaussianization}}}

{
    Let $(X,Y) \sim \pi \in \Pi(\mu,\nu)$, let $(X',Y')$ be an independent copy, and write $S_\pi \coloneq \E_\pi[X \otimes Y]$. Because $\nu$ is centered, its uncentered second-moment operator is $\Sigma_\nu$, whereas that of $\mu$ is $M_\mu$. Expanding the IGW objective gives
    \begin{align*}
        \E_{\pi \otimes \pi}\!\left[\big(\inner{X,X'}-\inner{Y,Y'}\big)^2\right]
        = \tr(M_\mu^2)+\tr(\Sigma_\nu^2)-2\,\norm{S_\pi}_{\HS}^2.
    \end{align*}
    Moreover, the uncentered second-moment operator of $(X,Y)$ is positive semidefinite, so
    \begin{align*}
        \begin{bmatrix}
            M_\mu   & S_\pi      \\
            S_\pi^* & \Sigma_\nu
        \end{bmatrix} \succeq 0.
    \end{align*}
    Relaxing the coupling constraint to this necessary block-positivity condition can only decrease the infimum. Conversely, every such block operator is the covariance of a centered Gaussian coupling between $\Normal(0,M_\mu)$ and $\nu$. Therefore the relaxed value is exactly $\IGW(\Normal(0,M_\mu),\nu)^2$, and
    \begin{align*}
        \IGW(\mu,\nu)^2 \geq \IGW\big(\Normal(0,M_\mu),\nu\big)^2.
    \end{align*}
    Taking nonnegative square roots proves the result.
}

\subsection{Proof of \texorpdfstring{\Cref{cor:igw-analytic-gaussian}}{Corollary \getrefnumber{cor:igw-analytic-gaussian}}}

\begin{proof}[Proof of (i) in \Cref{cor:igw-analytic-gaussian}]
    The upper and lower bounds of \Cref{thm:igw-gaussian-bound} match under the assumption that $\inner{\Lambda_1^{1/2}\tilde m_1,\Lambda_2^{1/2}\tilde m_2}=\norm{\Lambda_1^{1/2}\tilde m_1}\norm{\Lambda_2^{1/2}\tilde m_2}$; this corresponds to the equality case of the Cauchy-Schwarz inequality.
\end{proof}

\begin{proof}[Proof of (ii) in \Cref{cor:igw-analytic-gaussian}]
    This follows immediately from \Cref{thm:igw-gaussian-distance} since $\overline{\mc{U}}(\R) = \set{\pm 1}$.
\end{proof}

\begin{proof}[Proof of \Cref{prop:igw-gaussian-barycenter}]
    {For each $k$, the map $\nu\mapsto\lambda_k(\Sigma_\nu)$ is Borel measurable: $\mr W_2$ convergence implies trace-norm convergence of covariance operators, and $|\lambda_k(A)-\lambda_k(B)|\leq\norm{A-B}_{\mr{op}}\leq\norm{A-B}_{\mf S_1}$. First, we verify that $\Normal(0, \Sigma)$ is well-defined and that the objective is finite at $\Normal(0, \Sigma)$. Since $\rho$ is a probability measure, Jensen's inequality gives $\int \tr(\Sigma_\nu)\, d\rho(\nu) \leq \big( \int \tr(\Sigma_\nu)^2\, d\rho(\nu) \big)^{1/2} < \infty$. By Tonelli's theorem,
\begin{align*}\tr(\Sigma) = \sum_{k=1}^\infty \int \lambda_k(\Sigma_\nu)\, d\rho(\nu) = \int \tr(\Sigma_\nu)\, d\rho(\nu) < \infty,\end{align*}
so $\Sigma$ is a trace-class, self-adjoint, positive semidefinite operator and hence a valid Gaussian covariance operator. Moreover, using $\sum_{k=1}^\infty \lambda_k(\Sigma_\nu)^2 \leq \tr(\Sigma_\nu)^2$,
        \begin{align*}
            \int \sum_{k=1}^\infty \big( \lambda_k(\Sigma) - \lambda_k(\Sigma_\nu) \big)^2\, d\rho(\nu)
            \leq 2 \sum_{k=1}^\infty \lambda_k(\Sigma)^2 + 2 \int \tr(\Sigma_\nu)^2\, d\rho(\nu)
            < \infty,
        \end{align*}
        so the objective is finite at $\Normal(0, \Sigma)$.}
    For every candidate $\mu \in \mc{P}_2(\mc{H})$, \Cref{cor:igw-centered-gaussianization} gives $\IGW(\mu,\nu) \geq \IGW(\Normal(0,M_\mu),\nu)$ for every $\nu \in \supp(\rho)$. After integrating with respect to $\rho$, it therefore suffices to minimize over centered-Gaussian candidates. For such a candidate $\mu \in \mc{G}_0(\mc{H})$, using (i) in \Cref{cor:igw-analytic-gaussian} along with Tonelli's theorem, we have
    \begin{align*}
        \int \IGW(\mu, \nu)^2\, d\rho(\nu)
         & = \int \sum_{k=1}^\infty \left( \lambda_k(\Sigma_\mu) - \lambda_k(\Sigma_\nu) \right)^2\, d\rho(\nu)  \\
         & = \sum_{k=1}^\infty \int \left( \lambda_k(\Sigma_\mu) - \lambda_k(\Sigma_\nu) \right)^2\, d\rho(\nu).
    \end{align*}
    Hence, the optimization problem decomposes across the eigenvalues of $\Sigma_\mu$, and is strictly convex in each $\lambda_k(\Sigma_\mu)$ for $k \in \N$. Let $\bar{\lambda}_k \coloneq \int \lambda_k(\Sigma_\nu)\, d\rho(\nu)$ for $k \in \N$. Because every input spectrum is nonnegative and nonincreasing, the averaged sequence $(\bar{\lambda}_k)_{k \in \N}$ is also nonnegative and nonincreasing; the trace calculation above shows that it is summable. It is therefore the spectrum of a feasible covariance operator. We can rewrite
    \begin{align*}
         & \int \left( \lambda_k(\Sigma_\mu) - \lambda_k(\Sigma_\nu) \right)^2\, d\rho(\nu)                                                                                                                                               \\
         & \quad = \int \left( (\lambda_k(\Sigma_\mu) - \bar{\lambda}_k) - (\lambda_k(\Sigma_\nu) - \bar{\lambda}_k) \right)^2\, d\rho(\nu)                                                                                               \\
         & \quad = \int (\lambda_k(\Sigma_\nu) - \bar{\lambda}_k)^2\, d\rho(\nu) - 2 (\lambda_k(\Sigma_\mu) - \bar{\lambda}_k) \int (\lambda_k(\Sigma_\nu) - \bar{\lambda}_k)\, d\rho(\nu) + (\lambda_k(\Sigma_\mu) - \bar{\lambda}_k)^2.
    \end{align*}
    The integral in the second term vanishes by definition of $\bar{\lambda}_k$ and the first term is constant with respect to $\lambda_k(\Sigma_\mu)$, so the objective function is minimized when $\lambda_k(\Sigma_\mu) = \bar{\lambda}_k$; this choice is nonnegative and therefore feasible, yielding the result.
\end{proof}

\subsection{Proof of \texorpdfstring{\Cref{thm:multimarginal-igw}}{Theorem \getrefnumber{thm:multimarginal-igw}}}

{
    For each $i$, the series defining $X_i$ converges in $L^2(\mc H)$ because $\sum_k\lambda_k(\Sigma_i)<\infty$. Its law is $\Normal(0,\Sigma_i)$, so the joint law is a feasible coupling. For $i<j$, its cross-covariance is
    \begin{align*}
        K_{ij}=\sum_{k:\lambda_k(\Sigma_i)\lambda_k(\Sigma_j)>0}
        \sqrt{\lambda_k(\Sigma_i)\lambda_k(\Sigma_j)}\,v_{ik}\otimes v_{jk}.
    \end{align*}
    Orthonormality gives $\norm{K_{ij}}_{\HS}^2=\sum_k\lambda_k(\Sigma_i)\lambda_k(\Sigma_j)$. By the centered case of \Cref{cor:igw-analytic-gaussian}, the pairwise distortion is therefore exactly $\IGW(\mu_i,\mu_j)^2$. Every multimarginal coupling has cost at least the sum of these pairwise minima, and the constructed coupling attains that sum.
}

\subsection{Proof of \texorpdfstring{\Cref{thm:multimarginal-gaussian}}{Theorem \getrefnumber{thm:multimarginal-gaussian}}}

We begin by reducing to the case of centered measures. Fix a coupling $\pi \in \Pi(\mu_1, \dots, \mu_p)$, let $\tilde{\mu}_i = \Normal(0, \Sigma_i)$ for $i \in [p]$, and define $\tilde{\pi} \in \Pi(\tilde{\mu}_1, \dots, \tilde{\mu}_p)$ by $\tilde{\pi} = ((\cdot - m_1), \dots, (\cdot - m_p))_\sharp\, \pi$. Now, we can write
\begin{align*}
     & \int \sum_{i=1}^p \sum_{j=i+1}^p \norm{x_i - x_j}_2^2\, d\pi(x_1, \ldots, x_p)                                                                            \\
     & \quad = \int \sum_{i=1}^p \sum_{j=i+1}^p \norm{x_i - x_j + m_i - m_j}_2^2\, d\tilde{\pi}(x_1, \ldots, x_p)                                                \\
     & \quad = \int \sum_{i=1}^p \sum_{j=i+1}^p (\norm{x_i - x_j}_2^2 + 2 \inner{x_i - x_j, m_i - m_j} + \norm{m_i - m_j}_2^2)\, d\tilde{\pi}(x_1, \ldots, x_p).
\end{align*}
The second term vanishes because $\tilde{\pi}$ is centered, so we have
\begin{align*}
     & \int \sum_{i=1}^p \sum_{j=i+1}^p (\norm{x_i - x_j}_2^2 + 2 \inner{x_i - x_j, m_i - m_j} + \norm{m_i - m_j}_2^2)\, d\tilde{\pi}(x_1, \ldots, x_p)   \\
     & \quad = \int \sum_{i=1}^p \sum_{j=i+1}^p \norm{x_i - x_j}_2^2\, d\tilde{\pi}(x_1, \ldots, x_p) + \sum_{i=1}^p \sum_{j=i+1}^p \norm{m_i - m_j}_2^2.
\end{align*}
Hence, it suffices to solve the multimarginal OT problem assuming that all of the input measures are centered. Now, let $\pi \in \Pi(\mu_1, \ldots, \mu_p)$ be a coupling and define the projection maps $\Pi_{ij}(x_1, \dots, x_p) \coloneq (x_i, x_j)$ for $1 \leq i < j \leq p$. Defining $\pi_{ij} \coloneq (\Pi_{ij})_\sharp\, \pi$ and letting $C_{ij} \coloneq \Sigma_{\pi_{ij}}^\times$, we can expand the objective as
\begin{align*}
     & \int \sum_{i=1}^p \sum_{j=i+1}^p \norm{x_i - x_j}_2^2\, d\pi(x_1, \ldots, x_p)                                           \\
     & \quad = \int \sum_{i=1}^p \sum_{j=i+1}^p (\norm{x_i}_2^2 - 2 \inner{x_i, x_j} + \norm{x_j}_2^2)\, d\pi(x_1, \ldots, x_p) \\
     & \quad = \tr\left( \sum_{i=1}^p \left( (p - 1)\, \Sigma_i - 2 \sum_{j=i+1}^p C_{ij} \right) \right).
\end{align*}
Since the input measures are Gaussian and the only constraints are on the covariance of the coupling, we may assume without loss of generality that $\pi$ is a Gaussian coupling: the covariance matrix of any coupling is a positive semidefinite matrix with diagonal blocks $\Sigma_1, \dots, \Sigma_p$, and conversely any such matrix is the covariance of a centered Gaussian measure on $\R^{dp}$ whose marginals are exactly $\mu_1, \dots, \mu_p$; since the objective depends on $\pi$ only through its covariance, the optimal value is unchanged. Therefore, we can replace the optimization problem with the equivalent semidefinite program
\begin{align*}
    \inf_{C_{ij} \in \R^{d \times d}} & \quad \tr\left( \sum_{i=1}^p \left( (p - 1)\, \Sigma_i - 2 \sum_{j=i+1}^p C_{ij} \right) \right) \\
    \text{s.t.}                       & \quad \begin{bmatrix}
                                                  \Sigma_1         & C_{12}           & \cdots & C_{1p}   \\
                                                  C_{12}^\intercal & \Sigma_2         & \cdots & C_{2p}   \\
                                                  \vdots           & \vdots           & \ddots & \vdots   \\
                                                  C_{1p}^\intercal & C_{2p}^\intercal & \cdots & \Sigma_p
                                              \end{bmatrix} \succeq 0.
\end{align*}
The feasible set of this semidefinite program is nonempty (the block-diagonal matrix with $C_{ij} = 0$ is feasible), closed, and bounded: positive semidefiniteness of every two-by-two principal block implies $\norm{C_{ij}}_\mr{op} \leq \norm{\Sigma_i}_\mr{op}^{1/2} \norm{\Sigma_j}_\mr{op}^{1/2}$. Since the objective is linear, a solution of the semidefinite program exists.

Now, by Theorem 2.1 and Corollary 2.2 of Gangbo and {\'S}wi{\k{e}}ch~\cite{gangbo1998optimal} (whose hypotheses hold here: the cost agrees, up to an irrelevant positive factor, with the pairwise quadratic cost $\sum_{i < j} \norm{x_i - x_j}_2^2$ treated therein, and since $\Sigma_i \succ 0$, each $\mu_i$ is absolutely continuous with respect to the Lebesgue measure and thus vanishes on $(d-1)$-rectifiable sets), the optimal multimarginal coupling is unique and there exist measurable maps $S_i : \R^d \to \R^d$ for $i \in \set{2, \dots, p}$ such that $\pi = (I, S_2, \dots, S_p)_\sharp\, \mu_1$ is the optimal coupling. Any solution $\set{C_{ij}}$ of the semidefinite program induces a Gaussian coupling of $\mu_1, \ldots, \mu_p$ whose cost equals the optimal multimarginal cost; by uniqueness, this Gaussian coupling \emph{is} the optimal coupling $\pi$, and since a Gaussian measure is determined by its mean and covariance, the solution of the semidefinite program is unique. Moreover, we know since $\pi$ is jointly Gaussian that $\E_\pi[X_i \given X_1 = x] = C_{1i}^\intercal \Sigma_1^{-1} x = S_i(x)$ for Lebesgue-almost every $x \in \R^d$, where the conditional expectation exists and is unique Lebesgue-almost everywhere by the disintegration of measure. Hence, $S_i(x) = A_i x$ are linear Lebesgue-almost everywhere, and
\begin{align*}
    C_{ij}
    = \E_\pi[X_i X_j^\intercal]
    = A_i \E_\pi[X_1 X_1^\intercal] A_j^\intercal
    = A_i \Sigma_1 A_j^\intercal.
\end{align*}
Setting $A_1 \coloneq I$ and $B \coloneq (A_1^\intercal, A_2^\intercal, \dots, A_p^\intercal)^\intercal \in \R^{dp \times d}$, the display above shows that $\Sigma = B \Sigma_1 B^\intercal$, so $\rank(\Sigma) \leq d$; since the first diagonal block of $\Sigma$ is $\Sigma_1 \succ 0$, we also have $\rank(\Sigma) \geq d$. Hence $\rank(\Sigma) = d$, completing the proof.

\subsection{Proof of \texorpdfstring{\Cref{thm:burer-monteiro-multimarginal}}{Theorem \getrefnumber{thm:burer-monteiro-multimarginal}}}

{
    We prove the rank assertion for every $k>d$. Write $S=\sum_i\rho_iU_i$. Since $\norm{U_i}_{\mr F}^2=\tr(\Sigma_i)$ is fixed, the objective differs by a constant from $-\frac12\norm S_{\mr F}^2$. Every $U_i$ has full row rank, so the equality constraints are regular. For symmetric multipliers $A_i$, the Lagrangian is
    \begin{align*}
        L(U,A)=-\frac12\norm S_{\mr F}^2
        +\frac12\sum_i\tr\big(A_i(U_iU_i^\intercal-\Sigma_i)\big).
    \end{align*}
    At a second-order stationary point, $A_iU_i=\rho_iS$, and for every tangent direction $\Delta$,
    \begin{align*}
        \sum_i\tr(A_i\Delta_i\Delta_i^\intercal)
        -\left\|\sum_i\rho_i\Delta_i\right\|_{\mr F}^2\geq0.
    \end{align*}
    Fix $i$ and choose a unit vector $z\in\ker U_i$, which exists since $k>d$. For arbitrary $a\in\R^d$, the direction $\Delta_i=az^\intercal$, $\Delta_j=0$ for $j\ne i$, is tangent. Substitution gives $a^\intercal(A_i-\rho_i^2I)a\geq0$, hence $A_i\succeq\rho_i^2I\succ0$. Stationarity now implies $U_i=\rho_iA_i^{-1}S$. Thus all blocks have the same row space as $S$, and $\rank S=\rank U=d<k$.

    The rank-deficient second-order certificate of Boumal et al.~\cite[Corollary 8]{boumal2016non} yields global optimality. It can also be seen directly: choose a common unit vector $z\in\ker U$ and substitute $\Delta_i=a_i z^\intercal$ for arbitrary $a_i\in\R^d$. The second-order condition gives
    \begin{align*}
        D\coloneq\diag(A_1,\ldots,A_p)-(\rho\rho^\intercal)\otimes I_d\succeq0,
        \qquad DU=0,
    \end{align*}
    where $\rho=(\rho_1,\ldots,\rho_p)^\intercal$ in this display. These are dual feasibility and complementary slackness for the SDP objective $-\frac12\norm S_{\mr F}^2$, whose fixed diagonal terms have been absorbed. The barycenter formula follows from \Cref{prop:multimarginal-wasserstein-barycenter}. This argument uses positive product weights and positive definite marginal covariances.
}

\section{Experimental Details}\label{app:experimental-details}

In this section, we provide implementation details, parameter settings, and additional results for our experiments. All code for our experiments can be found in the following GitHub repository:
\begin{center}
\href{https://github.com/sanjitdp/gaussian-gw}{\texttt{https://github.com/sanjitdp/gaussian-gw}}
\end{center}

\subsection{Language-Model Representation Comparisons}\label{app:representation-details}

Consider a set $\mc{S}\subseteq \mc{V}^k$ of sentences (over a vocabulary $\mc{V}$ and context length $k$) and let $\varphi : \mc{V}^k \to \R^d$ represent the forward pass of a black-box LLM from the input to an internal representation space. Comparing the representations $\varphi(\mc{S})$ of different LLMs $\varphi$ has seen notable recent interest \cite{morcos2018insights, kornblith2019similarity, csiszarik2021similarity}. Existing approaches first embed the dataset into $\R^d$ using each model and then compare the resulting point clouds using a similarity metric. The typical setting involves comparing a large model $\varphi$ to many smaller distillation models $\varphi_1,\ldots, \varphi_p$ on a dataset $\mc{D} = \set{d_1, \dots, d_n} \subseteq \mc{S}$, to find the small model whose representations best align with those of $\varphi$. An unpaired version of this problem compares the distributions $\frac{1}{n} \sum_{i=1}^n \delta_{\varphi(d_i)}$ and $\frac{1}{n} \sum_{i=1}^n \delta_{\varphi_j(d_i)}$, for $j \in [p]$. 

\subsubsection{Existing and Proposed Similarity Metrics}
A common choice of similarity metric is the centered kernel alignment (CKA) introduced by Kornblith et al.~\cite{kornblith2019similarity}. For centered data matrices $X\in\R^{n\times d_1}$ and $Y\in\R^{n\times d_2}$ with nonzero Gram matrices, linear CKA is
{
\begin{align*}
    \operatorname{CKA}(X,Y)=\frac{\norm{Y^\intercal X}_{\mr F}^2}{\norm{X^\intercal X}_{\mr F}\norm{Y^\intercal Y}_{\mr F}}
    =\frac{\inner{XX^\intercal,YY^\intercal}_{\mr F}}{\norm{XX^\intercal}_{\mr F}\norm{YY^\intercal}_{\mr F}}.
\end{align*}
Equality to one implies $YY^\intercal=cXX^\intercal$ for some $c>0$. Thus, the centered configurations coincide after uniform rescaling and an orthogonal alignment of their spans. In particular, corresponding pairs of nonzero centered vectors have the same cosine similarity. Linear CKA is invariant to orthogonal transformations and isotropic scaling, but its centering and normalization remove location and global scale information, which may be undesirable when the embedding origin and scale are informative.

Uncentered IGW provides a complementary distributional comparison that retains this information and accommodates different embedding dimensions. CKA uses the correspondence between inputs, whereas IGW optimizes over all such correspondences. Gaussian IGW compares the means and covariances of embedding distributions, making it useful when the goal is to compare their overall structure, particularly when inputs are unpaired or embedding dimensions differ.
}

\subsubsection{Gaussian Modeling Assumption and Verification}

To employ our framework, we model text embedding distributions as Gaussian, which is a common practice in machine learning; see Lv et al.~\cite{lv2024wasserstein} and Mroueh~\cite{mroueh2019wasserstein} for instances closely related to our work, among many others.

To assess this modeling assumption, we first examine whether low-dimensional projections of the empirical distributions of text embeddings from LLMs on common datasets can be approximated by Gaussians. For reference, we consider the \ci{bert-base-uncased} teacher model and 23 of its distilled variants, described in Turc et al.~\cite{turc2019well}. We compare already trained distillations, without training or distilling models in these experiments. They have a variety of numbers of transformer layers $L \in \set{2, 4, 6, 8, 10, 12}$, and hidden embedding dimensions $H \in \set{128, 256, 512, 768}$. We display the number of parameters, in millions, of each model in \Cref{tab:model-params}.

\begin{table}[!t]
    \centering
    \begin{tabular}{|c|c|c|c|c|}
        \hline
                 & $H = 128$ & $H = 256$ & $H = 512$ & $H = 768$ \\
        \hline
        $L = 2$  & 4.4       & 9.7       & 22.8      & 39.2      \\
        \hline
        $L = 4$  & 4.8       & 11.3      & 29.1      & 53.4      \\
        \hline
        $L = 6$  & 5.2       & 12.8      & 35.4      & 67.5      \\
        \hline
        $L = 8$  & 5.6       & 14.4      & 41.7      & 81.7      \\
        \hline
        $L = 10$ & 6.0       & 16.0      & 48.0      & 95.9      \\
        \hline
        $L = 12$ & 6.4       & 17.6      & 54.3      & 110.1     \\
        \hline
    \end{tabular}
    \par\vspace{8pt}
    \caption{Number of parameters of the \ci{bert} model distillations of Turc et al.~\cite{turc2019well} (in millions), as a function of the number of transformer layers $L$ and hidden embedding dimension $H$. The teacher model \ci{bert-base-uncased} has $L = 12$ and $H = 768$, with 110.1 million parameters.}
    \label{tab:model-params}
\end{table}

We consider datasets of 1000 samples each from: (i) the \ci{amazon_polarity} dataset \cite{muennighoff2022mteb}, which contains Amazon reviews, and (ii) the \ci{ag_news} dataset \cite{zhang2015character}, which contains news articles from over 2000 different sources. First, we inspect the leading two-dimensional projections of the learned representations. In \Cref{fig:amazon_polarity_embeddings-ag_news_embeddings}, we plot the first two principal components of the embeddings of the \ci{bert-base-uncased} model on both datasets, along with a Gaussian fit using the empirical mean and covariance of the embeddings. Note that computing the 2-Wasserstein distance between the full high-dimensional distributions is statistically difficult due to the limited number of samples. In these two-dimensional PCA coordinates, the diagnostic is $\mr{W}_2(\hat\mu,\Normal(\hat m,\hat\Sigma))$, where $\hat\mu$ is the empirical distribution of the 1000 projected embeddings and $(\hat m,\hat\Sigma)$ are their sample mean and covariance. The reported distances are 0.458 for \ci{amazon_polarity} and 0.655 for \ci{ag_news}, in the original PCA units.

\imagesbs{0.45}{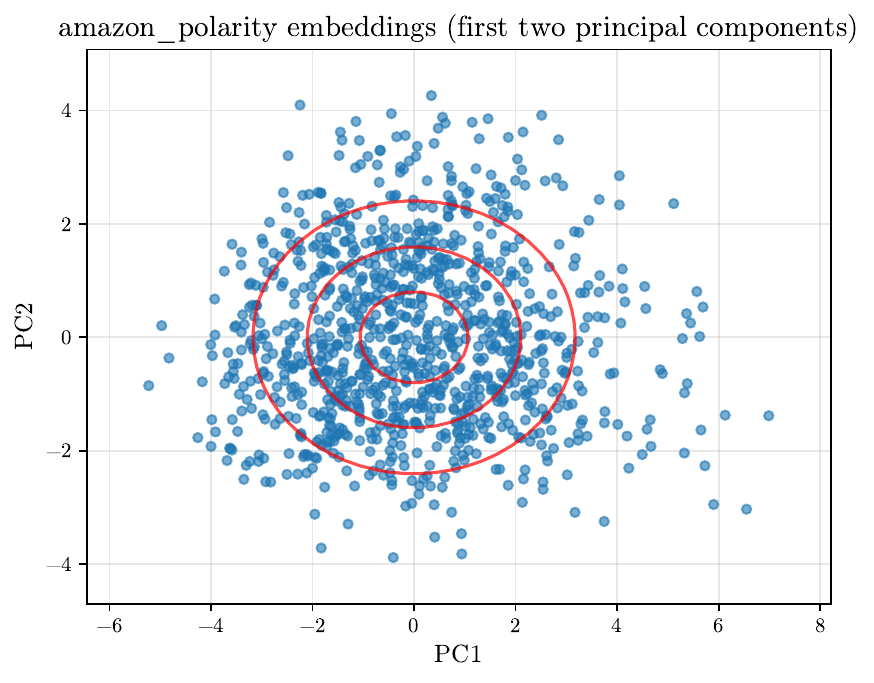}{\ci{amazon_polarity} embeddings}{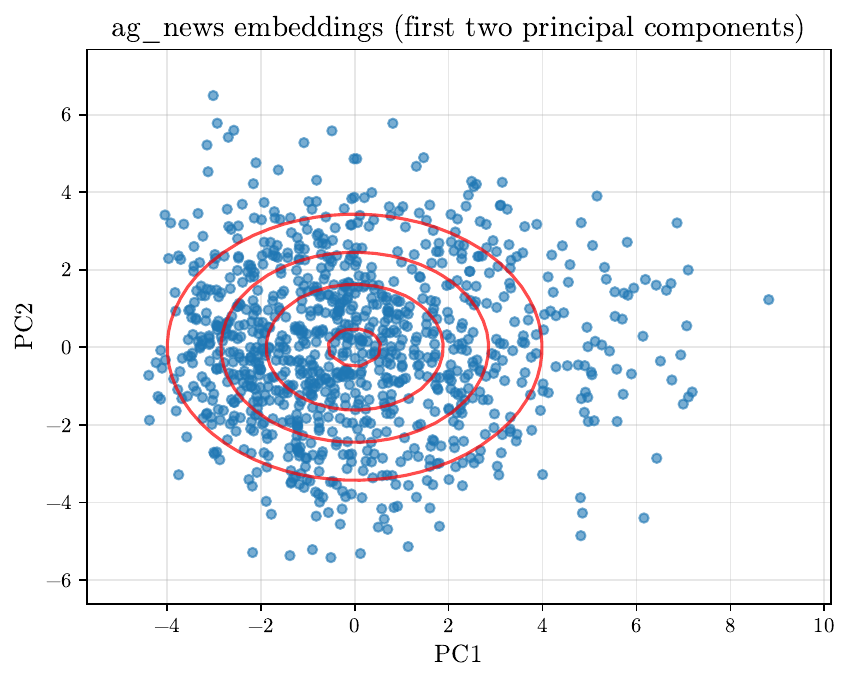}{\ci{ag_news} embeddings}{First two principal components of the embeddings produced by the \ci{bert-base-uncased} model on the (a) \ci{amazon_polarity} and (b) \ci{ag_news} datasets (blue), along with contour lines from a Gaussian fit (red). The leading two-dimensional projections are visually consistent with an approximate Gaussian model.}

We now turn to compute the IGW distances between the  \ci{bert-base-uncased} teacher model and its 23 distillations from Turc et al.~\cite{turc2019well}. To employ our results for Gaussian IGW, we fit a Gaussian to the embeddings of each model on both datasets. We first zero-pad the smaller embedding vectors to the teacher dimension $d=768$. For the Gaussian parameter estimation, we add a regularization of $10^{-6}\, I$ to the full $768\times768$ covariance for numerical stability. Then, we compute the analytic upper and lower bounds on the IGW distance from \Cref{thm:igw-gaussian-bound}. Furthermore, we tighten the upper bound by numerically optimizing over the Stiefel manifold $\mr{St}(d,d)$ using RGD initialized with the diagonal signs in \Cref{thm:igw-gaussian-bound} \cite{absil2008optimization} to solve the optimization problem in \Cref{thm:igw-gaussian-distance}. These bounds quantify the IGW distance between the fitted, regularized Gaussian models. Without covariance regularization, \Cref{cor:igw-general-bound} instead shows that the exact distance between the moment-matched Gaussians is a lower bound on the empirical-distribution IGW distance.

{
We run RGD with an adaptive step size for at most 50 iterations or until the Riemannian gradient norm is less than $10^{-2}$. The updated computation, including both analytic bounds, takes at most 6.2 seconds per comparison on the local CPU. \Cref{fig:igw_distance_amazon_polarity-igw_distance_ag_news} reports distances, taking square roots of the squared-cost bounds. In every comparison,
\begin{align*}
    \text{analytic lower bound}\ \leq\ \IGW\ \leq\ \text{feasible RGD upper bound}.
\end{align*}
}

\subsection{User Clustering}\label{app:clustering-details}

We therefore examine clustering based on covariance spectra, using the centered Gaussian IGW formula. We consider a synthetic user clustering task based on their text posts and simulate differences in interests across users by drawing posts from three different categories. We construct a population of 120 users arranged into three interest groups: science, news, and entertainment. Each group contains 40 users, and each user is represented by 400 non-overlapping posts drawn uniformly from \ci{scientific_papers/arxiv} \cite{cohan2018discourse}, \ci{ag_news} \cite{zhang2015character}, and \ci{imdb} \cite{maas2011learning}, respectively.

We embed each user's posts using the \ci{all-MiniLM-L6-v2} sentence embedding model from Hugging Face \cite{wang2020minilm}. Discarding the mean information, we model each user as $\Normal(0,\Sigma_i)$, where $\Sigma_i$ is the empirical covariance of their embeddings. By the centered formula,
{
\begin{align*}
    \IGW(\Normal(0,\Sigma_i),\Normal(0,\Sigma_j))
    =\norm{\lambda(\Sigma_i)-\lambda(\Sigma_j)}_2.
\end{align*}
Thus the implementation is ordinary Euclidean $k$-means++ on ordered covariance eigenvalue vectors, with zero padding across dimensions. Cluster centers are arithmetic means of these vectors, exactly as in \Cref{prop:igw-gaussian-barycenter}. A full pairwise distance matrix is needed only for the MDS visualization. The results in \Cref{fig:igw_mds_embedding-kmeans_clusters} show that covariance spectra contain useful clustering information.
}

\imagesbs{0.45}{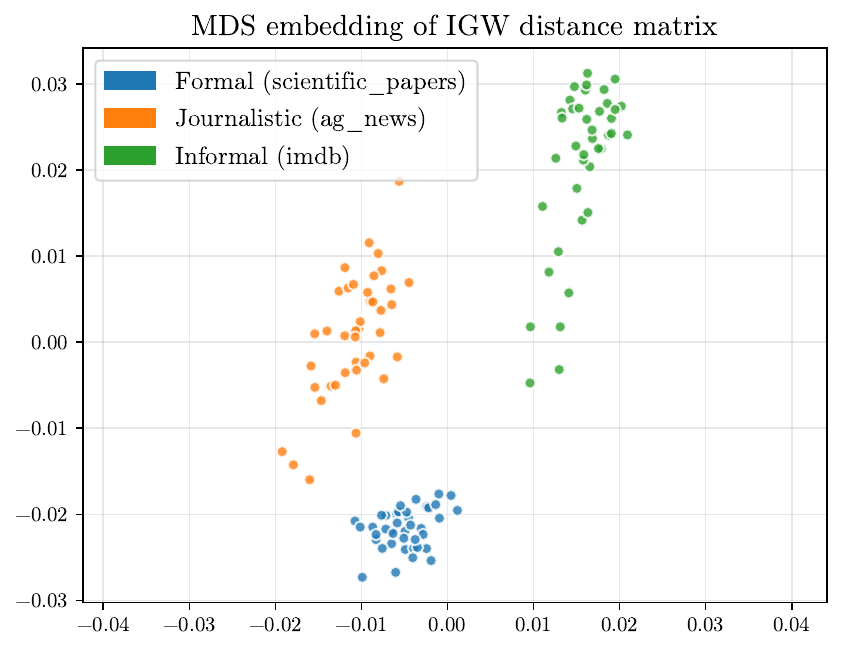}{Ground truth}{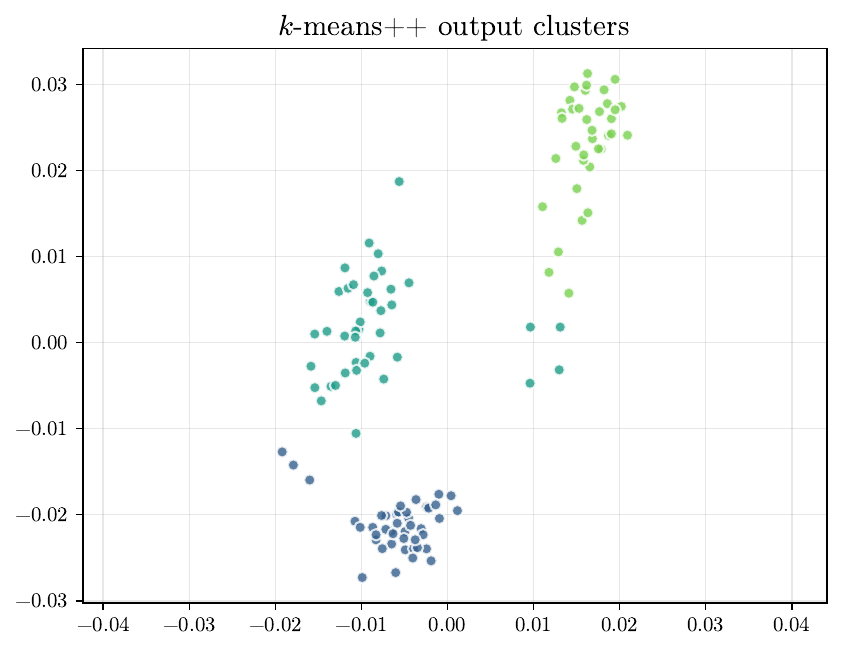}{$k$-means clusters}{(a) 2-D multidimensional scaling (MDS) embedding of the IGW distance matrix between users, colored by their ground truth interests. (b) Euclidean $k$-means++ clustering of the ordered covariance spectra, whose arithmetic means are the spectral IGW barycenters. The IGW distance is able to almost recover the original clusters, using only covariance information.}

{
To compare the contributions of the means and covariance spectra, we repeat the homogeneous experiment using the same \ci{all-MiniLM-L6-v2} embeddings for all three representations. We form a pool of the first 20,000 training documents from each of the three datasets. In each of 10 repeats, we sample 16,000 documents per category without replacement and randomly partition them into 40 users with 400 disjoint documents each. The repeats share the document pools but use independently sampled user populations. We compare Euclidean $k$-means++ on the mean vectors, the ordered covariance spectra, and their concatenation. For the combined representation, each block is centered across users and divided by its root-mean-square distance from the block mean, so the two blocks have equal total weight. Each fit uses 20 initializations, retaining the solution with the smallest sum of squared distances from users to their assigned cluster centers.

\begin{table}[htbp]
    \centering
    \begin{tabular}{|c|c|c|}
        \hline
        Representation & ARI & NMI\\
        \hline
        Mean only & $1.000\pm0.000$ & $1.000\pm0.000$\\
        \hline
        Covariance spectrum only & $1.000\pm0.000$ & $1.000\pm0.000$\\
        \hline
        Combined & $1.000\pm0.000$ & $1.000\pm0.000$\\
        \hline
    \end{tabular}
    \par\vspace{8pt}
    \caption{Adjusted Rand index (ARI) and normalized mutual information (NMI), reported as mean and standard deviation over 10 repeated user samples, using the same embedding model, users, and clustering procedure for all representations.}
    \label{tab:clustering-comparisons}
\end{table}

All three representations recover the category labels perfectly in every repeat (\Cref{tab:clustering-comparisons}). These results show that covariance spectra capture useful information about user interests.
}

\subsection{Multimarginal Optimal Transport}\label{app:optimization-details}

In this section, we show that the Burer-Monteiro approach of \Cref{thm:burer-monteiro-multimarginal} can be used to solve the multimarginal OT problem between Gaussian measures efficiently. We consider $p$ Gaussian measures $\mu_i = \Normal(0, \Sigma_i)$ on $\R^3$, where the $\Sigma_i$ are generated by sampling $A_i \in \R^{3 \times 3}$ with i.i.d.\  standard normal entries and setting $\Sigma_i = A_i A_i^\intercal + 0.1\, I$. We solve the semidefinite program in \Cref{thm:multimarginal-gaussian} using the Burer-Monteiro approach of \Cref{thm:burer-monteiro-multimarginal} with $k = d + 1 = 4$, using Riemannian gradient descent with a sign-normalized QR retraction \cite{absil2008optimization}; their objective values agree with the interior-point solutions, while the expanded experiments below assess numerical optimality using SDP dual bounds. We compare our results to the solution of the semidefinite program modeled in \ci{cvxpy} \cite{diamond2016cvxpy} and solved with an interior-point conic solver. All experiments were run on a MacBook with an M1 Pro processor. The results are reported in \Cref{fig:multimarginal-experiment}. The benchmark in \Cref{fig:multimarginal-experiment} uses explicit pairwise sums and forms the full joint covariance, whereas the comparison below uses vectorized operations on $S=\sum_i\rho_iU_i$ and retains the coupling factors. The former uses a block-gradient stopping threshold of $10^{-8}$ with at most 1000 iterations; the latter uses the relative primal--dual gap specified below, accounting for a further difference in the timing protocols.

{
We specify the constrained parameterization used by the RGD implementation. Write $U_i=\Sigma_i^{1/2}V_i$, where $V_i\in\R^{d\times(d+1)}$ and $V_iV_i^\intercal=I_d$. We use the product of row-Stiefel manifolds with the ambient Frobenius metric. For the equivalent objective $f(V)=-\frac12\norm{S}_{\mr F}^2$, where $S=\sum_i\rho_i\Sigma_i^{1/2}V_i$, the Euclidean and Riemannian gradients are
\begin{align*}
    G_i=-\rho_i\Sigma_i^{1/2}S,\qquad
    \operatorname{grad}_i f=G_i-\operatorname{sym}(G_iV_i^\intercal)V_i,
    \qquad\operatorname{sym}(A)=\tfrac12(A+A^\intercal).
\end{align*}
We use an Armijo line search to select the step size. For $(V_i-\alpha\operatorname{grad}_i f)^\intercal=QR$ with positive diagonal entries in $R$, the update is $V_i^+=Q^\intercal$. Thus $U_i^+(U_i^+)^\intercal=\Sigma_i$, whereas applying ordinary QR directly to $U_i$ would not preserve its covariance constraint.

For a specialized baseline, we use the Gaussian barycenter iteration of {\'{A}}lvarez-Esteban et al.~\cite{alvarez2016fixed},
\begin{align*}
    B^+=B^{-1/2}\left(\sum_i\rho_i(B^{1/2}\Sigma_iB^{1/2})^{1/2}\right)^2B^{-1/2}.
\end{align*}
This method also recovers the optimal multimarginal coupling. At the positive definite barycenter covariance $\bar\Sigma$, define
\begin{align*}
    T_i=\bar\Sigma^{-1/2}(\bar\Sigma^{1/2}\Sigma_i\bar\Sigma^{1/2})^{1/2}\bar\Sigma^{-1/2},
    \qquad U_i=T_i\bar\Sigma^{1/2}.
\end{align*}
The joint law of $(m_1+T_1Z,\ldots,m_p+T_pZ)$, with $Z\sim\Normal(0,\bar\Sigma)$, is an optimal coupling: the barycenter equation gives $\sum_i\rho_iT_i=I$, and each $T_i$ is the Gaussian optimal map to its marginal. Both approaches therefore represent the coupling using the matrices $U_i$, without forming the full $dp\times dp$ covariance.

We compare the two methods at the same objective accuracy, measured by the gap between the computed objective and a dual lower bound on the optimum. For feasible factors, the primal value is $f(U)=-\frac12\norm{\sum_i\rho_iU_i}_{\mr F}^2$. We construct symmetric dual multipliers $\widehat A_i$ with
\begin{align*}
    \widehat D=\diag(\widehat A_1,\ldots,\widehat A_p)-(\rho\rho^\intercal)\otimes I_d\succeq0.
\end{align*}
Specifically, start from $A_i=\operatorname{sym}(\rho_i S U_i^\intercal\Sigma_i^{-1})$, shift each block to be positive definite, and scale all blocks by a common factor at least $\lambda_{\max}(\sum_i\rho_i^2 A_i^{-1})$ after these shifts. A Schur-complement argument gives dual feasibility. The dual lower bound is $-\frac12\sum_i\tr(\widehat A_i\Sigma_i)$, so the relative gap is
\begin{align*}
    \frac{f(U)+\frac12\sum_i\tr(\widehat A_i\Sigma_i)}{\max\{1,|f(U)|\}}.
\end{align*}
We stop when this gap is at most $10^{-7}$ and the maximum relative covariance-constraint error is at most $10^{-10}$, with a limit of $10^4$ iterations. All certificates are evaluated numerically in double precision.

We use uniform weights and vary $p\in\{3,10,25,50,100\}$ at $d=3$, and $d\in\{2,5,8,12,16,32\}$ at $p=10$, with covariances sampled as above. Separately, at $(p,d)=(10,5)$, we use random orthogonal eigenvectors and geometrically spaced eigenvalues with condition numbers $1,10,10^2,10^3,10^4$, normalized to trace $d$. Each of the 16 settings uses three independently generated covariance collections and five random initializations per method, giving 240 runs for each. RGD starts from independent random row-Stiefel frames; the fixed-point method starts from a random positive definite covariance. Timings include initialization, computation of the matrices $U_i$, and the stopping checks, using one BLAS thread; the final eigenvalue check of the full dual matrix $\widehat D$ is excluded for both methods.

Both methods meet the stopping target in all 240 runs. \Cref{fig:final-optimization-scaling} shows that fixed-point iteration is faster at the matched accuracy in every tested setting, while RGD is more sensitive to ill-conditioning. \Cref{tab:optimization-certificates} reports the objective gaps, covariance-constraint errors, and gradient norms. The objective gap bounds the remaining optimization error, while the covariance-constraint error measures how closely the computed coupling matches the prescribed marginal covariances.

\begin{figure}[htbp]
    \centering
    \redboxed{\includegraphics[width=.96\textwidth]{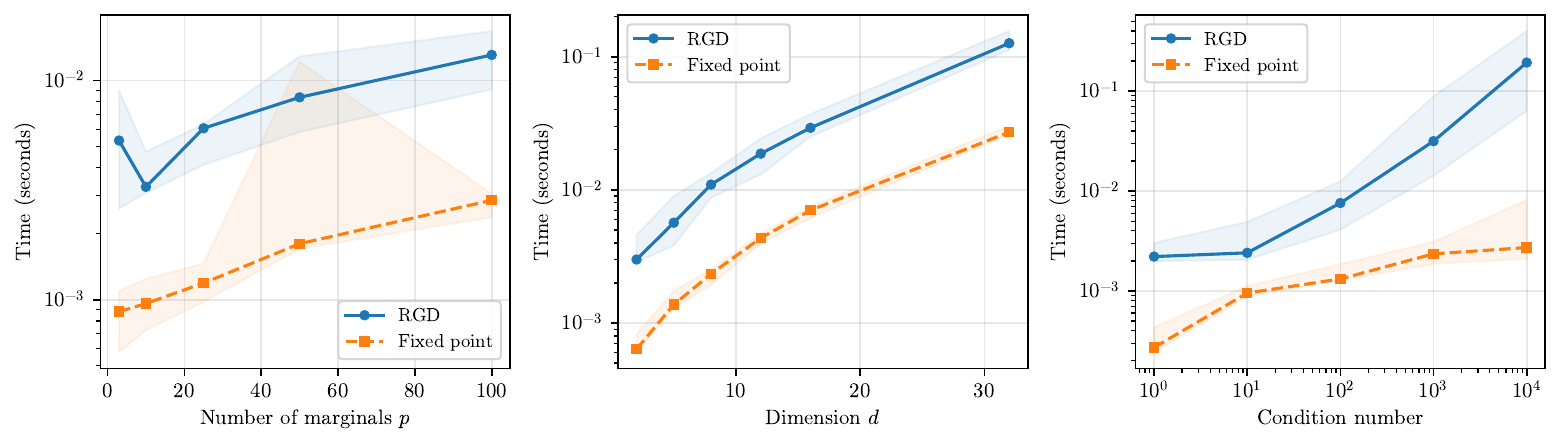}}
    \caption{Computation times at a common relative objective-gap target of $10^{-7}$, including computation of the coupling matrices $U_i$. Curves show medians and shaded bands show ranges over the 15 runs in each setting. The specialized fixed-point method is faster at the matched accuracy in every tested setting.}
    \label{fig:final-optimization-scaling}
\end{figure}

\begin{table}[htbp]
    \centering
    \begin{tabular}{|c|c|c|}
        \hline
        Diagnostic over 240 runs & RGD & Fixed point\\
        \hline
        Maximum Riemannian gradient norm & $9.29\times10^{-4}$ & $1.57\times10^{-3}$\\
        \hline
        Maximum relative objective gap & $1.00\times10^{-7}$ & $9.99\times10^{-8}$\\
        \hline
        Maximum relative covariance-constraint error & $4.08\times10^{-15}$ & $3.77\times10^{-14}$\\
        \hline
    \end{tabular}
    \par\vspace{8pt}
    \caption{Optimization accuracy at the same stopping tolerance. The relative objective gap is defined above; the covariance-constraint error is $\max_i\big(\norm{U_iU_i^\intercal-\Sigma_i}_{\mr F}/\norm{\Sigma_i}_{\mr F}\big)$. All per-run results are provided with the code.}
    \label{tab:optimization-certificates}
\end{table}
}

\section{Recovering the Barycenter from Multimarginal IGW}
\label{app:multimarginal-igw-barycenter}

In this section, we discuss the general connection between the multimarginal IGW problem and the IGW barycenter problem, which allows us to recover a special case of \Cref{prop:igw-gaussian-barycenter}. First, we define the projection maps
\begin{align*}
    \Pi_{i,\, p+1}(x_1, \dots, x_p, x) \coloneq (x_i, x)
\end{align*}
for $i \in [p]$ with $\Pi_{p+1}(x_1, \cdots, x_p, x) \coloneq x$.

\begin{lemma} \label{lem:multimarginal-igw}
    Fix $\rho \in \mc{P}(\mc{P}_2(\mc{H}))$ with $\supp(\rho) = \set{\mu_1, \dots, \mu_p}$. Then, the minimizer $\mu \in \mc{P}_2(\mc{H})$ of
    \begin{align*}
       \inf_{\pi \in \Pi(\mu_1, \dots, \mu_p, \mu)}\; \int \left( \int (\inner{x_\nu, x_\nu^\prime} - \inner{x, x^\prime})^2\, d\rho(\nu) \right)\, d(\pi \otimes \pi)(x_{\mu_1}, \ldots, x_{\mu_p}, x, x_{\mu_1}^\prime, \ldots, x_{\mu_p}^\prime, x^\prime)
    \end{align*}
    is a $\rho$-weighted IGW barycenter. Conversely, if $\bar{\mu}$ is a $\rho$-weighted IGW barycenter, then we can construct $\pi \in \Pi(\mu_1, \ldots, \mu_p, \bar{\mu})$ which solves the inner optimization problem and satisfies $(\Pi_{i,\, p+1})_\sharp\, \pi = \pi_i$, where $\pi_i$ is the pairwise optimal coupling for the IGW problem between $\mu_i$ and $\bar{\mu}$.
\end{lemma}
\begin{proof}

Fix a candidate $\mu\in\mc P_2(\mc H)$ and let $\rho_i=\rho(\{\mu_i\})$. For a pairwise coupling $\pi_i\in\Pi(\mu_i,\mu)$, write $\operatorname{cost}_{\IGW}(\pi_i)$ for its squared IGW distortion. We claim that
\begin{align*}
    \inf_{\pi\in\Pi(\mu_1,\ldots,\mu_p,\mu)}
    \sum_{i=1}^p\rho_i\operatorname{cost}_{\IGW}((\Pi_{i,p+1})_\sharp\pi)
    =\sum_{i=1}^p\rho_i\IGW(\mu_i,\mu)^2.
\end{align*}
The lower bound follows because each pairwise marginal is a coupling of $\mu_i$ and $\mu$. For the upper bound, choose pairwise optimal couplings and glue them over the common marginal $\mu$ using Lemma 7.6 of Villani~\cite{villani2003topics}. The resulting joint coupling attains the right-hand side. Minimizing both sides over $\mu$ proves the barycenter characterization, and the same construction at $\bar\mu$ proves the final assertion.
\end{proof}

Now, we can use \Cref{lem:multimarginal-igw} to recover a special case of \Cref{prop:igw-gaussian-barycenter} with only finitely many centered Gaussian measures in the support of $\rho$.

\begin{corollary}[special case of \Cref{prop:igw-gaussian-barycenter}] \label{cor:multimarginal-igw-barycenter}
    Let $\mc{H}$ be a separable Hilbert space. Suppose that $\rho \in \mc{P}(\mc{P}_2(\mc{H}))$ has finite support $\supp(\rho) = \set{\mu_1, \dots, \mu_p} \subseteq \mc{G}_0(\mc{H})$. Fix any total orthonormal set $\set{e_k}_{k \in \N}$ of $\mc{H}$. Then, defining the covariance operator
    \begin{align*}
        \Sigma \coloneq \sum_{k=1}^\infty \left( \int \lambda_k(\Sigma_\nu)\, d\rho(\nu) \right) e_k \otimes e_k,
    \end{align*}
    the measure $\Normal(0, \Sigma)$ is a $\rho$-weighted IGW barycenter.
\end{corollary}
\begin{proof}

By \Cref{cor:igw-centered-gaussianization}, it suffices to consider centered Gaussian candidates. Applying the centered case of \Cref{cor:igw-analytic-gaussian} independently to each pair gives
\begin{align*}
    \sum_{i=1}^p\rho_i\IGW(\mu_i,\Normal(0,\Sigma))^2
    &=\sum_{i=1}^p\rho_i\left(\tr(\Sigma_i^2)+\tr(\Sigma^2)
    -2\sum_k\lambda_k(\Sigma_i)\lambda_k(\Sigma)\right)\\
    &=\sum_k\sum_{i=1}^p\rho_i(\lambda_k(\Sigma_i)-\lambda_k(\Sigma))^2.
\end{align*}
The minimum is attained at $\lambda_k(\Sigma)=\sum_i\rho_i\lambda_k(\Sigma_i)$. These averages are nonnegative, nonincreasing, and summable, hence form a feasible covariance spectrum. The result follows.
\end{proof}

\section{Optimal Transport Between Gaussian Measures}
\label{app:gaussian-wasserstein}

In this section, we discuss the 2-Wasserstein distance between Gaussian measures on $\R^d$ and the resulting optimal coupling.

\begin{theorem} \label{thm:wasserstein-gaussian}
    If $\mu_1 = \Normal(m_1, \Sigma_1)$ and $\mu_2 = \Normal(m_2, \Sigma_2)$ are two Gaussian measures on $\R^d$ with $\Sigma_1 \succ 0$ and $\Sigma_2 \succ 0$, then the OT map for the Kantorovich problem under the quadratic cost $c(x, y) = \norm{x - y}_2^2$ is given by
    \begin{align*}
        T(x) = \Sigma_1^{-1/2} (\Sigma_1^{1/2} \Sigma_2 \Sigma_1^{1/2})^{1/2} \Sigma_1^{-1/2}\, (x - m_1) + m_2,
    \end{align*}
    which induces the Wasserstein distance
    \begin{align*}
        \mr{W}_2(\mu_1, \mu_2)
        = \sqrt{\norm{m_1 - m_2}_2^2 + \tr\left( \Sigma_1 + \Sigma_2 - 2\, (\Sigma_1^{1/2} \Sigma_2 \Sigma_1^{1/2})^{1/2} \right)}.
    \end{align*}
\end{theorem}

If $m_1 = m_2$, this is called the \emph{Bures-Wasserstein distance} between symmetric positive definite matrices $\Sigma_1$ and $\Sigma_2$. One option is to use Brenier's theorem to prove \Cref{thm:wasserstein-gaussian}.

\begin{proof}[Proof (from Knott and Smith~\cite{knott1984optimal})]
    We make the ansatz that the transport map $T(x) = Ax + b$ is affine. Furthermore, we pick $A \succeq 0$ so that $T$ is the gradient of a convex function, which is required by Brenier's theorem (\Cref{thm:brenier}). By translating each marginal, it suffices to solve the centered problem $m_1=m_2=0$, where $b=0$; the original means are then restored by $T(x)=m_2+A(x-m_1)$. Then, since we need
    \begin{align*}
        \Sigma_2
        = \int (Ax) (Ax)^\intercal\, d\mu_1(x)
        = A \Sigma_1 A^\intercal,
    \end{align*}
    this forces
    \begin{align*}
        \Sigma_1^{1/2} \Sigma_2 \Sigma_1^{1/2}
        = \Sigma_1^{1/2} (A \Sigma_1 A^\intercal) \Sigma_1^{1/2}
        = (\Sigma_1^{1/2} A \Sigma_1^{1/2})^2.
    \end{align*}
    Solving for $A$ yields the OT plan, and then some algebra gives the Wasserstein distance.
\end{proof}

Another option is to give a more direct proof of \Cref{thm:wasserstein-gaussian} by expanding the cost function, without relying on Brenier's theorem.

\begin{proof}[Proof (modified from Givens and Shortt~\cite{givens1984class})]
    We begin by reducing to the case of centered measures; this step works generally for the 2-Wasserstein distance and does not require the assumption that $\mu$ and $\nu$ are Gaussian. Define $\tilde{\mu}_1 = (\cdot - m_1)_\sharp\, \mu_1$ and $\tilde{\mu}_2 = (\cdot - m_2)_\sharp\, \mu_2$. Then, for any coupling $\pi \in \Pi(\mu_1, \mu_2)$, define $\tilde{\pi} = ((\cdot - m_1),\, (\cdot - m_2))_\sharp\, \pi$ so that
    \begin{align*}
        & \int \norm{x - y}_2^2\, d\pi(x, y) \\
         & \quad = \int \norm{(x + m_1) - (y + m_2)}_2^2\, d\tilde{\pi}(x, y)                                                                \\
         & \quad = \int \norm{x - y}_2^2\, d\tilde{\pi}(x, y) + 2\int \inner{x - y,\, m_1 - m_2}\, d\tilde{\pi}(x, y) + \norm{m_1 - m_2}_2^2.
    \end{align*}
    The middle term vanishes since $\tilde{\pi}$ is centered, so it is clear that
    \begin{align*}
        \mr{W}_2(\mu_1,\mu_2)^2
        = \inf_{\pi \in \Pi(\mu_1,\mu_2)}\; \int \norm{x - y}_2^2\, d\pi(x, y)
        = \inf_{\pi \in \Pi(\tilde\mu_1,\tilde\mu_2)}\; \int \norm{x - y}_2^2\, d\pi(x, y) + \norm{m_1 - m_2}_2^2
    \end{align*}
    and it suffices to consider the case where $m_1 = m_2 = 0$. In this case, we obtain
    \begin{align*}
        \mr{W}_2(\mu_1, \mu_2)^2
        & = \inf_{\pi \in \Pi(\mu_1, \mu_2)}\; \int \norm{x - y}_2^2\, d\pi(x, y) \\
        & = \int \norm{x}_2^2\, d\mu_1(x) + \int \norm{y}_2^2\, d\mu_2(y) - 2 \sup_{\pi \in \Pi(\mu_1, \mu_2)}\; \int \inner{x, y}\, d\pi(x, y).
    \end{align*}
    Since the optimization problem only depends on the covariance of the coupling $\pi$, we may assume without loss of generality that $\pi$ is a Gaussian coupling: the covariance of any coupling is a positive semidefinite matrix with diagonal blocks $\Sigma_1$ and $\Sigma_2$, and conversely any such matrix is the covariance of a centered Gaussian coupling with marginals $\mu_1$ and $\mu_2$. Letting $\Sigma_\pi^\times \coloneq \E_{(X, Y) \sim \pi}[XY^\intercal]$, the squared 2-Wasserstein distance $\mr{W}_2(\mu_1,\mu_2)^2$ is therefore the optimal value of the following semidefinite program:
    \begin{align*}
        \inf_{\Sigma_\pi^\times \in \R^{d \times d}} & \quad \tr(\Sigma_1 + \Sigma_2 - 2 \Sigma_\pi^\times) \\
        \text{s.t.}                  & \quad \begin{bmatrix}
                                                 \Sigma_1 & \Sigma_\pi^\times        \\
                                                 (\Sigma_\pi^\times)^\intercal   & \Sigma_2
                                             \end{bmatrix} \succeq 0.
    \end{align*}
    Because $\Sigma_2 \succ 0$, the constraint is equivalent to the Schur complement constraint $\Sigma_1 - \Sigma_\pi^\times \Sigma_2^{-1} (\Sigma_\pi^\times)^\intercal \succeq 0$. The Lagrangian for this problem is
    \begin{align*}
        \mc{L}(\Sigma_\pi^\times, \Lambda)
        = \tr(\Sigma_1 + \Sigma_2 - 2 \Sigma_\pi^\times) - \tr(\Lambda (\Sigma_1 - \Sigma_\pi^\times \Sigma_2^{-1} (\Sigma_\pi^\times)^\intercal))
    \end{align*}
    for $\Lambda \succeq 0$. Finally, we use the KKT conditions to characterize the solution. By stationarity, we have
    \begin{align*}
        D_{\Sigma_\pi^\times}\, \mc{L}(\Sigma_\pi^\times, \Lambda)
        = -2 I + 2 \Lambda \Sigma_\pi^\times \Sigma_2^{-1} = 0
        \implies \Lambda \Sigma_\pi^\times = \Sigma_2.
    \end{align*}
    Complementary slackness (together with stationarity) gives
    \begin{align*}
        \Lambda (\Sigma_1 - \Sigma_\pi^\times \Sigma_2^{-1} (\Sigma_\pi^\times)^\intercal) = 0
        \implies \Lambda \Sigma_1 = \Lambda \Sigma_\pi^\times \Sigma_2^{-1} (\Sigma_\pi^\times)^\intercal
        \implies \Lambda \Sigma_1 = (\Sigma_\pi^\times)^\intercal
        \implies \Sigma_\pi^\times = \Sigma_1 \Lambda.
    \end{align*}
    Plugging this back in to the stationarity condition, we find that $\Lambda \Sigma_1 \Lambda = \Sigma_2$. Multiplying on the left and right by $\Sigma_1^{1/2}$ yields the equation
    \begin{align*}
        (\Sigma_1^{1/2} \Lambda \Sigma_1^{1/2})^2 = \Sigma_1^{1/2} \Sigma_2 \Sigma_1^{1/2},
    \end{align*}
    and solving for $\Lambda$, we find
    \begin{align*}
        \Lambda
        = \Sigma_1^{-1/2} (\Sigma_1^{1/2} \Sigma_2 \Sigma_1^{1/2})^{1/2} \Sigma_1^{-1/2}.
    \end{align*}
    By complementary slackness, we obtain the solution
    \begin{align*}
        \Sigma_\pi^\times
        = \Sigma_1 \Lambda
        = \Sigma_1^{1/2} (\Sigma_1^{1/2} \Sigma_2 \Sigma_1^{1/2})^{1/2} \Sigma_1^{-1/2},
    \end{align*}
    which fully characterizes the optimal coupling $\pi$. Using the cyclic property of trace yields the Wasserstein distance
    \begin{align*}
        \mr{W}_2(\mu_1,\mu_2)^2
         & = \tr(\Sigma_1 + \Sigma_2 - 2 \Sigma_\pi^\times)                                                                             \\
         & = \tr(\Sigma_1 + \Sigma_2 - 2 \Sigma_1^{1/2} (\Sigma_1^{1/2} \Sigma_2 \Sigma_1^{1/2})^{1/2} \Sigma_1^{-1/2}) \\
         & = \tr(\Sigma_1 + \Sigma_2 - 2 (\Sigma_1^{1/2} \Sigma_2 \Sigma_1^{1/2})^{1/2}),
    \end{align*}
    proving the theorem.
\end{proof}

Although the proof of Givens and Shortt~\cite{givens1984class} is longer than that of Knott and Smith~\cite{knott1984optimal}, it has the advantage that it does not rely on Brenier's theorem. Therefore, we take a similar approach to Givens and Shortt~\cite{givens1984class} in the proof of \Cref{thm:igw-gaussian-distance}.

\section{Background on the OT and GW Problems}

In this section, we collect additional background on OT, GW distances, Riemannian optimization, and probability theory over Hilbert spaces.

\subsection{Optimal Transport} \label{app:background-ot}

We begin by reviewing several fundamental concepts from optimal transport; the content of this section comes mostly from Villani~\cite{villani2003topics}.

\subsubsection*{The Monge-Kantorovich Problem}

In 1781, Gaspard Monge formulated the problem of finding a deterministic transport map between measures. Suppose $(\mc{X}, d_\mc{X})$ is a Polish space (complete and separable metric space).

\begin{definition}[Monge problem]
    If $\mu_1, \mu_2 \in \mc{P}(\mc{X})$ and $c : \mc{X} \times \mc{X} \to \R$ is a cost function, the \emph{Monge problem} is
    \begin{align*}
        \inf_{T_\sharp\, \mu_1 = \mu_2}\; \int c(x, T(x)) \, d\mu_1(x),
    \end{align*}
    where the infimum is taken over all Borel-measurable functions $T$.
\end{definition}

We depict the Monge problem visually in \Cref{fig:monge-problem}.

\image{0.1}{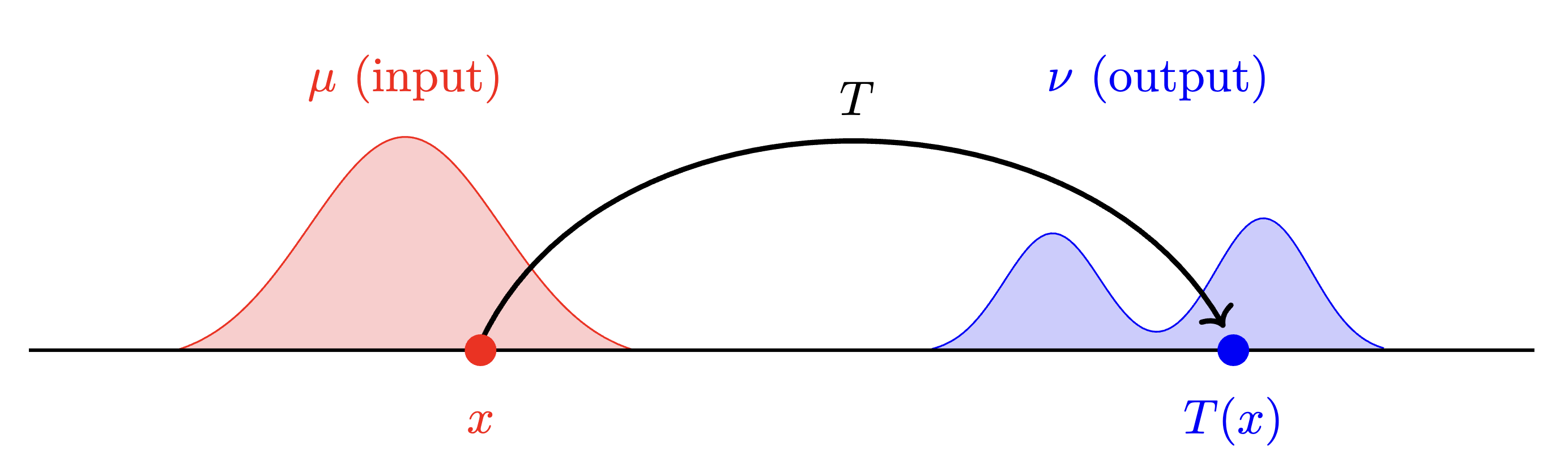}{A visual depiction of the Monge mass transportation problem.}

In 1948, Leonid Kantorovich relaxed the Monge problem to that of finding a stochastic transport plan; one motivation for the Kantorovich relaxation is that there is \emph{no deterministic transport map} from $\mu_1 = \delta_0$ to $\mu_2 = \frac{1}{2} \delta_{-1} + \frac{1}{2} \delta_1$ and another is that the Monge problem is highly nonconvex and hard to solve in practice. Recall that $\Pi(\mu_1, \mu_2)$ denotes the set of couplings between $\mu_1$ and $\mu_2$ (joint distributions with marginals $\mu_1$ and $\mu_2$).

\begin{definition}[Kantorovich problem]
    If $\mu_1, \mu_2 \in \mc{P}(\mc{X})$ and $c : \mc{X} \times \mc{X} \to \R$ is a cost function, the \emph{Kantorovich problem} is
    \begin{align*}
        \inf_{\pi \in \Pi(\mu_1, \mu_2)}\; \int c(x, y) \, d\pi(x, y).
    \end{align*}
\end{definition}

Equivalently, if $X$ and $Y$ are random variables, the Kantorovich problem is to minimize the expected cost $\E_\pi[c(X, Y)]$ over all joint distributions $\pi$ of $X$ and $Y$. Note that the Kantorovich problem is a (usually infinite-dimensional) linear program since the objective is linear in $\pi$ and the constraint is convex (an intersection of linear constraints on the marginals). For probability measures on Polish spaces, an optimal coupling exists if $c$ is lower semicontinuous and $c(x,y)\geq a(x)+b(y)$ for upper semicontinuous functions $a\in L^1(\mu_1)$ and $b\in L^1(\mu_2)$. If at least one coupling has finite cost, the optimal value is finite \cite[Theorem 4.1]{villani2009oldnew}. Even though the Kantorovich problem often has a solution, it is not always unique; for example, consider the case $\mu_1 = \frac{1}{2} \delta_{(-1, 0)} + \frac{1}{2} \delta_{(1, 0)}$ and $\mu_2 = \frac{1}{2} \delta_{(0, -1)} + \frac{1}{2} \delta_{(0, 1)}$ with the quadratic cost function.

\subsubsection*{Wasserstein Distances}

We can now define the Wasserstein distances on $\mc{P}_p(\mc{X})$ as follows.

\begin{definition}[Wasserstein distances]
    For $p \geq 1$, the $p$-Wasserstein distance between measures $\mu_1, \mu_2 \in \mc{P}_p(\mc{X})$ is defined as $\mr{W}_p(\mu_1, \mu_2) = \mr{OT}_{c_p}(\mu_1, \mu_2)^{1/p}$, where $\mr{OT}_{c_p}(\mu_1, \mu_2)$ is the optimal value of the Kantorovich problem with cost function $c_p(x, y) = d_\mc{X}(x, y)^p$.
\end{definition}

Equivalently, the Wasserstein distance can be formulated as \begin{align*}
    \mr{W}_p(\mu_1, \mu_2) = \inf_{\pi \in \Pi(\mu_1, \mu_2)}\, \norm{d_\mc{X}(X, Y)}_{L^p(\pi)}.
\end{align*}
The $p$-Wasserstein distance is a metric on $\mc{P}_p(\mc{X})$. Next, assume that $c$ is the quadratic cost function on $\R^d$ and let $\lambda$ denote the Lebesgue measure on $\R^d$. An astonishing result (due to Brenier) is that as long as $\mu \ll \lambda$, the OT map is unique and is given by the graph of the gradient of a convex function. Furthermore, \emph{any} valid transport coupling which is the gradient of a convex function is optimal.

\begin{theorem}[\cite{brenier1991polar}] \label{thm:brenier}
    Let $\mu_1, \mu_2 \in \mc{P}_2(\R^d)$ be such that $\mu_1$ is absolutely continuous with respect to the Lebesgue measure. Then, $\pi \in \Pi(\mu_1, \mu_2)$ is optimal for the Kantorovich problem with $c(x, y) = \norm{x - y}_2^2$ if and only if there exists a proper l.s.c. convex function $\varphi : \R^d \to \R\cup\{+\infty\}$ such that $\pi = (I,\, \nabla \varphi)_\sharp\, \mu_1$, where $\nabla \varphi$ is defined uniquely $\mu_1$-almost everywhere.
\end{theorem}

Using Brenier's theorem (the approach taken by Knott and Smith~\cite{knott1984optimal}) or by expanding the cost directly (the approach taken by Givens and Shortt~\cite{givens1984class}), one can compute the 2-Wasserstein distance and OT map between Gaussian measures in closed-form; details of these two approaches and the resulting formulae are given in \Cref{app:gaussian-wasserstein}.

\subsection{Gromov-Wasserstein Distances} \label{app:background-gw}

In this section, we motivate and formally define GW distances \cite{memoli2011gromov}, which provide a principled method to compare two metric measure spaces; in particular, we can use these to compare two probability distributions (possibly over different spaces) in a way that respects their underlying geometry. We begin motivating GW distances by asking the question: how can one compare two compact sets (shapes) in a metric space? One principled method is to use the Hausdorff distance (recall that $\dist(a, B) \coloneq \inf_{b \in B} d_\mc{X}(a, b)$ for a compact subset $B \subseteq \mc{X}$ and a point $x \in \mc{X}$).

\begin{definition}[Hausdorff distance]
    The \emph{Hausdorff distance} between compact subsets $A$ and $B$ of a metric space $(\mc{X}, d_\mc{X})$ is defined as
    \begin{align*}
        d_H(A, B)
        \coloneq \max\set*{\sup_{a \in A}\, \dist(a, B),\; \sup_{b \in B}\, \dist(b, A)}.
    \end{align*}
\end{definition}

However, this notion of distance is overly sensitive to minor intricacies in these shapes; for instance, two shapes may be mostly identical except for a thin spike in the second one, but the Hausdorff distance will detect this and label the two sets as far apart. One way to make the Hausdorff distance less stringent, as well as to make these distances faster to compute, is to consider an $L^p$ relaxation. In particular, we can place probability measures $\mu_A$ and $\mu_B$ over $A$ and $B$ respectively (which intuitively measure relative importance of features in each shape) and compute the $p$-Wasserstein distance between these measures. The Hausdorff distance is a metric on the set of compact subsets of a metric space $(\mc{X}, d_\mc{X})$. In fact, we can use the Hausdorff distance to define a notion of distance between any two compact metric spaces by minimizing over all isometric embeddings of the two spaces; this process is sometimes called \emph{Gromovization}.

\begin{definition}[Gromov-Hausdorff (GH) distance]
    The \emph{Gromov-Hausdorff} (GH) distance between two compact metric spaces $(\mc{X}, d_\mc{X})$ and $(\mc{Y}, d_\mc{Y})$ is
    \begin{align*}
        \operatorname{GH}(\mc{X}, \mc{Y})
        = \inf_{\iota_\mc{X},\, \iota_\mc{Y}}\; d_H(\iota_\mc{X}(\mc{X}), \iota_\mc{Y}(\mc{Y})),
    \end{align*}
    where the infimum is taken over all isometric embeddings $\iota_\mc{X}: \mc{X} \to \mc{Z}$ and $\iota_\mc{Y}: \mc{Y} \to \mc{Z}$ for some shared metric space $(\mc{Z}, d_\mc{Z})$.
\end{definition}

Intuitively, the infimum is doing the work of alignment, while the Hausdorff distance measures the distance between the resulting compact sets. However, it is not clear at all how to compute the GH distance; this motivates the GW distance as an $L^p$ relaxation of the GH distance. Of course, this means that the GW distance is a less stringent version of the GH distance. The relationships between these distances are roughly shown in \Cref{fig:distance-comparison}, taken from M{\'e}moli~\cite{memoli2011gromov}.

\begin{figure}[htbp]
    \centering
    \begin{tikzpicture}[auto]
        \node (hausdorff) at (0,2) {Hausdorff ($d_H(A, B)$)};
        \node (gromov-hausdorff) at (9,2) {Gromov-Hausdorff ($\operatorname{GH}(\mc{X}, \mc{Y})$)};
        \node (wasserstein) at (0,0) {Wasserstein ($\mr{W}_p(\mu_A, \mu_B)$)};
        \node (gromov-wasserstein) at (9,0) {Gromov-Wasserstein ($\GW_{(p,\, q)}(\mu_\mc{X}, \mu_\mc{Y})$)};

        \draw[->] (hausdorff) -- (gromov-hausdorff) node[midway, above] {\emph{Gromovization}};
        \draw[->] (wasserstein) -- (gromov-wasserstein) node[midway, above] {\emph{Gromovization}};
        \draw[->] (hausdorff) -- (wasserstein) node[midway, left] {\emph{$L^p$ relaxation}};
        \draw[->] (gromov-hausdorff) -- (gromov-wasserstein) node[midway, right] {\emph{$L^p$ relaxation}};
    \end{tikzpicture}
    \caption{Relationships between various distances on spaces of objects and distributions.}
    \label{fig:distance-comparison}
\end{figure}

We begin with a relevant definition. Recall that a \emph{locally finite} measure space is one where every point has an open neighborhood of finite measure.

\begin{definition}[Metric measure space]
    A \emph{metric measure space} (m.m. space) is a tuple $(\mc{X},\, d,\, \mu)$ such that $(\mc{X}, d)$ is a Polish space and $(\mc{X},\, \mc{B}(\mc{X}),\, \mu)$ is a locally finite measure space.
\end{definition}

Here are a few important examples of metric measure spaces.

\begin{example}[Euclidean space]
    The space $(\R^d,\, \norm{\cdot}_p,\, \lambda)$ is an m.m. space, where $p \geq 1$ and $\lambda$ denotes the Lebesgue measure on $\R^d$.
\end{example}

\begin{example}[Random variables]
    If $X$ is an $\R^d$-valued random variable with distribution $\mu$, then $(\R^d,\, \norm{\cdot}_p,\, \mu)$ is an m.m. space for $p \geq 1$.
\end{example}

We are now ready to define the GW distance between two (not necessarily compact) m.m. spaces. Recall that $\Pi(\mu_\mc{X},\, \mu_\mc{Y})$ denotes the set of couplings between $\mu_\mc{X}$ and $\mu_\mc{Y}$.

\begin{definition}[GW distance]
    Let $p,q\in[1,\infty)$, and let $(\mc{X},d_\mc{X},\mu_\mc{X})$ and $(\mc{Y},d_\mc{Y},\mu_\mc{Y})$ be Polish metric spaces equipped with Borel probability measures satisfying $\mu_\mc{X}\in\mc{P}_{pq}(\mc{X})$ and $\mu_\mc{Y}\in\mc{P}_{pq}(\mc{Y})$. Then the \emph{$(p, q)$-GW distance} between $\mu_\mc{X}$ and $\mu_\mc{Y}$ is defined as
    \begin{align*}
        \GW_{(p,\, q)}(\mu_\mc{X},\, \mu_\mc{Y})
         & = \inf_{\pi \in \Pi(\mu_\mc{X},\, \mu_\mc{Y})}\; \left( \int \abs{d_\mc{X}(x, x^\prime)^q - d_\mc{Y}(y, y^\prime)^q}^p\, d(\pi \otimes \pi)(x, y, x^\prime, y^\prime) \right)^{1/p} \\
         & = \inf_{\pi \in \Pi(\mu_\mc{X},\, \mu_\mc{Y})}\; \norm{d_\mc{X}(X, X^\prime)^q - d_\mc{Y}(Y, Y^\prime)^q}_{L^p(\pi \otimes \pi)},
    \end{align*}
\end{definition}

The $(p, q)$-GW distance looks for the measure coupling $\pi$ between two m.m. spaces which is as close to an isometry as possible. The GW distance independently samples two pairs $(X, Y)$ and $(X^\prime, Y^\prime)$ according to $\pi$ and minimizes the $p$th moment of the deviation of the $q$th powers of distances in $\mc{X}$ and $\mc{Y}$. In this work, we analyze a variant of the GW distance, called the inner product GW distance, which was introduced in Vayer~\cite{vayer2020contribution}.

\begin{definition}[IGW distance]
    Let $\mu, \nu \in \mc{P}_2(\mc{H})$. Then, the \emph{inner product Gromov-Wasserstein (IGW) distance} between $\mu$ and $\nu$ is
    \begin{align*}
        \IGW(\mu, \nu)
         & = \inf_{\pi \in \Pi(\mu, \nu)}\; \left( \int (\inner{x, x^\prime} - \inner{y, y^\prime})^2\, d(\pi \otimes \pi)(x, y, x^\prime, y^\prime) \right)^{1/2} \\
         & = \inf_{\pi \in \Pi(\mu, \nu)}\; \norm{\inner{X, X^\prime} - \inner{Y, Y^\prime}}_{L^2(\pi \otimes \pi)}.
    \end{align*}
\end{definition}
While the GW distance was looking for a measure coupling which is as close to an isometry as possible, the IGW distance looks for a measure coupling which is as close to unitary as possible. The existence of an optimal coupling for the IGW problem was shown in Vayer~\cite{vayer2020contribution}. IGW is a metric on the zero-distance quotient of measures with finite second moment \cite[Proposition 3.1]{zhang2024gradient}.

\begin{proposition}[Zhang et al.~\cite{zhang2024gradient}, Proposition 3.1]
    The IGW distance is a metric on equivalence classes of $\mc{P}_2(\mc{H})$, where $\mu \equiv \nu$ if there exists a measurable map $\iota : \mc{H} \to \mc{H}$ with $\iota_\sharp\, \mu = \nu$, in the sense that $\inner{x, y} = \inner{\iota(x), \iota(y)}$ for $(\mu \times \mu)$-almost every $x, y \in \mc{H}$.
\end{proposition}

The inner-product-preserving identification on the supports need not extend to a unitary operator on an infinite-dimensional ambient space. By orthogonal invariance, IGW can compare measures on different separable Hilbert spaces by isometrically embedding them into a common Hilbert space $\ell^2$. The objective is nonlinear in the coupling, but admits a useful variational form in Euclidean spaces \cite{zhang2024gradient}.

{
Let $\mu\in\mc P_2(\R^{d_\mu})$ and $\nu\in\mc P_2(\R^{d_\nu})$. For $\pi\in\Pi(\mu,\nu)$ and independent copies $(X,Y),\allowbreak(X',Y')\sim\pi$, define
\begin{align*}
    M_\pi=\int xy^\intercal\,d\pi(x,y),\qquad
    S_1(\mu,\nu)=\left\|\int xx^\intercal\,d\mu(x)\right\|_{\mr F}^2
    +\left\|\int yy^\intercal\,d\nu(y)\right\|_{\mr F}^2.
\end{align*}
The second-moment identity gives
\begin{align*}
    \E[(\inner{X,X'}-\inner{Y,Y'})^2]=S_1(\mu,\nu)-2\norm{M_\pi}_{\mr F}^2.
\end{align*}
Consequently, $\IGW(\mu,\nu)^2=S_1(\mu,\nu)+S_2(\mu,\nu)$, where
\begin{align*}
    S_2(\mu,\nu)\coloneq\inf_{\pi\in\Pi(\mu,\nu)}\{-2\norm{M_\pi}_{\mr F}^2\}.
\end{align*}
Here $M_\pi$ is an uncentered cross-moment; it is a cross-covariance only when the marginal means vanish.
}

\begin{proposition}[Zhang et al.~\cite{zhang2024gradient}, Lemma 2.1]
    Given $A \in \R^{d_\mu \times d_\nu}$, define $\mc{T}_A(\mu, \nu)$ to be the optimal value of the Kantorovich problem between $\mu$ and $\nu$ under the cost $c_A(x, y) = -8 x^\intercal A y$. Then we have the following variational form, for $S_2(\mu,\nu)$ defined above:
    \begin{align*}
        S_2(\mu,\nu) = \inf_{A \in \R^{d_\mu \times d_\nu}} \set{8\, \norm{A}_F^2 + \mc{T}_A(\mu, \nu)}.
    \end{align*}
    Furthermore, the minimum is achieved at some
    \begin{align*}
        A \in \left[ -\frac{1}{2} \sqrt{\E_{(X, Y) \sim \mu \times \nu}[\norm{X}^2\, \norm{Y}^2]},\; \frac{1}{2} \sqrt{\E_{(X, Y) \sim \mu \times \nu}[\norm{X}^2\, \norm{Y}^2]} \right]^{d_\mu \times d_\nu}.
    \end{align*}
\end{proposition}
\begin{proof}

For fixed $\pi$, completing the square gives
\begin{align*}
    8\norm A_{\mr F}^2-8\inner{A,M_\pi}_{\mr F}
    =8\norm{A-M_\pi/2}_{\mr F}^2-2\norm{M_\pi}_{\mr F}^2.
\end{align*}
Taking infima over $A$ and $\pi$ in either order proves the formula. An optimal IGW coupling exists, and $A=M_\pi/2$ for such a coupling attains the minimum. Cauchy--Schwarz bounds each entry of $M_\pi$ by $\big(\int\norm x^2d\mu\int\norm y^2d\nu\big)^{1/2}$, giving the stated box.
\end{proof}

While we derive our main results about the IGW distance by expanding the square directly, note that this variational form could be used to recover some of our results in the finite-dimensional case.

\section{Background on Riemannian Optimization} \label{app:background-riemannian}

In this section, we briefly review Riemannian gradient descent (RGD) and some properties of the Stiefel manifold. We begin by recalling several important definitions from manifold optimization theory; see Absil et al.~\cite{absil2008optimization} for a more thorough treatment. Let $\mc{M}$ be a Riemannian manifold and let $T_x \mc{M}$ denote the tangent space to $\mc{M}$ at $x \in \mc{M}$. We let $\operatorname{grad} f(x) \in T_x \mc{M}$ denote the Riemannian gradient of a function $f : \mc{M} \to \R$ at $x \in \mc{M}$ and let $\operatorname{Hess} f(x) : T_x \mc{M} \to T_x \mc{M}$ denote the Riemannian Hessian of $f$ at $x$.

We say that $y \in \mc{M}$ is a \emph{first-order stationary point} of the optimization problem $\inf_{z \in \mc{M}}\, f(z)$ if $f \in C^1(\mc{M})$ and $\operatorname{grad} f(y) = 0$. We say that $y \in \mc{M}$ is a \emph{second-order stationary point} of  $\inf_{z \in \mc{M}}\, f(z)$ if $y$ is a first-order stationary point, $f \in C^2(\mc{M})$, and $\operatorname{Hess} f(y) \succeq 0$.

These definitions are analogous to the standard definitions of first- and second-order stationary points in unconstrained optimization, except that the gradient and Hessian are replaced by their Riemannian counterparts. Intuitively, first-order stationary points are points where the objective function has zero directional derivative along all directions tangent to the manifold, while second-order stationary points are first-order stationary points where the Riemannian Hessian is positive semidefinite on the tangent space.

The RGD algorithm for $\inf_{z \in \mc{M}}\, f(z)$ is analogous to the usual gradient descent algorithm, and begins with an initialization $x_0 \in \mc{M}$ and a step size $\alpha > 0$. Let $\operatorname{retr}_x : T_x(\mc{M}) \to \mc{M}$ denote a retraction map onto $\mc{M}$ for $x \in \mc{M}$. Until convergence, the RGD update is $x_{t+1} = \operatorname{retr}_{x_t}(- \alpha\, \operatorname{grad} f(x_t))$. Under standard line-search assumptions, accumulation points are first-order stationary. Riemannian trust-region methods may be used to find second-order stationary points (as required by \Cref{thm:burer-monteiro-multimarginal}); see Absil et al.~\cite{absil2008optimization} for details.

Let $\mr{St}(d_1, d_2) \coloneq \set{C \in \R^{d_1 \times d_2} : C^\intercal C = I}$ denote the Stiefel manifold of orthonormal $d_2$-frames in $\R^{d_1}$; this is the manifold of interest in \Cref{sec:igw-gaussian}. For the metric induced by the ambient Frobenius inner product, the Riemannian gradient of $f\in C^1(\mr{St}(d_1,d_2))$ is $\operatorname{grad} f(C) = \nabla f(C) - \frac{C}{2} (\nabla f(C)^\intercal C + C^\intercal \nabla f(C))$. Finally, we define $\operatorname{retr}_C(G) \coloneq Q$ as the retraction onto the Stiefel manifold, where $C+G=QR$ is the QR factorization with positive diagonal entries in $R$.

\section{Background on Random Variables in Hilbert Spaces} \label{app:background-hilbert}
In this section, we collect several useful facts about Hilbert space-valued random variables. Let $\mc H$ be a separable real Hilbert space. For a centered measure $\mu\in\mc P_2(\mc H)$, recall \cite{baker1973joint} that its covariance is the bilinear form
\begin{align*}
    \Cov_\mu(x, y) \coloneq \int \inner{x, z} \inner{y, z}\, d\mu(z).
\end{align*}
By the Riesz-Fr\'echet representation theorem, there exists a unique linear operator $\Sigma$ such that $\Cov_\mu(x,y)=\inner{x,\Sigma y}$, and $\Sigma=\Sigma_\mu$ because $\mu$ is centered. The covariance operator is positive semidefinite because
\begin{align*}
    \inner{x, \Sigma x}
    = \Cov_\mu(x, x)
    = \int \inner{x, z}^2\, d\mu(z)
    \geq 0
\end{align*}
for all $x \in \mc{H}$ and self-adjoint by symmetry of the uncentered covariance. Fix a total orthonormal set $\set{e_k}_{k \in \N}$ of $\mc{H}$, so that
\begin{align*}
    \tr(\Sigma)
    = \sum_{k=1}^\infty \inner{e_k, \Sigma e_k}
    = \sum_{k=1}^\infty \Cov_\mu(e_k, e_k)
    = \sum_{k=1}^\infty \int \inner{e_k, z}^2\, d\mu(z).
\end{align*}
Using Tonelli's theorem followed by Parseval's identity,
\begin{align*}
    \sum_{k=1}^\infty \int \inner{e_k, z}^2\, d\mu(z)
    = \int \sum_{k=1}^\infty \inner{e_k, z}^2\, d\mu(z)
    = \int \norm{z}^2\, d\mu(z)
    < \infty,
\end{align*}
it follows that $\Sigma$ is trace-class and therefore compact. {
By the spectral theorem, the positive eigenvalues have orthonormal eigenvectors $v_k$ spanning $\overline{\range(\Sigma)}$, and
\begin{align*}
    \Sigma=\sum_{k:\lambda_k(\Sigma)>0}\lambda_k(\Sigma)v_k\otimes v_k,
    \qquad\tr(\Sigma)=\sum_k\lambda_k(\Sigma).
\end{align*}
The trace identity is a special case of Lidskii's theorem \cite{simon2005trace}. Positive eigenvalues can be listed in nonincreasing order, with zeros appended when the rank is finite. If the rank is infinite and the kernel is nontrivial, this ordered sequence does not index a total eigenbasis of $\mc H$; an orthonormal basis of the kernel must be treated separately.
}

Similarly, suppose we are given centered measures $\mu,\nu\in\mc P_2(\mc H)$ and a coupling $\pi \in \Pi(\mu, \nu)$. Then, we can define the cross-covariance operator as the unique bounded linear operator $\Sigma_\pi^\times$ such that
\begin{align*}
    \Cov_\pi^\times(x, y)
    \coloneq \int \inner{x, z} \inner{y, w}\, d\pi(z, w)
    = \inner{x, \Sigma_\pi^\times y}
\end{align*}
for all $x, y \in \mc{H}$, which exists by the Riesz-Fr\'echet representation theorem. The cross-covariance operator is Hilbert-Schmidt because by the Cauchy-Schwarz inequality,
\begin{align*}
    \norm{\Sigma_\pi^\times}_{\HS}^2
     & = \sum_{k=1}^\infty \sum_{j=1}^\infty \inner{e_k, \Sigma_\pi^\times e_j}^2                                                                         \\
     & = \sum_{k=1}^\infty \sum_{j=1}^\infty \left( \int \inner{e_k, z} \inner{e_j, w}\, d\pi(z, w) \right)^2                                 \\
     & \leq \sum_{k=1}^\infty \sum_{j=1}^\infty \left( \int \inner{e_k, z}^2\, d\mu(z) \right) \left( \int \inner{e_j, w}^2\, d\nu(w) \right) \\
     & = \tr(\Sigma_\mu) \tr(\Sigma_\nu)                               < \infty.
\end{align*}
For general, uncentered marginals, the same argument bounds the uncentered cross-moment using $M_\mu=\Sigma_\mu+m_\mu\otimes m_\mu$ and $M_\nu$, giving $\norm{\E[X\otimes Y]}_{\HS}^2\leq\tr(M_\mu)\tr(M_\nu)$. Finally, given any $m \in \mc{H}$ and trace-class self-adjoint positive semidefinite operator $\Sigma$, there exists a unique Gaussian measure $\Normal(m, \Sigma) \in \mc{P}_2(\mc{H})$ with mean $m$ and covariance operator $\Sigma$; see Bogachev~\cite{bogachev1998gaussian} for details. This measure is characterized precisely by the property that for any $h \in \mc{H}$, the pushforward measure $(x\mapsto\inner{h,x})_\sharp\Normal(m,\Sigma)$ is a univariate Gaussian with mean $\inner{h, m}$ and variance $\inner{h, \Sigma h}$.

\end{document}